\documentclass[twoside]{article}

\usepackage[accepted]{aistats2022}
\usepackage[round]{natbib}

\usepackage{macros}

\usepackage[utf8]{inputenc} % allow utf-8 input
\usepackage[T1]{fontenc}    % use 8-bit T1 fonts
\usepackage{hyperref}       % hyperlinks
\usepackage{url}            % simple URL typesetting
\usepackage{booktabs}       % professional-quality tables
\usepackage{amsfonts}       % blackboard math symbols
\usepackage{nicefrac}       % compact symbols for 1/2, etc.
\usepackage{microtype}      % microtypography
\usepackage{xcolor}         % colors

\renewcommand{\cite}{\citep}

\newcommand{\algorithmicelsif}{\textbf{else if}}
\algdef{C}[IF]{IF}{NoThenElseIf}[1]{\algorithmicelsif\ #1}
\begin{document}

% If your paper is accepted and the title of your paper is very long,
% the style will print as headings an error message. Use the following
% command to supply a shorter title of your paper so that it can be
% used as headings.
%
%\runningtitle{I use this title instead because the last one was very long}

% If your paper is accepted and the number of authors is large, the
% style will print as headings an error message. Use the following
% command to supply a shorter version of the authors names so that
% they can be used as headings (for example, use only the surnames)
%
%\runningauthor{Surname 1, Surname 2, Surname 3, ...., Surname n}

\twocolumn[

\aistatstitle{Nearly Optimal Algorithms for Level Set Estimation}

% \aistatsauthor{ Blake Mason \And Romain Camilleri \And  Subhojyoti Mukherjee \AND Kevin Jamieson \And Robert Nowak \And Lalit Jain }
\aistatsauthor{ Blake Mason \And Romain Camilleri \And  Subhojyoti Mukherjee }
\aistatsaddress{ Rice University \And  University of Washington \And University of Wisconsin-- Madison } 
\aistatsauthor{Kevin Jamieson \And Robert Nowak \And Lalit Jain }
\aistatsaddress{ University of Washington \And  University of Wisconsin-- Madison \And University of Washington }
]

\begin{abstract}
The level set estimation problem seeks to find all points in a domain $\X$ where the value of an unknown function $f:\X\rightarrow \mathbb{R}$ exceeds a threshold $\alpha$. The estimation is based on noisy function evaluations that may be acquired at sequentially and adaptively chosen locations in $\X$. The threshold value $\alpha$ can either be \emph{explicit} and provided a priori, or \emph{implicit} and defined relative to the optimal function value, i.e.  $\alpha = (1-\epsilon)f(\bx_\ast)$ for a given $\epsilon > 0$ where $f(\bx_\ast)$ is the maximal function value and is unknown.  In this work we provide a new approach to the level set estimation problem by relating it to recent adaptive experimental design methods for linear bandits in the Reproducing Kernel Hilbert Space (RKHS) setting. We assume that $f$ can be approximated by a function in the RKHS up to an unknown misspecification and provide novel algorithms for both the implicit and explicit cases in this setting with strong theoretical guarantees. 
Moreover, in the linear (kernel) setting, we show that our bounds are nearly optimal, namely, our upper bounds match existing lower bounds for threshold linear bandits. 
To our knowledge this work provides the first instance-dependent, non-asymptotic upper bounds on sample complexity of level-set estimation that match information theoretic lower bounds.
\end{abstract}

\section{Introduction}

The level-set of a function is a subset of its domain where it exceeds a specific value. 
Level set estimation is the problem of identifying a subset that approximates the true level-set based on a finite set of potentially noisy function evaluations. 
As an example, consider the goal of detecting a region in a body of water, such as a channel, that is at least $20m$ deep for ships to safely pass. 
Given that we can obtain noisy estimates of depth using a sonar device at the locations of our choosing, where should we measure in order to acquire the most accurate level-set estimation while using as few total measurements as possible?
Level-set estimation can also be interpreted as a kind of classification rule. For example, using as few total experiments as possible, we may want to identify all compounds among a given finite set under consideration that have some property (e.g., binding affinity) that exceeds some target threshold.

% For example, this goal could be related to testing water quality in lakes to detect the presence of contaminants above a safe level or to finding all product designs that exceed a minimal consumer score/rating. 
% \kevin{In each of these examples, its not clear what the features are and what the level-set represents. Are we sure we don't want a more visual depth example?}
% \rob{I don't really like the sonar example, and I think the problem can be motivated/illustrated by just one simple sentence, like the one above.}
% To  illustrate, consider the goal of identifying the regions of a body of water, say, a canal, that are of depth at least 20m for ship traffic. 
% To estimate this level-set of depth 20m, one can imagine going to different locations in the body of water and using a sonar measurement device that simply approximates the distance to the bottom by multiplying the speed of sound in the water by the round trip time. 
% Unfortunately, the water conditions like salinity and temperature alter the speed of sound so the observation is only approximating the truth. 
% Importantly, these sources of ``error'' are not independent or even stochastic, but are unlikely to be too large in magnitude.   
% Given that we can obtain an approximate depth at the locations of our choosing, where should we measure in order to acquire the most accurate level-set estimation using as few total measurements as possible?

While level-set estimation is somewhat of a well-studied problem, to date there is a lack of theoretical understanding of the limits and tradeoffs of estimation accuracy and number of measurements. 
Most algorithms proceed by sequentially and greedily optimizing an \emph{acquisition function} that is constructed using all the measurements observed up to the current time. 
These heuristics are known to work very well in practice, but their guarantees are ad hoc and, at best, worst-case (minimax).
In this work we are interested in understanding the instance-dependent sample complexity of level-set estimation.
That is, we would like for an algorithm to output a satisfactory estimate of the level-set as fast as \emph{any} algorithm could for \emph{this particular} instance, not some worst-case instance.

In contrast to prior works that propose a sampling heuristic--usually based on identifying an informative point--and bound its sample complexity, we work backwards. 
Namely, we first consider an information theoretic lower bound for the level-set estimation problem that suggests an ``optimal'' sampling strategy. 
Because this ideal sampling strategy is a function of the true (unknown) function, it is a priori impossible to realize.
Instead, we propose a series of sampling strategies, based on \emph{experimental designs}, that mimic this optimal sampling strategy given the information available at the current time. By the end, these strategies provably achieve the optimal sample complexity with minimal overhead. Furthermore, we show that our sampling strategy leads to a upper bound on the sample complexity that is tighter than those in the existing literature.
In what follows, we first formally state the problem and our desired objectives. 
We then review the related work in context before proceeding to our lower bounds and algorithms. 
We finish with experiments contrasting with existing work. 

\subsection{Problem statement}

We assume there exists an unknown function $f: \R^d \rightarrow [-B,B]$ and a  subset of allowable sampling locations $\mc{X}\subset \mathbb{R}^d$ which span $\R^d$.
Though the function $f$ is unknown, we may query its value for any $\bx \in \X$ and receive a noisy estimate $f(\bx) + \eta$ where $\eta$ is iid, $\E[\eta] = 0$, and $\E[\eta^2] \leq \sigma^2$.
We define two objectives.  

%\kevin{We can change this model to be that errors are time-varying and the only thing we know is that they are totally bounded by $h$. 
% That is, the errors may change over time but they don't average to anything (since they're not random)

\textbf{Explicit Level Set Estimation:} Given a specified threshold $\alpha \in \mathbb{R}$, the goal is to identify $G_{\alpha} := \{\bx\in \mc{X}: f(x) > \alpha\}$. 

\textbf{Implicit Level Set Estimation:} Let $\bx_{\ast} \in \arg\max_{\bx\in \mc{X}} f(\bx)$. Given $\epsilon > 0$, the goal is to identify $G_{\epsilon}:=\{\bx\in \mc{X}: f(\bx)> (1-\epsilon)f(\bx_{\ast})\}$.

% \textbf{Explicit Level Set Estimation:} Given a specified threshold $\alpha \in \mathbb{R}$ and a tolerance $\widetilde{\beta}$, the goal is to identify a set $\hG$ such that $\hG \subset G_{\alpha, \widetilde{\beta}} := \{x\in \mc{X}: f(x) > \alpha - \widetilde{\beta}\}$. 

% \textbf{Implicit Level Set Estimation:} Let $x^{\ast} \in \arg\max_{x\in \mc{X}} f(x)$. Given $\epsilon > 0$ and a tolerance $\widetilde{\beta}$, the goal is to identify a set $\hG$ such that $\hG \subset G_{\epsilon, \widetilde{\beta}}:=\{x\in \mc{X}: f(x)> (1-\epsilon)f(x^{\ast}) - \widetilde{\beta}\}$. 

Consider an algorithm that at each time $t$ selects an arm $\bx_{t} \in \X$ that is measurable with respect to a $\sigma$-algebra $\mathcal{F}_{t-1} = \sigma(\bx_1, y_1, \cdots, \bx_{t-1},y_{t-1})$ and receives a value $y_t = f(\bx_{t}) + \eta_t$. %The objective of the algorithm is to identify all points in $G_{\alpha,\widetilde{\beta}}$ in the explicit setting and $G_{\epsilon, \widetilde{\beta}}$ in the implicit setting in as few samples as possible. We take $G_\alpha = G_{\alpha, 0}$ and $G_\epsilon = G_{\epsilon, 0}$. 
To be precise, we say that an algorithm is \emph{PAC-$\delta$} for the explicit (respectively implicit) level set problem if it stops at a time $T_\delta$ which is measurable with respect to the filtration $(\mc{F}_t)_{t\geq 1}$ and returns $G_{\alpha}$ (and in the implicit setting returns $G_{\epsilon}$).
If $f(\bx)$ is very close to the threshold, it may take an enormous number of samples to determine whether it is above or below the threshold, so in practice we introduce a $\widetilde{\beta} \geq 0$ tolerance that ensures that any learner has a finite sample complexity (see theorems). 
But in the discussion that follows, assume that $f(x)$ is bounded away from the threshold.

%$G_{\alpha, \widetilde{\beta}}$ (resp. $G_{\epsilon, \widetilde{\beta}}$) with probability greater than $1-\delta$. 

Our approach is based on modeling $f$ in a Reproducing Kernel Hilbert Space (RKHS) $\mc{H}$.  Let $\phi : \R^d \mapsto \mc{H}$ be the ``feature map'' associated with the RKHS.  Since $|f(x)|\leq B$ for all $\bx\in \X$, there exists a $\theta_\ast \in \mc{H}$ and a scalar $h\geq 0$ such that $\max_{\bx\in\mc{X}} |f(\bx) - \langle \theta_*,\phi(\bx)\rangle_{\mc{H}}| \leq h$. Our sample complexity bounds will depend on $h$ and $\|\theta^*\|_{\mc{H}}$ which we denote $\|\theta_\ast\|$.
% To aid in learning $G_{\alpha,\widetilde{\beta}}$ and $G_{\epsilon,\widetilde{\beta}}$, we place the following structure on $f$:
% \textbf{Assumption:} There exists a known feature map $\phi : \R^d \mapsto \mc{H}$ that maps each $x \in \mc{X}$ to a (possibly infinite dimensional) Hilbert space $\mc{H}$, and moreover, there exists an unknown $\theta_* \in \mc{H}$ and an unknown $h \geq 0$ such that $\max_{x\in\mc{X}} |f(\bx) - \langle \theta_*,\phi(\bx)\rangle_{\mc{H}}| \leq h$. \rob{how is this an assumption if $h\geq 0$ is unknown... this is just a fact (if $\mc{H}$ is infinite), not an assumption, since $\X$ is finite.}
If $h$ is small, then $f$ is well approximated as a linear function of the feature maps $\phi(\bx)$. We refer to the case when $h > 0$ as being \emph{misspecified} and otherwise when $h = 0$ as being \emph{well-specified}. 
This class of functions is frequently used for level-set estimation because it is often sufficiently rich to model real-world functions but also contains enough structure to quantify the uncertainty of generalizing a learned function to unmeasured locations. 
One note of departure from the existing literature is that we do not assume the unknown function is precisely captured by a function in an RKHS, only that it is well approximated by one (i.e., the misspecified setting). In the discussion that follows, we additionally assume $|\mc{X}| < \infty$  for simplicity since in practice given an arbitrary bounded domain we can replace $\mc{X}$ with a finite cover.

\section{Related Works}

The level-set estimation problem naturally connects to several related ideas in Bayesian optimization and multi-armed bandits. In the former setting, methods tend to sample greedily according to an acquisition function that seeks to minimize the uncertainty of the learner about the level set.
The first work on level set estimation that employed the use of Gaussian processes and introduced the Straddle heuristic is due to \citet{bryan2005active}. These ideas were further developed in \citet{gotovos2013active} which proposed the LSE and LSE-imp algorithms for explicit and implicit level set respectively. They provide a theoretical guarantee on the sample complexities of LSE and LSE-imp, and as we will show below, our sample complexity is always at least as good as their stated bounds.  \citet{bogunovic2016truncated} further connected Bayesian optimization with level set estimation and considered the setting of heteroscedastic noise. The work of \citet{shekhar2019multiscale} focuses on the level-set problem in a continuous domain, and provides an algorithm that maintains a notion of uncertainty over regions, providing a potentially improved computational complexity, along with tighter sample complexity bounds compared to LSE for certain kernels and smoothness assumptions.  The work of \citet{zanette2018robust} reposes level-set estimation as a classification problem and introduces a novel acquisition acquisition function. \citet{iwazaki2020bayesian} extends the work of \citet{zanette2018robust} to improve model robustness in quality control applications.
This line of work is also related to Gaussian Process Bandits, namely the GPUCB algorithm and improved variants \cite{srinivas2009gaussian, chowdhury2017kernelized, valko2013finite}. \citet{ha2020high} introduces a Bayesian Neural Network approach for active level set estimation using Monte Carlo dropout techniques. 
Table~\ref{table:lse} in the appendix summarizes the results we are aware of in the Gaussian process setting. 

In the multi-armed and linear bandit setting, the explicit level set estimation problem is related to threshold bandits where one seeks to find all arms above an explicit threshold \cite{locatelli2016optimal,jamieson2018bandit, degenne2020gamification}. The approach of \citet{degenne2020gamification}, would provide an  asymptotically optimal algorithm in the linear setting, however we are not aware of any other works that provide an optimal finite-time guarantee. The implicit level set problem in the standard multi-armed bandit setting is equivalent to the \emph{multiplicative all-$\epsilon$} problem introduced by \citet{mason2020finding}. Algorithm~\ref{alg:MILK} recovers the sample complexities of the instance-optimal $(\texttt{ST})^2$ algorithm given there. 
Finally, our experimental design techniques are inspired by \citet{soare2014best, fiez2019sequential}, and especially
the recent work of \citet{camilleri2021highdimensional} that introduces the RIPS estimator which we use to perform experimental design in an RKHS. 
% \sm{cite houng ha's paper \citep{ha2020high}. possibly another paper but slightly different setting (Houng cites it) \citep{iwazaki2019bayesian}}

% \begin{itemize}
%     The level-set estima

%     \item Gaussian process related works: key ones, LSE, Zanette, Volkan. A slightly different line of work is GP-UCB, Chaudhury, Valko
%     \item Connection to multi-armed bandits: Threshold bandits (Locatelli - APT)
%     \item Connection to All-Epsilon (and single epsilon?): Blake's paper, Julian's paper
%     \item Rob's paper, Javidi - slightly different objective?
%     \item Finally, Romain's paper providing the key ideas necessary to compute estimators. 
% \end{itemize}

%There have been a ton of works on black box level set. Two main areas: GP papers, kernel bandits. In the standard MAB case, (which is equivalent to when we take the RKHS bandwith to inf), this problem has been studied by Mason et al. 

%In Table~\ref{table:lse}, we summarize the results of different LSE algorithms. In the case of the \texttt{TruVar} algorithm, we restrict to the case of homoskedastic noise and equal costs for sampling each point. In Table~\ref{table:lse-imp}, we summarize algorithms and theoretical guarantees for implicit level set estimation. 

%\subsection{Optimal Design techniques for active learning}

%RAGE, PEACE, RIPS paper. State estimator from RIPS - give a quick run through of the ideas of optimal design so that we can transition into them in the experiment section. 

\section{Explicit Level Set Estimation}

 In recent years, adaptive experimental design has arisen as a popular paradigm for active learning in structured settings, for example in linear bandits and RKHS \cite{soare2014best, fiez2019sequential, camilleri2021highdimensional}, and we adapt these ideas for the level set problem. To motivate this paradigm, in the following example we focus on the well-specified linear case where $\phi(\bx) = \bx, \widetilde{\beta}=0, h=0$. Imagine we have access to a collection of $n$-measurements $\{(\bx_i, y_i)\}_{i=1}^n$ and let $\widehat{\theta} = \arg\min_{\theta\in \R^d} \sum_{i=1}^n(y_i - \bx_i^{\top}\theta)^2$ be the least squares estimator. Standard results show that with probability greater than $1-\delta$, we have for all $\bx\in \mc{X}$ simultaneously
 \[|\bx^{\top}(\widehat{\theta} -\theta_\ast)| \leq  \|\bx\|_{\left(\sum_{i=1}^n \bx_i\bx_i^{\top}\right)^{-1}}\sqrt{\frac{2\log(2|\X|/\delta)}{n}},\]
where the additional factor of $|\X|$ in the logarithm arises from a union bound over $\X$. In particular, if our data is chosen so that for each arm $\bx\in \mc{X}$
\begin{equation}\label{ref:ci_bound}
    |\bx^{\top}\theta_* -\alpha| >  \|\bx\|_{\left(\sum_{i=1}^n \bx_i\bx_i^{\top}\right)^{-1}}\sqrt{\frac{2\log(2|\X|/\delta)}{n}},
\end{equation}
we see that $\{\bx:\bx^{\top}\widehat{\theta} > \alpha \} = \{\bx:\bx^{\top}\theta_* > \alpha \} = G_\alpha$, i.e. we have a high probability guarantee that we return the correct set of arms above the threshold. 
Letting $\lambda_x = n_x/n$ be the proportion of times we sample $\bx\in \mc{X}$, we see see that equation~\eqref{ref:ci_bound} is equivalent to
\begin{equation}
    n \geq \max_{x\in \X} \frac{\|\bx\|^2_{\left(\sum_{\bx\in \X} \lambda_x \bx\bx^{\top}\right)^{-1}}}{(\theta_\ast^{\top}\bx - \alpha)^2}.
\end{equation}
% \kevin{Mismatch of $n$ above and $N$ here.}
In particular, this implies that to achieve a good sample complexity we can minimize the right side of this expression over all possible distributions $\lambda\in \triangle_{\mc{X}}$ where $\triangle_{\X} = \{\lambda\in \mathbb{R}^{|\mc{X}|}: \sum_{x\in \X} \lambda_x = 1, \lambda_x \geq 0 \ \forall \bx\}$. Indeed as the following theorem shows, this gives a lower bound on this problem.

\begin{theorem}\label{thm:explicit_lower_bound} 
% \todo{port to RKHS with a min over $\gamma$?}
Assume $\eta_t \stackrel{iid}{\sim} {\cal N}(0, 1) \ \forall t$.
In the well-specified linear setting when $\phi(\bx) = \bx$ and $f(\bx) = \theta_\ast^{\top}\bx$, for any $\delta > 0$, any PAC-$\delta$ algorithm with stopping time $T_\delta$ that returns the set $G_\alpha$ with probability at least $1-\delta$
must satisfy 
\vspace{-1em}
\begin{align*}
    \frac{\E[T_\delta]}{\log(1/2.4\delta)} \geq 2\min_{\lambda\in \triangle_{\X}}\max_{\bx \in \X}\frac{\|\bx \|_{A(\lambda)^{-1}}^2}{(\theta_\ast^{\top}\bx -\alpha)^2}
\end{align*}
where $A(\lambda) := \sum_{\bx\in \X}\lambda_x \bx\bx^{\top}$.
\end{theorem}
\textbf{Remark.} We prove this result for completeness in the appendix using ideas from \citet{fiez2019sequential}. A similar result has appeared previously in the Appendix of \citet{degenne2020gamification}.
% \lalit{Ask romain about lower bound in RKHS - why isn't this a lower bound?}

We now operationalize this lower bound to provide an
algorithm for level set estimation that has a nearly matching upper bound. In the following sections, we will explain our algorithm and the adaptations necessary to handle the general setting of the RKHS.

% \lalit{Put an informal theorem here for the RKHS setting??}

\subsection{Algorithm}
Motivated by this lower bound, we now provide an experimental design approach. 
% In contrast to past approaches for level set estimation which design acquisition functions and then choose a single sample at a time based on posterior updates, our algorithm proceeds in rounds and plans many samples in advance. 
Ideally, we would sample according to the distribution that achieves the minimum in the lower bound in Theorem~\ref{thm:explicit_lower_bound}, however this is not possible since $\theta_{\ast}$ is not known a priori. Instead, we approximate this distribution by solving a series of designs based on the information we have thus far.
% up to that point. 

Our approach, \texttt{MELK} (Misspecified Explicit Level set via Kernelization), for the generalized RKHS setting is given in Algorithm~\ref{alg:MELK}. \texttt{MELK} proceeds in phases. 
%In each phase it computes an experimental design which it uses to separate as many points $\bx \in \X$ from the threshold $\alpha$ as possible. 
To keep track of the points it has identified so far, \texttt{MELK} maintains two sets: 1) $\widehat{G}_t$ is the set of all points that up to round $t$ have been declared as being in $G_\alpha$ by \texttt{MELK}, that is $f(\bx) > \alpha$. 2) $\widehat{B}_{t}$ is the set of all points declared as being in $G_\alpha^c$. The remaining, uncertain points are \emph{active} and in the set $\A_t$. 
Motivated by the lower bound from the linear setting, it then computes the experimental design: $\lambda_t = \arg\min_{\lambda\in \triangle_{\X}}\max_{\bx\in \A_t}\|\phi(\bx)\|_{A^{(\gamma)}(\lambda)^{-1}}^2$ with $A^{(\gamma)}(\lambda) := \sum_{\bx\in \X} \lambda_x\phi(\bx)\phi(\bx)^\top + \gamma I$ where $\gamma$ is a necessary regularization in the kernelized (infinite-dimensional) setting. 
Indeed, the number of samples taken in each round equals $N_t\approx \min_\lambda\max_{\bx\in \A_t}\tfrac{\|\phi(\bx)\|_{A^{(\gamma)}(\lambda)^{-1}}^2}{(2^{-t})^2}$ from $\lambda_{t}$. This guarantees that at the end of the round, $\A_{t+1}\subset \{\bx\in \X:|\theta_\ast^\top\bx - \alpha|\leq 2^{-(t+1)}\}$ and, we can interpret our design as an approximation to the lower bound on the points that are remaining. \texttt{MELK} declares that $\bx \in G_\alpha$ if $\htheta^T\phi(\bx) -2^{-t} \gtrsim \alpha$ and adds $\bx$ to the set $\hG_t$. Similarly, \texttt{MELK} adds $\bx$ to declares $\bx \in G_\alpha^c$ and adds $\bx$ to $\hB_t$ if $\htheta^T\phi(\bx) +2^{-t} \lesssim \alpha$. Finally, \texttt{MELK} terminates when either all arms have been added to the sets $\hG_t$ or $\hB_t$ or when $t \gtrsim \log_2(1/\widetilde{\beta})$ and it has achieved the practitioner's desired tolerance of $\widetilde{\beta}$.

%In order to obtain sufficiently precise estimates of $\theta_\ast^\top \bx$ from the $N_t$ samples, 
\texttt{MELK} leverages a robust inverse propensity scoring (RIPS) estimator introduced in \citet{camilleri2021highdimensional} and reviewed in Appendix~\ref{sec:rips}. 
Previous works in linear bandits have utilized rounding procedures for sampling followed by ordinary least squares that are not applicable in the infinite dimensional setting. Instead, the RIPS estimator appeals to an inverse propensity score estimator plus robust mean estimation.
% though we point out that it leads to an additional multiplicative dependence on $B^2$. 
% \lalit{Maybe leave a soft version of the theorem in?}
We state the guarantee of the RIPS estimator below. 
\begin{theorem}[Theorem 1, \citep{camilleri2021highdimensional}]\label{thm:robust_estimator}
Consider the model $y = \langle \phi(\bx), \theta_\ast\rangle_{\H} + \zeta_{\bx} + \eta$ for misspecification $|\zeta_{\bx}| \leq h$ where it is assumed that $|\langle \phi(\bx), \theta_\ast\rangle_{\H} + \zeta_{\bx}| \leq B$, $\E[\eta]=0$, and $\E[\eta^2]\leq\sigma^2$. Fix any finite sets $\X \subset \R^d$ and $\V \subset \H$, feature map $\phi: \R^d \rightarrow \H$, number of samples $\tau$, regularization $\gamma>0$, and distribution $\lambda\in \triangle_{\X}$. 
%If the RIPS procedure of Algorithm~\ref{alg:rips} is run with $\tfrac{\delta}{|\V|}$-robust mean estimator  $\widehat{\mu}(\cdot)$ and if $\tau \geq c_1 \log(|\V|/\delta)$ then with probability at least $1-\delta$, we have 
If $\tau \geq 2 \log(|\V|/\delta)$ then with probability at least $1-\delta$, RIPS returns $\htheta$ satisfying
\begin{align*}
    \max_{\bv \in \V} \frac{ |\langle \widehat{\theta},\bv\rangle - \langle \theta_\ast, \bv \rangle| }{\|v\|_{\A^{(\gamma)}(\lambda)^{-1}} } \leq  &2\sqrt{\gamma} \|\theta_\ast \| + 2h \\
    & + 4 \sqrt{\tfrac{(B^2 + \sigma^2)}{\tau}\log\left(\tfrac{2|\V|}{\delta}\right) }.
\end{align*}
%Moreover, $W^{(v)} = \widehat{\mu}( \{ \bv^\top A^{(\gamma)}(\lambda)^{-1} \phi(\bx_t) y_t \}_{t=1}^\tau )$ can be replaced by $\langle \widehat{\theta},\bv\rangle$ by multiplying the RHS by a factor of $2$.
\end{theorem}
% For RIPS, we leverage Catoni's estimator \citep{lugosi2019mean} for which $c_1=2$ and $c=4$ suffice. 

% Naively, this approach may seem less adaptive than the single step look-ahead policies.
% \lalit{This discussion needs to reflect the algorithm now.}
% To do so, keeps track of a set of $\A_t \in \X$ of points that its unsure about whether they belong in $G_\alpha$ or $G_\alpha^c$. 
% Intuitively, since the lower bound for explicit level-set estimation depends on a variance term $\|\bx \|_{A(\lambda)^{-1}}^2$, one should choose a design to minimize this quantity. For any $\bx$ achieving this maximum, if an algorithm can simultaneously minimize the variance while ensuring that $\bx$ never receives more than $(\theta_\ast^T\bx -\alpha)^{-2}\log(1/\delta)$ samples, then it can achieve this lower bound.  Algorithm~\ref{alg:MELK}, \texttt{MELK}, does exactly this. 
 %\footnote{melk translates to `milk' in Norwegian} 
% Let $\phi \circ \X := \{\phi(\bx_1), \cdots, \phi(\bx_{|\X|})\}$.
% \texttt{MELK} proceeds in rounds and in each round $t$

%In each round $t$, \texttt{MELK} constructs a design to identify as many points in $\A_t$ as possible. 

%Algorithm~\ref{alg:MELK} \texttt{MELK} 
%keeps track of an active set of arms $\A_t \subset \X$ of which it is unsure if they belong to $G_\alpha$ or its complement $G_\alpha^c$. 
%On these arms, . The addition of $\gamma I$ is a necessary regularization in the kernelized setting and corresponds to ridge regression. 
\newcommand{\cX}{\mathcal{X}}

\textbf{Computational Considerations.} We note briefly that while we state the optimal design in terms of the potentially infinite dimensional $\phi(\bx)$ for clarity, we never explicitly compute $\phi(\bx)$ and instead resort to the kernel trick (see Appendix~\ref{sec:kernel_compute}). Furthermore the design can be computed using first order optimization methods, such as Frank-Wolfe \cite{lattimore2020bandit,todd2016minimum}. The total computational cost of each design is $\text{poly}(|\cX|)$. Though these designs can be expensive to compute, this is done very rarely by the algorithm. In particular, for $T$ total samples drawn by \texttt{MELK}, the design is computed $O(\log_2(T))$ times leading to an overall computational cost of $O(\text{poly}(|\cX|)\log_2(T))$ for the design. By contrast, any algorithm that computes an acquisition function at every sample suffers computational complexity $\Omega(T)$ for the design. Furthermore, for Gaussian process approaches, the added cost of computing posterior means and variances leads to an overall computational cost of either $\Omega(\text{poly}(|\X|)T)$ or $\Omega(|\X|\text{poly}(T))$ depending on implementation for computing acquisition functions. Hence, when many samples are drawn, \texttt{MELK} can be significantly more efficient than past approaches. 

% \lalit{Pull computational considerations from reviews}

%In order to obtain estimates of $f(\bx)$ for each $\bx \in \A_t$, \texttt{MELK} leverages a robust inverse propensity scoring (RIPS) estimator introduced in \citep{camilleri2021highdimensional} and reviewed in Appendix~\ref{sec:rips}. 
%This allows \texttt{MELK} to draw samples from the optimal $\lambda_t$ and obtain precise estimates of $ \theta_\ast^T\phi(\bx)$ denoted $\htheta^T\phi(\bx)$. In each round $t$, the number of samples is chosen so that $|(\theta_\ast - \htheta)^T\phi(\bx)| \lesssim 2^{-t}$ for all $\bx$ in the active set $\A_t$. This allows \texttt{MELK} to declare that $\bx \in G_\alpha$ if $\htheta^T\phi(\bx) -2^{-t}\gtrsim \alpha$. If so, it adds $\bx$ to the set $\hG_t$ of points it believes are in $G_{\alpha}$. Similarly, \texttt{MELK} adds arms to a set $\hB_t$ of arms it believes are in $G_\alpha^c$. 
%The algorithm terminates when either all arms have been added to a set $\hG_t$ or $\hB_t$, either when $t \gtrsim \log_2(1/\widetilde{\beta})$ and it has achieved the practitioner's desired tolerance. It returns all points it has not proven to be below the threshold: $\X \setminus \hB_t$.

\setlength{\textfloatsep}{5pt}
\begin{algorithm}[tbh]
\caption{\texttt{MELK}: \textbf{M}isspecified \textbf{E}xplicit \textbf{L}evel set via \textbf{K}ernelization}
\label{alg:MELK}
\begin{algorithmic}[1]
\Require{Arms $\X$, $\phi$, $\sigma\geq 0$, $\delta > 0$, $\gamma\geq 0$, threshold $\alpha$, tolerance $\widetilde\beta$ }
\State{$t \leftarrow 1$, $\hG_1 \rightarrow \emptyset$, $\hB_1 \leftarrow \emptyset$, $\A_1 \leftarrow \X$}
\While{$|\hG_t \cup \hB_t| < |\X|$ and $t\leq \lceil\log_2(4/\widetilde\beta)\rceil$}
    \State{$\delta_t \leftarrow \delta/2t^2$}
    % \State{Let $\displaystyle{\lambda}_t \in \triangle_{\X}$ minimize $f(\lambda; \A_t; \gamma)$ where}
    %     $$
    %     f(\V; \gamma) = \inf_{\lambda \in \triangle_{\X}} f(\lambda; \V; \gamma) 
    %         \!=\! \!\inf_{\lambda \in \triangle_{\X}} \!\max_{\bx \in \V}\!\|\phi(\bx)\|_{(A^{(\gamma)}(\lambda))^{-1}}^2
    %     $$
    % \State{$\lambda_t\!\!\gets\!\!\min_{\lambda\in \triangle_{\X}}g(\lambda; \A_t; \gamma)$ for $g(\lambda; \V; \gamma) 
    %         \!\!\!:=\!\!\!\max_{\bx \in \V}\!\|\phi(\bx)\|_{(A^{(\gamma)}(\lambda))^{-1}}^2$}
    \State{Let $\displaystyle{\lambda}_t \in \triangle_{\X}$ minimize $g(\lambda; \A_t; \gamma)$ where}
        $$
         g(\lambda; \V; \gamma) 
            :=\max_{\bx \in \V}\!\|\phi(\bx)\|_{A^{(\gamma)}(\lambda)^{-1}}^2
        $$
        % $$
        % f(\V; \gamma) = \inf_{\lambda \in \triangle_{\X}} g(\lambda; \V; \gamma) 
        %     \!=\! \!\inf_{\lambda \in \triangle_{\X}} \!\max_{\bx \in \V}\!\|\phi(\bx)\|_{(A^{(\gamma)}(\lambda))^{-1}}^2
        % $$
    \State{$q_t \leftarrow 16 \cdot 2^{2t} g(\lambda_t;\A_t;\gamma) (B^2 + \sigma^2) \log(2t^2|\X|^2/\delta)$}\\
    \State{Set $N_t \leftarrow \left\lceil\max\left\{ q_t, 2\log(|\X|/\delta)\right\}\right\rceil$ and sample $x_{1}, \cdots, x_{N_t}$ observing noisy function values $y_1, \cdots, y_{N_t}$ according to $\lambda_t$.}
    % \State{Use Alg~\ref{alg:rips}, RIPS, with $\X$, $\V_t= \phi(\A_t), \phi, \gamma$, sampling $N_t$ measurements $x_1, \ldots, x_{N_t}$ to get $\widehat{\theta}_t$ } 
    \State{$\widehat{\theta}_t \!\gets\! \text{RIPS}(\A_t, \{A^{(\gamma)}(\lambda_t)^{-1}\phi(\bx_i)y_i\}_{i=1}^{N_t})$, Alg~\ref{alg:rips} in Appendix \ref{sec:rips}}
    \For{$\bx \in \A_t$}
        \If{$\htheta^T\phi(\bx) < \alpha -2 \cdot 2^{-t}$}
            \State{$\hB_{t+1} \leftarrow \bx$ }
            \State{$\A_{t+1} \leftarrow \A_t \backslash \{\bx\}$}
        \ElsIf{$\htheta^T\phi(\bx) > \alpha + 2 \cdot 2^{-t
        }$}
            \State{$\hG_{t+1} \leftarrow \hG_t \cup \{\bx\}$ }
            \State{$\A_{t+1} \leftarrow \A_t \backslash \{\bx\}$}
        \EndIf
    \EndFor
    % \For{$\bz \in \Z \backslash (\hG_t \cup \hB_{t+1})$ }
    %     \If{$\Y_{z}\cap \A_t = \emptyset$}
    %         \State{$\hG_{t+1} \leftarrow \hG_t \cup \{\bz\}$}        
    %     \EndIf
    % \EndFor
    \State{$t \leftarrow t + 1$}
\EndWhile
%\Return{$\hG_{t}$}
\Return{$\widehat{\mc{R}} := \X \setminus \hB_{t}$}
\end{algorithmic} 
\end{algorithm}

\subsection{Optimal sample complexity for explicit level set estimation}

Next we state \texttt{MELK}'s complexity, deferring constants and doubly logarithmic factors to the appendix.  

% Due to the misspecifity of the model, one cannot guarantee that we recover $G_\alpha$ exactly. Instead for a misspecification $h \geq 0$, we guarantee that we return a set $\hG$ satisfying $G_{\alpha - h} \subset \hG \subset G_{\alpha + h}$. Hence, we recover the true level set $G_\alpha$ up to a tolerance of our misspecification $h$. 

% \todo{for Romain tau goes to alpha and add beta tolerance ("sought gap"}
\begin{theorem}\label{thm:MELK_complex}
Fix $\delta >0$, threshold $\alpha > 0$, tolerance $\widetilde\beta$, and regularization $\gamma\geq 0$. Define $\Delta_{\min}(\alpha) := \min_{\bx \in \X} |\phi(\bx)^T\theta_\ast - \alpha|$. 
% Assume that there exists a $\theta_\ast \in \H$ such that $|f(\bx) - \theta_\ast^T\bx| \leq h$ for all $\bx \in \X$. 
Define also 
% \begin{align*}
%     &\bar\beta(\alpha) = \min \bigg\{ \beta>0 : 4(\sqrt{\gamma} \|\theta_* \| + h)\times\\
%     &\hspace{-0.5cm}\left.\left(2\!+\!\sqrt{f(\X,\left\{\phi(\bx)| \bx \in\X, |\phi(\bx)^T\theta_\ast - \alpha| \leq  \beta\right\};\gamma)}\right) \leq  \beta \right\}.
% \end{align*}
\begin{align*}
    &\bar\beta(\alpha) = \min \bigg\{ \beta>0 : 4(\sqrt{\gamma} \|\theta_* \| + h)\times\\
    &\hfill\Big(\!2\!+\!\sqrt{\min_{\lambda\in \triangle_{\X}}\max_{\bx\in\cX: |\phi(\bx)^\top\theta_\ast - \alpha| \leq  \beta}\|\phi(\bx)\|_{A^{(\gamma)}(\lambda)^{-1}}^2} \Big)\!\leq\!\beta \bigg\}\!.\!
\end{align*}
With probability at least $1-\delta$,  
\texttt{MELK} returns a set $\widehat{\mc{R}}$ at time $T_{\delta}$
such that
\begin{align*}
    & \widehat{\mc{R}} \supseteq \{\bx\in \mc{X}: f(\bx) \geq \alpha + \bar{\beta}(\alpha)\} \\
     & \text{ and }\widehat{\mc{R}} \subseteq \{\bx\in \mc{X}: f(\bx) \geq \alpha - \widetilde{\beta} -  \bar{\beta}(\alpha)\}
\end{align*}
and for any $\alpha$, $\widetilde\beta$ such that $\max(\Delta_{\min}(\alpha), \widetilde\beta) \geq \bar\beta(\alpha)$
\begin{align*}
    T_\delta \leq (B^2 + \sigma^2)&\min_{\lambda \in \tX}  \max_{\bx\in \X}  \frac{\|  \phi(\bx)\|_{A^{(\gamma)}(\lambda)^{-1}}^2}{\max\{(\phi(\bx)^T\theta_\ast - \alpha)^2, \Tilde{\beta}^2\}}\\
    &\times \log((\Delta_{\min}(\alpha)\vee \widetilde{\beta})^{-1})\log\left(|\X|\delta^{-1}\right).
\end{align*}
\end{theorem} 
% \kevin{Oh wow, I misread this theorem. The second part of $\hat R$ is easy to miss under the theorem. just formatting but we should be careful we don't get a bad line break - BLAKE: incorporated by adding space after alg}

We now contextualize the result of our theorem. In the well specified setting with $\phi(\bx) = \bx$, $h = 0$, $\widetilde{\beta} = 0$, and $\gamma = 0$ \texttt{MELK} will terminate and return $G_\alpha$ in a time
\begin{align*}
    \vspace{-0.2cm}T_\delta \!\lesssim\! (B^2\!\!+\!\sigma^2)\min_{\lambda \in \tX}  \max_{\bx\in \X}  \tfrac{\|\bx\|_{A(\lambda)^{-1}}^2}{(\bx^T\theta_\ast - \alpha)^2} \log(\Delta_{\min}^{-1})
    % \log(|\X|\delta^{-1})
    \log\!\left(\!\tfrac{|\X|}{\delta}\!\right)
\end{align*}
samples which nearly matches the rate suggested by the linear lower bound in Theorem~\ref{thm:explicit_lower_bound}.
The added factor of $\log(|\X|)$ stems from a union bound, while the dependence on $\log(\Delta_{\min}^{-1})$ is an additional overhead incurred as \texttt{MELK} builds up an estimate of the optimal sample allocation over rounds. We visualize this estimation process in Figure~\ref{fig:allocations} in the experiments.

In the more general misspecified setting when $h > 0$, we cannot expect to return $G_\alpha$ exactly and $\overline{\beta}(\alpha)$ characterizes the limit of how well one can estimate $f(\bx)$. Hence, $\bx$'s with gaps smaller than $\bar{\beta}(\alpha)$ cannot reliably be detected by \texttt{MELK}. To better understand this quantity, note that for any $\gamma'\in \R$ if we run \texttt{MELK} with $\gamma = \gamma'/T$, Lemma 2 of \citet{camilleri2021highdimensional} can be used to show  that $\bar{\beta}(\alpha) \lesssim (\sqrt{\gamma}\|\theta_\ast\| + h) \sqrt{\Gamma_T}$ 
% \kevin{Something is broken here. Something should be squared or under the square root. Romain please check.} 
where $\Gamma_T := \sup_{\lambda \in \tX} \log\det(T A^{(0)}(\lambda) + \gamma' I)$ is the \emph{maximum information gain} as defined by \citet{srinivas2009gaussian, gotovos2013active, bogunovic2016truncated}. Additionally, it can be shown that $\Gamma_T \leq d_{eff}$, where $d_{eff}$ is the \emph{effective dimension} of $\phi(\bx_1), \ldots, \phi(\bx_n) \in \H$ as defined in \citet{alaoui2014fast, derezinski2020bayesian}. 
% \begin{align*}
%     \bar{\beta}(\alpha) &\lesssim (\gamma\|\theta_\ast\| + h)\cdot \text{Trace}\left(\left(A^{(0)}(\lambda^\ast)\right)\left(A^{(\gamma)}(\lambda^\ast)\right)^{-1}\right)\\
%     & = (\gamma\|\theta_\ast\| + h)\cdot d_{eff}
% \end{align*}
% where $\lambda^\ast\in \arg\max_{\lambda\in \triangle_{\X}} \log\text{det}(A^{(\gamma)}(\lambda))$ and $d_{eff}$ is the \emph{effective dimension} of $\phi(\bx_1), \ldots, \phi(\bx_n) \in \H$ as defined in \cite{alaoui2014fast, derezinski2020bayesian}. 
In particular, to ensure that \texttt{MELK} correctly identifies all points that are at least some gap $\Delta > h$ away from the threshold, then we can choose $\gamma$ so that $\Delta > (\sqrt{\gamma}\|\theta_\ast\| + h) \sqrt{\Gamma_T}$. 
% \kevin{something is broken here}
In practice we find that $\gamma = 1/T$ works well. Finally, the user may additionally set a tolerance $\widetilde{\beta} > 0$. In this case, we err on the side of potentially returning extra arms that are not in $G_\alpha$ and show that the returned set $\widehat{\mc{R}}$ contains all $\bx$ such that $f(\bx) > \alpha +\overline{\beta}(\alpha)$ and none such that $f(\bx) < \alpha - \widetilde{\beta} - \overline{\beta}(\alpha)$. If however, a more selective criteria is desired, the following remark characterizes the output if $\widehat{G}_t$ is returned instead.

\textbf{Remark.} If \texttt{MELK} instead returns $\widehat{\mc{R}} =  \hG_t$ then 
with probability at least $1-\delta$
    $\widehat{\mc{R}} \supseteq \{\bx\in \mc{X}: f(\bx) \geq \alpha + \tilde{\beta} + \bar{\beta}(\alpha)\} \text{ and }$ and 
    $\widehat{\mc{R}} \subseteq \{\bx\in \mc{X}: f(\bx) \geq \alpha - \bar{\beta}(\alpha)\}.$

% Hence, \texttt{MELK} achieves lower bound on the sample complexity of explicit level set estimation given in Theorem~\ref{thm:explicit_lower_bound} up to logarithmic factors and is therefore optimal. The additional factor of $\log((\Delta_{\min}(\alpha)\vee \widetilde{\beta})^{-1})$ stems from a bound on the maximum number of rounds \texttt{MELK} runs for. 
% % \blake{To do: add in discussion of the Theorem}

\textbf{Contrast with Existing Approaches. } The experimental design based sampling approach is a departure from past work on level set estimation. As opposed to constructing an acquisition function and then bounding the sample complexity of the resulting algorithm as past works have done, we instead begin with an oracle sampling scheme that arises from a lower bound and attempt to design a practical sampling scheme that matches it as more data is collected. In what follows, we show that this leads to tighter results than have previously been shown in works such as \citet{gotovos2013active, shekhar2019multiscale,  bogunovic2016truncated}. 

The past  state of the art sample complexities all scale as $O(\Gamma_T \Delta_{\min}(\alpha)^{-2})$ up to log factors (cf. Thm 1 of \citep{gotovos2013active}, Cor. 3.1 of \cite{bogunovic2016truncated}, Thm 1 of \cite{shekhar2019multiscale}, etc.).
If we run \texttt{MELK} with $\gamma = \gamma'/T$ then 
% $\min_{\lambda\in \triangle_{\mc{X}}}  \max_{\bx}  \frac{\| \phi(\bx)\|_{A^{(\gamma)}(\lambda)^{-1}}^2}{(\phi(\bx)^T\theta_\ast - \alpha)^2}  \leq 3\Gamma_T \Delta_{\min}(\alpha)^{-2} $
\begin{align*}
     &\min_{\lambda\in \triangle_{\mc{X}}}  \max_{\bx}  \frac{\| \phi(\bx)\|_{A^{(\gamma)}(\lambda)^{-1}}^2}{(\phi(\bx)^T\theta_\ast - \alpha)^2} \\ 
%   &\leq 
%     \min_{\lambda \in \triangle_{\mc{X}}}  \max_{\bx}  \frac{\|  \phi(\bx)\|_{(A(\lambda)+\gamma I)^{-1}}^2}{\Delta_{\min}(\alpha)^2}\leq 3\Gamma_T
%      \Delta_{\min}(\alpha)^{-2} \\
    &{\leq 
    \min_{\lambda \in \triangle_{\mc{X}}}    \frac{\max_{\bx} \|  \phi(\bx)\|_{(A(\lambda)+\gamma I)^{-1}}^2}{\min_{\bx} (\phi(\bx)^T\theta_\ast - \alpha)^2}\leq 3\Gamma_T
     \Delta_{\min}(\alpha)^{-2} }
\end{align*}
% \kevin{I think the statement in blue makes clear why we're so much tighter}
where the final
inequality follows from Lemma 2 of \citet{camilleri2021highdimensional} and the definition of $\Delta_{\min}(\alpha)$.
Hence, as a consequence of Theorem~\ref{thm:MELK_complex}, \texttt{MELK} likewise achieves a sample complexity at least as tight as $O(\Gamma_T \Delta_{\min}(\alpha)^{-2})$, though potentially much tighter.  
Indeed, the first inequality is only tight in the pathological case when $|\phi(\bx)^T\theta_\ast - \alpha| = \Delta_{\min}(\alpha) \ \forall  \bx\in \X$. 
% \lalit{What is $\gamma'$ - in the GP case it's the noise variance. }
% In particular this highlights that the result of Theorem~\ref{thm:MELK_complex} is at least as tight as the result of \citep{gotovos2013active} but can be much tighter depending on the instance.

%We remark that while our theoretical results are tighter, the sampling policy itself is not wholly dissimilar to some previously proposed acquisition functions. Both \citep{bogunovic2016truncated, zanette2018robust} propose sampling strategies aimed and reducing the variance as quickly as possible. Our experimental design likewise attempts to minimize the variance as quickly as possible but does so by drawing many samples from a distribution to enable the use of tight confidence widths.

%with the key difference being that we design and sample according to a distribution to minimize variance, rather than sampling one arm at a time. 

% \blake{to do: add remark about last round similar to other section}
% \lalit{Need a remark about implementation.}

\section{Implicit Level Set Estimation}

In the \emph{implicit} level-set problem, for an $\epsilon \geq 0$ we seek to identify the set $G_{\epsilon} = \{\bx: f(\bx) > (1-\epsilon)f(\bx_\ast)\}$.
Note that unlike the explicit setting where the threshold $\alpha$ was a given input to the algorithm, now the equivalent notion of a threshold value $\alpha$ is equal to $(1-\epsilon)f(\bx_\ast)$, an unknown quantity since it relies on knowledge of the unknown function $f$. 
A naive strategy would be to attempt estimate $(1-\epsilon)f(\bx_\ast)$ 
% \kevin{shouldn't this just be $(1-\epsilon)f(\bx_\ast)$ without the LHS? Especially with the next sentence} 
directly and then apply explicit level-set estimation techniques using this estimated threshold value. 
Indeed, this is precisely the strategy of past works \citep{mason2020finding, gotovos2013active}.
Perhaps surprisingly however, it turns out that estimating the threshold is unnecessary and potentially wasteful. 
Towards developing lower bound to guide an experimental design, we begin with a simple but powerful observation.

\begin{lemma}\label{lem:mult_alle_sets}
In the well specified setting where $h = 0$, $\bx \in G_\epsilon \iff \forall \bx' \in\X: \theta_\ast^{\top}(\phi(\bx) - (1-\epsilon)\phi(\bx')) \geq 0$ and conversely, $\bx \in G_\epsilon^c \iff \exists \bx': \theta_\ast^{\top}(\phi(\bx) - (1-\epsilon)\phi(\bx')) < 0$.
\end{lemma}
\begin{proof}
\begin{align*}
    \bx \in G_\epsilon  &\iff \not\exists \bx': (1-\epsilon)\theta_\ast^{\top}\phi(\bx') > \theta_\ast^{\top}\phi(\bx) \\
    &\iff \forall \bx': (1-\epsilon)\theta_\ast^{\top}\phi(\bx') \leq \theta_\ast^{\top}\phi(\bx) \\
    & \iff \forall \bx':  \theta_\ast^{\top}(\phi(\bx) - (1-\epsilon)\phi(\bx)) \geq 0
\end{align*}
where the second equivalence holds by definition since $\bx_\ast$ maximizes $(1-\epsilon)\theta_\ast^{\top}\phi(\bx')$ and we have that $\theta_\ast^{\top}\phi(\bx) \!>\! (1-\epsilon)\theta_\ast^{\top}\phi(\bx_\ast)$ for any $\bx \in G_\epsilon$. The statement for $\bx\in G_\epsilon^c$ holds via the negation. 
\end{proof}

This lemma highlights that to determine if $\bx \in G_\epsilon$, one need only check if 
\[\theta_\ast^{\top}(\phi(\bx) - (1-\epsilon)\phi(\bx')) > 0 \text{ for all } \ \bx'\in \mc{X}.\] 
In particular, this does not require any estimate of the threshold $(1-\epsilon)f(\bx^\ast)$. 
% This lemma highlights the fact that declaring that $\theta_\ast^{\top}\phi(\bx) > (1-\epsilon)\theta_\ast^{\top}\phi(\bx_\ast)$ is \emph{equivalent} to a declaration about the difference between $\theta_\ast^{\top}\phi(\bx)$ and \emph{every} other $\theta_\ast^{\top}\phi(\bx')$. 
% Concretely, it is sufficient to check a series of linear inequalities to determine if $\bx \in G_\epsilon$:
% \[\theta_\ast^{\top}(\phi(\bx) - (1-\epsilon)\phi(\bx')) > 0 \text{ for all } \ \bx'\in \mc{X}.\] 
%To understand the optimality of \texttt{MILK}, we begin with a lower bound in the linear setting. 
Next, to guide our algorithm design we look to an information-theoretic lower bound. 
\begin{theorem}\label{thm:mult_lower_bound}
In the well-specified linear setting when $\phi(\bx) = \bx$ and $f(\bx) = \theta_\ast^{\top}\bx$, for any $\delta > 0$, any algorithm that returns the set $G_\epsilon$ with probability at least $1-\delta$
must satisfy 
\begin{align*}
    \tfrac{\E[T_{\delta}]}{\log(1/2.4\delta)} \!\geq\!2\!\min_{\lambda\in \triangle_{\mc{X}}}\!&\max\left\{\max_{\bz \in G_\epsilon} \max_{\bx'\in \X}\tfrac{\|\bx -(1-\epsilon)\bx'\|_{A(\lambda)^{-1}}^2}{(\theta_\ast^{\top}(\bx - (1-\epsilon)\bx'))^2},\right.\\ 
    &\left.\ \max_{\bx \in G_\epsilon^c} \min_{\bx'\in \X}\tfrac{\|\bx -(1-\epsilon)\bx'\|_{A(\lambda)^{-1}}^2}{(\theta_\ast^{\top}(\bx - (1-\epsilon)\bx'))^2} \right\}
\end{align*}
where $T_{\delta}$ denotes the random stopping time. 
\end{theorem}
Notably, the directions $\phi(\bx) - (1-\epsilon)\phi(\bx')$ naturally arise in the lower bound. This suggests an optimal sampling distribution $\lambda^\ast$ that achieves the minimum of the inequality in \ref{thm:mult_lower_bound}. As was the case in explicit level set estimation, this sampling distribution also depends on the unknown $\theta_\ast$. 

%\romain{(Conjecture) We also provide an information theoretic lower bound.

% \begin{theorem}\label{thm:implicit_lower_bound_esp0}
% In the well-specified linear setting when $f(\bx) = \theta_\ast^{\top}\phi(\bx)$, for any $\delta > 0$, any algorithm with stopping time $\tau$ that returns the set $G_\varepsilon$ with probability at least $1-\delta$
% must satisfy 
% \begin{align*}
%     \E[\tau] \geq\min_{\lambda}\max_{\bx, \bx' \in \X}\inf_{\gamma > 0} \frac{\|\phi(\bx) - (1-\epsilon)\phi(\bx') \|_{(A(\lambda)+\gamma I)^{-1}}^2}{(\theta_\ast^{\top}(\phi(\bx) - (1-\epsilon)\phi(\bx')))^2}\log(1/2.4\delta)
% \end{align*}
% where $T$ denotes the potentially random number of samples drawn and $A(\lambda) := \sum_{\bx\in \X}\lambda_x \phi(\bx)\phi(\bx)^{\top}$.
% \end{theorem}
% %}

\subsection{Algorithm}

% Note that while in general $(1-\epsilon)\phi(\bx') \not\in \phi(\X)$, we may estimate it's value by sampling $\phi(\bx')$ and scaling the reward by $1-\epsilon$. \lalit{Do you mean you want an estimator for $(1-\epsilon)\theta_{\ast}^{\top}\phi(x)$.}
Motivated by the lower bound, we propose Algorithm~\ref{alg:MILK} called \texttt{MILK} which proceeds in phases where we attempt to progressively match the optimal distribution from the lower bound as was done by \texttt{MELK} for the explicit setting. 
% Here we introduce Algorithm~\ref{alg:MILK} called \texttt{MILK} designed for the implicit setting level-set estimation problem. 
% % The algorithm centers around estimating directions $\phi(\bx) - (1-\epsilon)\phi(\bx')$. 
% \texttt{MILK} follows the same flow as \texttt{MELK} by first computing an optimal design $\lambda$, and then giving this design to Algorithm~\ref{alg:rips} which returns the RIPS estimates. 
The key difference, however is that \texttt{MILK} instead computes a design to optimally estimate $\theta_{\ast}^{\top}(\phi(\bx) - (1-\epsilon)\phi(\bx'))$ rather than $\theta_\ast^\top\phi(\bx)$ as in \texttt{MELK}.
% Whereas in \texttt{MELK}, the active set was over arms $\bx \in \X$, in \texttt{MILK}, the active set is instead over \emph{differences} of arms. 
Given active set $\mc{A} \subset \X\times \X$ of \emph{pairs} of arms define,
\[\Y^\epsilon(\A) := \{\phi(\bx) - (1-\epsilon)\phi(\bx'): (\bx,\bx') \in \A\}.\]
% The active set in round $1$ is initialized as $\A_1 = \X\times \X$ and the optimal design is computed over $\Y^\epsilon(\A_1)$. As the algorithm proceeds, pairs $(\bx, \bx')\in \A$ are removed from $\A_t$. For any given $\bx\in \X$, $\bx$ may still be present in other pairs $(\bx, \bx'')$ even if $(\bx, \bx')$ has been eliminated. 
% This reflects the fact that even if an arm has been declared as being in $G_\epsilon$ or $G_\epsilon^c$ it may be necessary to continue sampling it to check if other arms are in $G_\epsilon$. 
The active set in round $1$ is initialized as $\A_1 = \X\times \X$. 
\texttt{MILK} keeps track of sets $\hG_t\subset \X$ and $\hB_t\subset\X$ of arms it believes to be in $G_\epsilon$ and $G_\epsilon^c$ and makes use of the RIPS procedure to robustly estimate means. 
As the algorithm proceeds, in each round $t$ an optimal design is computed over remaining difference vectors in $\Y^\epsilon(\A_t)$ and the number of samples $N_t$ is sufficient to ensure that
$|(\theta_\ast - \htheta)^{\top}(\phi(\bx) - (1-\epsilon)\phi(\bx'))| \leq 2^{-t+1}$. Then for every arm that has not been added to $\hG_t$ or $\hB_t$, \texttt{MILK} does the following: 
\begin{align*}
    \text{ if } \exists \bx': \htheta^{\top}((\phi(\bx) - (1-\epsilon)\phi(\bx')) < 2^{-t}
\end{align*}
then $\bx$ is added to $\hB_t$. In our proof, we show this condition occurs if and only if there exists a $\bx'$ such that $\theta_\ast^{\top}(\phi(\bx) - (1-\epsilon)\phi(\bx')) < 0$. If this occurs, all pairs of the form $(\bx,\bx')$ or $(\bx', \bx)$, $\bx'\in \X$ are removed from $\A_t$\footnote{We assume that pairs are ordered, i.e. $(\bx,\bx') \neq (\bx', \bx)$ for $\bx \neq \bx'$.}.
Semantically, if \texttt{MILK} can ensure that $\bx$ is not in $G_\epsilon$, then $\bx$ is never sampled again. Otherwise, for any $\bx'$ if $\htheta^{\top}(\phi(\bx) - (1-\epsilon)\phi(\bx') > 2^{-t}$, the single pair $(\bx,\bx')$ is removed from $\A_t$. An arm $\bx$ is only ever added to $\hG_t$ if $\{(\bx, \bx'), \bx'\in \X\} \cap \A_t = \emptyset$ which occurs when
\begin{align*}
    \forall\bx': \exists t' \ \text{ such that } \ \htheta_{t'}^{\top}((\phi(\bx) - (1-\epsilon)\phi(\bx')) > 2^{-t'}.
\end{align*}
In our proof, we show that this occurs if and only if $\theta_\ast^{\top}(\phi(\bx) - (1- \epsilon)\phi(\bx')) > 0$ for all $\bx'\in \X$ which is both necessary and sufficient by Lemma~\ref{lem:mult_alle_sets}. Note that even if $\bx$ has been added to $\hG_t$ implying that all pairs $(\bx, \bx')$ have been removed from $\A_t$, $\bx$ may be present in other pairs $(\bx', \bx)$ which can be necessary to determine if $\bx'\in G_\epsilon$. Finally, the algorithm terminates when either every arm has been added to either $\hG_t$ or $\hB_t$ or it has reached a round $t \gtrsim \log_2(1/\widetilde{\beta})$ when the desired tolerance $\widetilde{\beta}$ is achieved.

\setlength{\textfloatsep}{0pt}
\begin{algorithm}[tbh]
\caption{\texttt{MILK}: \textbf{M}ultiplicative \textbf{I}mplicit \textbf{L}evel set via \textbf{K}ernelization}
\label{alg:MILK}
\begin{algorithmic}[1]
\Require{Arms $\X$, $\phi$, $\delta > 0$, $\epsilon > 0$, $\gamma\geq 0$, tolerance $\widetilde\beta$}
\State{$t \leftarrow 1$, $\hG_1 \rightarrow \emptyset$, $\hB_1 \leftarrow \emptyset$, $\A_1 \leftarrow \{(\bx,\bx'), \bx,\bx'\in \X\}$}
%\State{$\A_1 \leftarrow \Y^\epsilon(\X)$}
% \State{$\A_1 \leftarrow \{(\bx,\bx'), \bx,\bx'\in \X\}$}

\While{$|\hG_t \cup \hB_t| < |\X|$ and $t\leq \lceil\log_2(4/\widetilde\beta)\rceil$}
    \State{$\delta_t \leftarrow \delta/2t^2$}
    \State{Let $\displaystyle{\lambda}_t \in \triangle_{\X}$ minimize $g(\lambda;  \A_t; \gamma)$ where}
        $$
        g(\lambda, \V; \gamma) :=  \!\max_{(\bx, \bx') \in \V}\!\|\phi(\bx) - (1-\epsilon)\phi(\bx')\|_{A^{(\gamma)}(\lambda)^{-1}}^2
        $$
    % \State{Let $\displaystyle{\lambda}_t \in \triangle_{\X}$ minimize $g(\lambda;  \Y^\epsilon(\A_t); \gamma)$ where}
    %     $$
    %     f(\V; \gamma) = \inf_{\lambda \in \triangle_{\X}} f(\lambda; \V; \gamma) 
    %         \!=\! \!\inf_{\lambda \in \triangle_{\X}} \!\max_{\by \in \V}\!\|\by\|_{\left(A^{(\gamma)}(\lambda)\right)^{-1}}^2
    %     $$
    % \State{$q^{(1)}_t \gets c_1\log(|\X|/\delta)$}
    \State{$q_t \leftarrow 16 \cdot 2^{2t} g(\lambda_t; \A_t;\gamma) (B^2 + \sigma^2) \log(2t^2|\X|^2/\delta)$}\\
    \State{Set $N_t \leftarrow \left\lceil\max\left\{ q_t,  2\log(|\X|/\delta)\right\}\right\rceil$ and sample $x_{1}, \cdots, x_{N_t}$ observing noisy function values $y_1, \cdots, y_{N_t}$ according to $\lambda_t$.}
    % \State{Use Alg~\ref{alg:rips} with with sets $\X$, $\V_t = \Y^\epsilon(\A_t)$, sampling $N_t$ measurements $x_1, \ldots, x_{N_t}$ to get $\widehat{\theta}_t$}   
    \State{$\widehat{\theta}_t \!\gets\! \text{RIPS}(\Y^\epsilon(\A_t), \{A^{(\gamma)}(\lambda_t)^{-1}\phi(x_i)y_i\}_{i=1}^{N_t})$}
    \For{$(\bx,\bx')\in \A_t$}%$\phi(\bx) - (1-\epsilon)\phi(\bx') \in \A_t$}
        \If{$\htheta_t^{\top}(\phi(\bx) - (1-\epsilon)\phi(\bx')) < -2 \cdot 2^{-t}$}
            \State{$\hB_{t+1} \leftarrow \bx$}
            %\State{$\A_{t+1} \leftarrow \A_t \backslash \Y^\epsilon_{x}\cup \{\phi(\bx') - (1-\epsilon)\phi(\bx): \forall \bx'\}$}
            % \State{$\bx\text{-pairs}\!\gets\!\! \{\!(\bx,\bx')|\bx'\!\in\! \!\X\!\}\!\cup\! \{\!(\bx',\bx)|\bx'\!\in\!\X\!\}$}
            \State{$\bx\text{-pairs}\gets \{(\bx,\bx') \text{ and } (\bx',\bx)|\bx'\in\X\}$}
            \State{$\A_{t+1} \leftarrow \A_t \setminus \bx\text{-pairs}$}
            % \State{$\A_{t+1} \leftarrow \A_t \setminus \{(\bx,\bx')|\bx'\in \X\}\!\cup\!\!\{(\bx',\bx)|\bx'\in\X\}$}
        \NoThenElseIf{$\htheta_t^{\top}(\phi(\bx) - (1-\epsilon)\phi(\bx')) > 2 \cdot 2^{-t
        }$}
            \State{$\A_{t+1} \leftarrow \A_t \setminus \{(\bx,\bx')\}$}
            \If{$\{(\bx,\bx')|\bx'\in \X\}\cap \A_t = \emptyset$}
            \State{$\hG_{t+1} \leftarrow \hG_t \cup \{\bx\}$ }
        \EndIf
        \EndIf
    \EndFor
    % \For{$\bz \in \Z \backslash (\hG_t \cup \hB_{t+1})$ }
    %     \If{$\Y_{z}\cap \A_t = \emptyset$}
    %         \State{$\hG_{t+1} \leftarrow \hG_t \cup \{\bz\}$}        
    %     \EndIf
    % \EndFor
    \State{$t \leftarrow t + 1$}
\EndWhile
\Return{$\widehat{\mc{R}} := \X\setminus \hB_t$}
\end{algorithmic} 
\end{algorithm} 

\subsection{Theoretical Guarantees}

% Note that due to the misspecificity of the model, one cannot hope to correctly identify $G_\epsilon(f)$. In general, as $|f(\bx) - \theta_\ast^{\top}\phi(x)| \leq h \ \forall\bx\in \X$ the best we may hope to achieve is to correctly classify all points with means greater than $O(h)$ away from the threshold of $(1-\epsilon)f(\bx_\ast)$ where $\bx_\ast$ maximizes $f$. 

Next we state \texttt{MILK}'s complexity, again deferring constants and doubly logarithmic factors to the appendix.  
% As in the case of Theorem~\ref{thm:MELK_complex}, we cannot guarantee exact recovery of $G_\epsilon$ due to the misspecification $h \geq 0$ but instead guarantee that we identify the set up to a tolerance defined by $h$. 

% The set that one can hope to identify is $G^\phi_\epsilon(f) := \{\bx \in \X: \langle \phi(\bx), \theta^\ast\rangle_{\H} \geq \max_{\bx'}\langle \phi(\bx'), \theta^\ast\rangle_{\H}- \epsilon\}$. 
% \begin{lemma}
% An algorithm identifying $G^\phi_{\epsilon}(\theta_\ast)$ will only return elements in $G_{\epsilon+2h}(\theta_\ast)$ and will return all the elements of $G_{\epsilon-2h}(\theta_\ast)$
% \end{lemma}
% \begin{proof}
% It is easy to check that $G_{\epsilon-2h}(\theta_\ast) \subset G^\phi_{\epsilon}(\theta_\ast) \subset G_{\epsilon+2h}(\theta_\ast)$.
% \end{proof}
% Without loss of generality, consider the model $y = \langle \phi(\bx), \theta^\ast\rangle_{\H} + \zeta + \xi$ where is assumed that $|y| \leq B$, $|\zeta| \leq h$ and $\E[\xi^2]\leq\sigma^2$. 
% Let $\Y_{z} := \{\phi(\bz') - \phi(\bz) \in \Y(\Z)\}$, the set of all differences with respect to $\phi(\bz)$ where we assume $\Y(\Z)$ includes all $|\Z|^2$ directions $\phi(\bz') - \phi(\bz)$ and $\phi(\bz) - \phi(\bz')$. 
% \todo{for Romain alpha goes to beta and add beta tolerance ("sought gap" + update algos (algo 1 and 3)}

\begin{theorem}\label{thm:MILK_complex}
Fix $\delta >0$, $\epsilon > 0$, tolerance $\widetilde\beta$, and regularization $\gamma>0$. 
Define $\Delta_{\min}(\epsilon) = \min_{\bx}|\theta_\ast^{\top}(\phi(\bx) - (1-\epsilon)\phi(\bx^\ast))|$. 
% Define the quantities $\Delta_{\min}^{\text{Above}}(\epsilon) = \min_{\bx \in G_\epsilon}\min_{\bx'}  \theta_\ast^{\top}(\phi(\bx) - (1-\epsilon)\phi(\bx')) $ and $\Delta_{\min}^{\text{Below}}(\epsilon) = \min_{\bx \in G_\epsilon^c}\max_{\bx': (\phi(\bx) - (1-\epsilon)\phi(\bx'))^{\top}\theta_\ast < 0}  (\phi(\bx) - (1-\epsilon)\phi(\bx'))^{\top}\theta_\ast$, and $\Delta_{\min}(\epsilon) = \min\left\{\Delta_{\min}^{\text{Above}}(\epsilon), \Delta_{\min}^{\text{Below}}(\epsilon) \right\}$.
% \begin{align*}
%   \Delta_{\min}(\epsilon) &:= \min \left\{\Delta_{\min}^{\text{Good}}(\epsilon), \Delta_{\min}^{\text{Bad}}(\epsilon) \right\}\\
%   &= \min \left\{ \min_{\bz \in G^\phi_\epsilon}\min_{\bz'}\epsilon - (\phi(\bz') - \phi(\bz))^{\top}\theta_\ast, \min_{\bz \in (G^\phi_\epsilon)^c} \max_{\bz': (\phi(\bz') - \phi(\bz))^{\top}\theta_\ast > \epsilon} (\phi(\bz') - \phi(\bz))^{\top}\theta_\ast - \epsilon\right\}.
% \end{align*}
Define also 
% \begin{align*}
%     &\bar \beta(\epsilon) = \min \bigg\{ \beta>0 : 4(\sqrt{\gamma} \|\theta_* \| + h)\times\\
%     &\left.\left(2+ \sqrt{f(\X,\left\{\by \in \Y^\epsilon(\X \times \X): |\by^{\top}\theta_\ast| \leq  \beta\right\};\gamma)}\right) \leq  \beta \right\}.
% \end{align*}
% \begin{align*}
%     &\bar \beta(\epsilon) = \min \bigg\{ \beta>0 : 4(\sqrt{\gamma} \|\theta_* \| + h)\times\\
%     &\hspace{-1em}\left.\left(\!2\!+\!\!\!\! \sqrt{\min_{\lambda\in \triangle_{\X}}\hspace{-1em}\max_{\underset{|\theta_\ast^\top(\phi(\bx - (1-\epsilon)\phi(\bx'))| \leq  \beta}{(\bx, \bx')\in\cX\times \cX} }\hspace{-1em}\|\theta_\ast^\top(\phi(\bx - (1-\epsilon)\phi(\bx'))\|_{A^{(\gamma)}(\lambda)^{-1}}^2}\right)\!\! \leq\!\!  \beta\! \right\}.
% \end{align*}
\begin{align*}
    &\bar \beta(\epsilon) = \min_{\beta>0} \left\{  4(\sqrt{\gamma} \|\theta_* \| + h)\left(2\!\! +\!\! \sqrt{\min_{\lambda\in \tX} \nu(\lambda, \beta)}\right)\leq \beta\right\},\\
    &{\nu(\lambda, \beta) :=}\hspace{-2em} \max_{\underset{|\theta_\ast^\top(\phi(\bx) - (1-\epsilon)\phi(\bx'))| \leq  \beta}{(\bx, \bx')\in\cX\times \cX} }\hspace{-1.5em} \|\theta_\ast^\top(\phi(\bx) - (1-\epsilon)\phi(\bx'))\|_{A^{(\gamma)}(\lambda)^{-1}}^2.
\end{align*}
With probability $1-\delta$, \texttt{MILK} returns a set $\widehat{\mc{R}}$ at a time $T_\delta$ such that 
\begin{align*}
    &\widehat{\mc{R}}\supseteq \{\bx \in \X: f(\bx) \geq (1-\epsilon)f(\bx_\ast) + \bar\beta(\epsilon)\} \text{ and } \\
    &\widehat{\mc{R}} \subseteq \{\bx \in \X: f(\bx) \geq (1-\epsilon)f(\bx_\ast) - \widetilde\beta - \bar\beta(\epsilon)\}
\end{align*}
and for any $\epsilon$, $\widetilde\beta$ such that $\max(\Delta_{\min}(\epsilon), \widetilde\beta) \geq \bar\beta(\epsilon)$
\begin{align*}
    T_\delta \leq & 
     \text{\footnotesize{\((B^2 \!+\! \sigma^2)\)}} H^\texttt{MILK}(\theta_\ast) \text{\footnotesize{\(\log_2((\Delta_{\min}(\epsilon)\vee \widetilde{\beta})^{-1})\)}}\text{\footnotesize{\(\log\)}}\left(\tfrac{|\X|}{\delta}\right)
\end{align*}
for $H^\texttt{MILK}(\theta_\ast) = \underset{\lambda \in \tX}{\min}\left\{H_\lambda^{\texttt{MILK-}G_\epsilon}(\theta_\ast)\vee H_\lambda^{\texttt{MILK-}G_\epsilon^c}(\theta_\ast)\right\}$, 
% where $H^\texttt{MILK}(\theta_\ast) = \underset{\lambda \in \tX}{\min}\max\left\{H_\lambda^{\texttt{MILK-}G_\epsilon}(\theta_\ast), H_\lambda^{\texttt{MILK-}G_\epsilon^c}(\theta_\ast)\right\}$, 
where
\begin{align*}
    H_\lambda&^{\texttt{MILK-}G_\epsilon}(\theta_\ast) :=\\
    &\max_{\bx \in G_\epsilon}\max_{\bx'\in \X} \frac{\|\phi(\bx) - (1-\epsilon)\phi(\bx')\|_{A^{(\gamma)}(\lambda)^{-1}}^2}{\max\{((\phi(\bx) - (1-\epsilon)\phi(\bx'))^{\top}\theta_\ast)^2, \widetilde \beta^2\}},
    \vspace{-1em}
\end{align*}    
\begin{align*}
    \text{ and }&H_\lambda^{\texttt{MILK-}G_\epsilon^c}(\theta_\ast) :=\\
    &\vspace{-1em}\max_{\bx \in G_\epsilon^c}\max_{\bx'} \frac{\|\phi(\bx) - (1-\epsilon)\phi(\bx')\|_{A^{(\gamma)}(\lambda)^{-1}}^2}{\max\{((\phi(\bx) - (1-\epsilon)\phi(\bx_\ast))^{\top}\theta_\ast)^2, \widetilde \beta^2\}}.
\end{align*}
% Further, for any $\epsilon$ such that $\Delta_{\min}(\epsilon) \geq \max\{\Tilde{\beta},\bar{\beta}(\epsilon)\}$, \texttt{MILK} returns a set  $\hG_\epsilon$ which satisfies 
% $\{\bx: f(\bx) > (1-\epsilon)f(\bx_\ast) - 2h\} \subset \hG_\epsilon \subset \{\bx: f(\bx) > (1-\epsilon)f(\bx_\ast) + 2h\}$. 
\end{theorem}

The statement of Theorem~\ref{thm:MILK_complex} for \texttt{MILK} is similar that of \ref{thm:MELK_complex} for \texttt{MELK}. In the well specified case when $\widetilde{\beta} = 0$, \texttt{MILK} returns $G_\epsilon$ exactly at a time $T_\delta$ that satisfies
\begin{align*}
    T_\delta \lesssim (B^2 \!+\! \sigma^2) H^\texttt{MILK}(\theta_\ast) \log_2(\Delta_{\min}(\epsilon))\log\left(|\X|\delta^{-1}\right)
\end{align*}
In this case, however, $H^\texttt{MILK}(\theta_\ast)$ is a maximum of two different complexity terms. $H_\lambda^{\texttt{MILK-}G_\epsilon}$ represents the complexity of identifying all $\bx \in G_\epsilon$. Similarly, $H_\lambda^{\texttt{MILK-}G_\epsilon^c}$ represents the complexity of identifying all $\bx\in G_\epsilon^c$. Similar to the explicit setting, in the misspecified case when $h > 0$, $\overline{\beta}(\epsilon)$ similarly represents the limit of how well we can estimate $f(\bx)$ for any $\bx \in \X$ and $\widetilde{\beta}$ allows for an additional tolerance such that \texttt{MILK} detects all $\bx$ for which $f(\bx) > (1-\epsilon)f(\bx_\ast)+ \overline{\beta}(\epsilon)$ and none worse than $f(\bx) < (1-\epsilon)f(\bx_\ast) - \overline{\beta}(\epsilon) - \widetilde{\beta}$. The following remark addresses the setting where \texttt{MILK} returns $\widehat{G}_t$ instead. 

\textbf{Remark:} If the algorithm instead returns $\widehat{\mc{R}} = \hG_t$, then with probability at least $1-\delta$
\begin{align*}
    &\widehat{\mc{R}} \supseteq\{\bx \in \X: f(\bx) \geq (1-\epsilon) f(\bx_\ast) +\widetilde{\beta} + \bar{\beta}(\epsilon) \} \text{ and }\\
    &\widehat{\mc{R}} \subseteq \{\bx \in \X: f(\bx) \geq (1-\epsilon)f(\bx_\ast) -\bar\beta(\epsilon)\}.
\end{align*}

% each $\bx \in \hG_t$ satisfies $f(\bx) \geq (1-\epsilon)f(\bx_\ast) +\widetilde{\beta} - \bar{\beta}(\epsilon)$. Furthermore, $\{\bx \in \X: f(\bx) \geq (1-\epsilon)f(\bx_\ast) + \widetilde{\beta} + \bar{\beta}(\epsilon)\} \subset \hG_t$.

% \textbf{Remark:} If the algorithm instead returns $\hG_t$, then each $\bx \in \hG_t$ satisfies $f(\bx) \geq (1-\epsilon)f(\bx_\ast) +\widetilde{\beta} - \bar{\beta}(\epsilon)$. Furthermore, $\{\bx \in \X: f(\bx) \geq (1-\epsilon)f(\bx_\ast) + \widetilde{\beta} + \bar{\beta}(\epsilon)\} \subset \hG_t$.

% The above theorem guarantees that for $\epsilon$ satisfying a constraint stemming from the misspecification parameter $h$ and the regularization $\gamma$, we correctly identify the set $G_\epsilon$ up to a tolerance of $O(h)$. In the well specified setting where $h = 0$, by choosing $\widetilde\beta < \Delta_{\min}$, this ensures that we recover $G_\epsilon$ exactly. 
\textbf{Comparison with the Lower bound}

The complexity term $H^\texttt{MILK}(\theta_\ast)$ naturally breaks into two terms.  $H^{\texttt{MILK-}G_\epsilon}(\theta_\ast)$ represents the complexity of finding arms in $G_\epsilon$ and it matches a corresponding term in the lower bound. 
$H^{\texttt{MILK-}G_\epsilon^c}(\theta_\ast)$ represents the complexity of removing arms in $G_\epsilon^c$ but is slightly different than the term in the lower bound. 
As a consequence of Theorem 4.1 of \citet{mason2020finding} however, one can show the term given in the lower bound for $\bx \in G_\epsilon^c$ is not achievable except asymptotically as $\delta \rightarrow 0$ in general. 
Instead, the problem of implicit level set estimation reduces to the problem of all $\epsilon$-good arm identification in multi-armed bandits studied by \citet{mason2020finding} when $\phi(\bx) = \bx$, $h = 0$, and $\bx_i = e_i$.
We show in the appendix that \texttt{MILK}'s sample complexity matches the optimal finite time rate up to logarithmic factors as shown in \citet{mason2020finding}.

\textbf{Contrast with Existing Results}

As was shown in the explicit setting, we can show that the sample complexity bound in Theorem~\ref{thm:MILK_complex} improves on the current state of the art. Take $\gamma = \gamma' / T$ for any $\gamma'\in \R$. Then we may bound $ H^{\texttt{MILK}-G_\epsilon}(\theta_\ast)$ as
\begin{align*}
    &\min_{\lambda \in \tX}  \max_{\bx, \bx'} \ \left\{ \frac{\|\phi(\bx) - (1-\epsilon)\phi(\bx')\|_{A^{(\gamma)}(\lambda)^{-1}}^2}{((\phi(\bx) - (1-\epsilon)\phi(\bx'))^{\top}\theta_\ast)^2}\ \right\}\\
    & \stackrel{(a)}{\leq} 2\min_{\lambda \in \tX}  \max_{\bx, \bx'} \left\{ \frac{(1-\epsilon)^2\|\phi(\bx) - \phi(\bx')\|_{A^{(\gamma)}(\lambda)^{-1}}^2 }{((\phi(\bx) - (1-\epsilon)\phi(\bx'))^{\top}\theta_\ast)^2}\ \right.\\
    &\hspace{2.5cm} \vee\left. \frac{ \epsilon^2\|\phi(\bx)\|_{A^{(\gamma)}(\lambda)^{-1}}^2}{((\phi(\bx) - (1-\epsilon)\phi(\bx'))^{\top}\theta_\ast)^2}\ \right\}\\
    &\stackrel{(b)}{\leq} 2\frac{(1+\epsilon)^2}{\Delta_{\min}(\epsilon)^2}\min_{\lambda \in \tX}  \max_{\bx, \bx'} \ \left\{ \|\phi(\bx') - \phi(\bx)\|_{A^{(\gamma)}(\lambda)^{-1}}^2 \right.\\
    &\hspace{4.5cm}\left.\vee \| \phi(\bx)\|_{(A^{(\gamma)}(\lambda))^{-1}}^2\ \right\} \\
    &\leq \frac{4(1+\epsilon)^2}{\Delta_{\min}(\epsilon)^2}\min_{\lambda \in \tX}  \max_{\bx} \|\phi(\bx)\|_{A^{(\gamma)}(\lambda)^{-1}}^2 \\
    % &\stackrel{(c)}{\leq} \frac{6(1+\epsilon)^2}{\Delta_{\min}(\epsilon)^2}\sup_{\lambda \in \tX} \log\det(T K_\lambda + \gamma I) = \frac{6(1+\epsilon)^2}{\Delta_{\min}(\epsilon)^2}\Gamma_T
    &\stackrel{(c)}{\leq}  \frac{6(1+\epsilon)^2}{\Delta_{\min}(\epsilon)^2}\Gamma_T
\end{align*}
where $(a)$ follows by the triangle inequality, $(b)$ by definition of $\Delta_{\min}(\epsilon)$ and $(c)$ follows by Lemma 2 of \citet{camilleri2021highdimensional}. A similar computation follows for $H^{\texttt{MILK}-G_\epsilon^c}(\theta_\ast)$
Hence, the sample complexity of \texttt{MILK} is at most $O(\Gamma_T\Delta_{\min}^{-2})$ though it can be much tighter
as inequality $(b)$ is tight only in the worst case when all gaps are equal. In particular, the result of Theorem~\ref{thm:MILK_complex} is tighter than Theorem 2 of \citet{gotovos2013active}. 
\section{Experiments}\label{sec:experiments}
\begin{figure*}{}
    \centering
    \vspace{-1em}
    \includegraphics[width=0.6\textwidth]{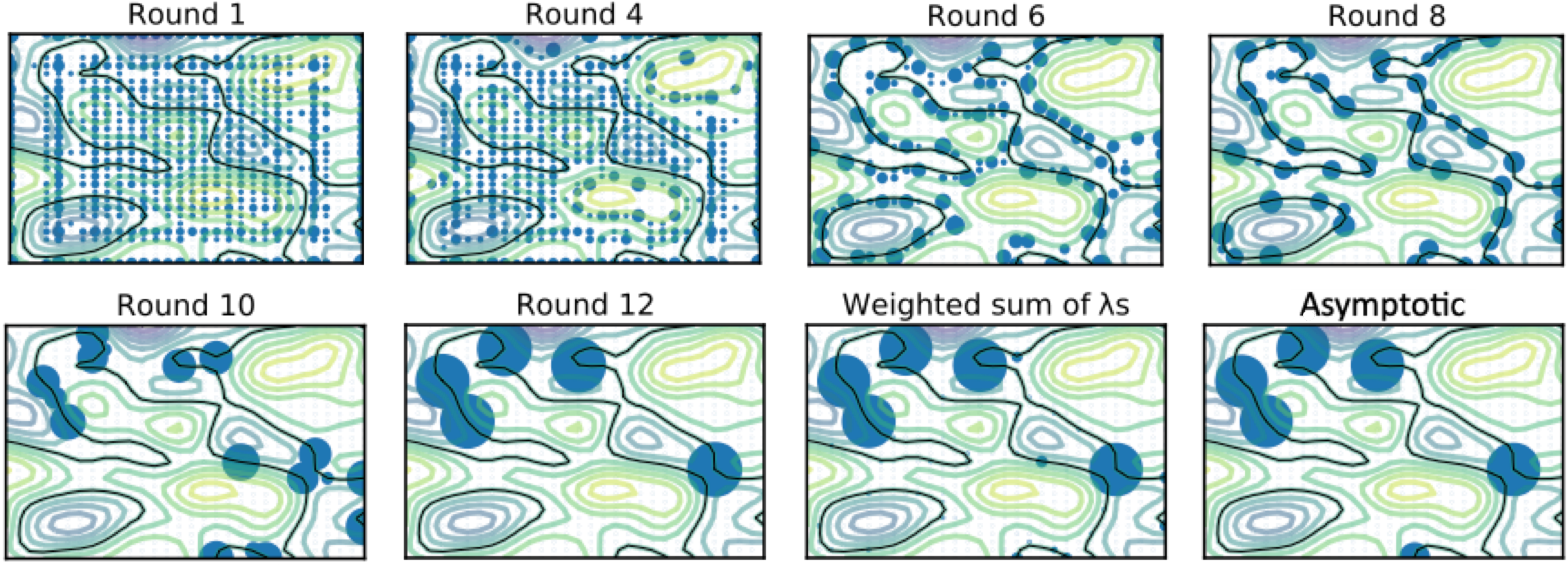}
    \vspace{-0.75em}
    \caption{Allocations across rounds for a function $f(x, y)$ with a threshold of $\alpha = 0$ shown in black. \label{fig:allocations}}
    \vspace{-1em}
\end{figure*}

\begin{figure}
\centering
% \begin{subfigure}{0.22\textwidth}
%   \centering
%   \includegraphics[width=\linewidth]{img/soare_explicit_v2.pdf}
%   \caption{Linear, explicit}
%   \label{fig:lin_explicit}
%     \end{subfigure}
\begin{subfigure}{.18\textwidth}
  \centering
  \hspace{-1em}
  \includegraphics[width=\linewidth]{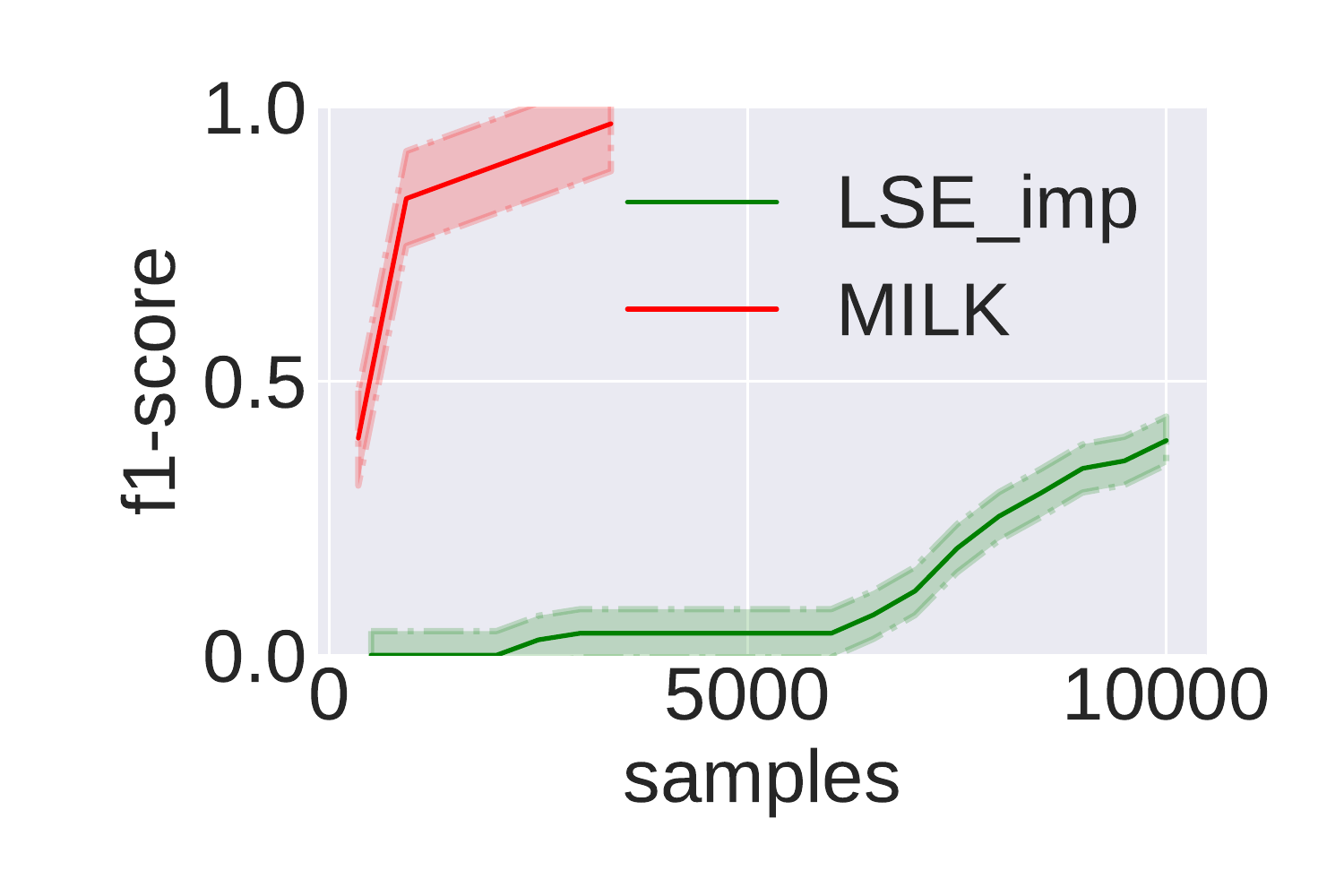}
\caption{Implicit}
\label{fig:lin_implicit}
\end{subfigure}
\hspace{-1.9em}
\begin{subfigure}{.155\textwidth}
  \centering
  \includegraphics[width=\linewidth]{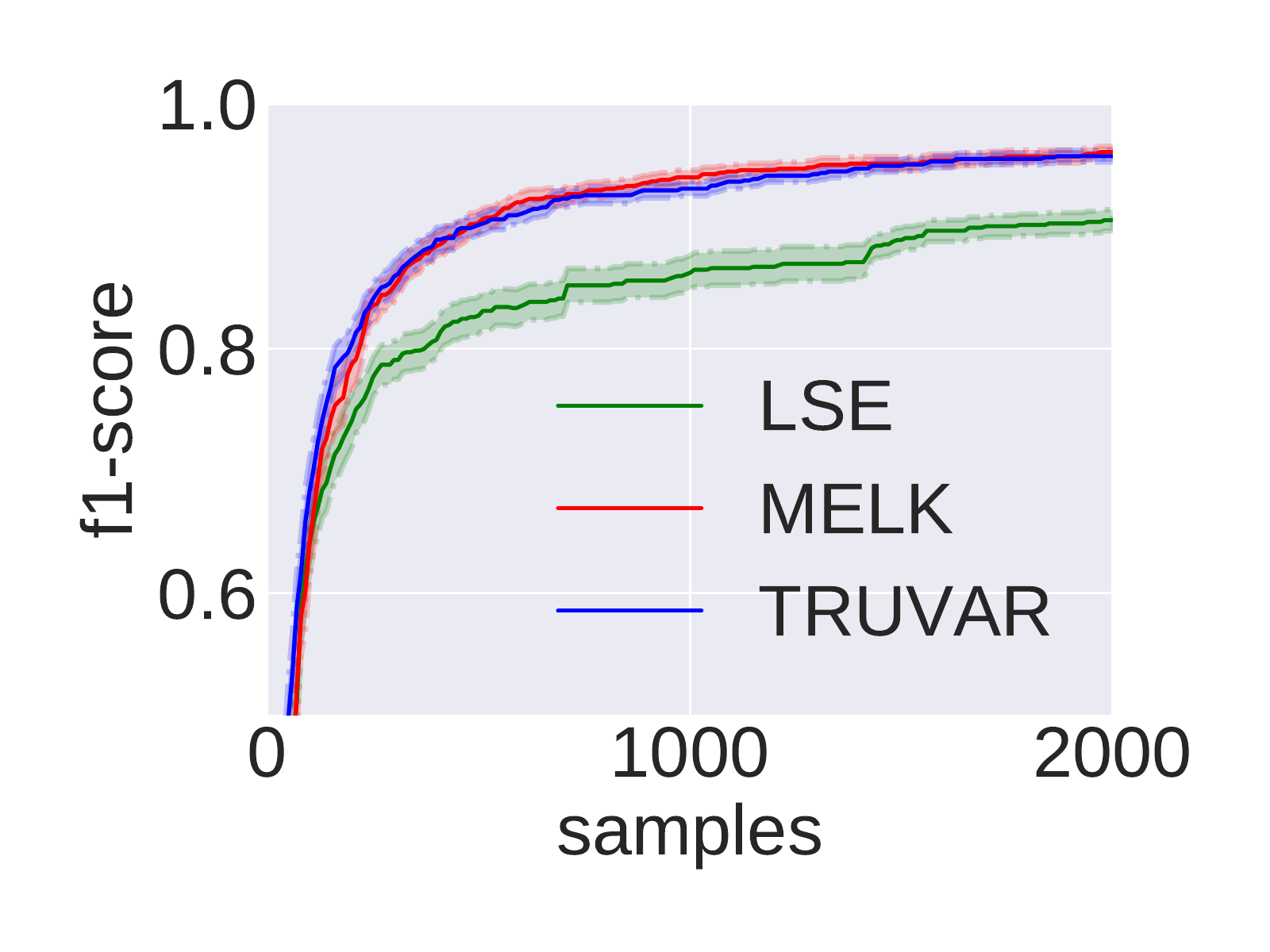}
\caption{GP, $\ell = 0.05$}
\label{fig:rkhs_well_spec}
\end{subfigure}%
\hspace{-0.8em}
\begin{subfigure}{0.17\textwidth}
  \centering
  \includegraphics[width=\linewidth]{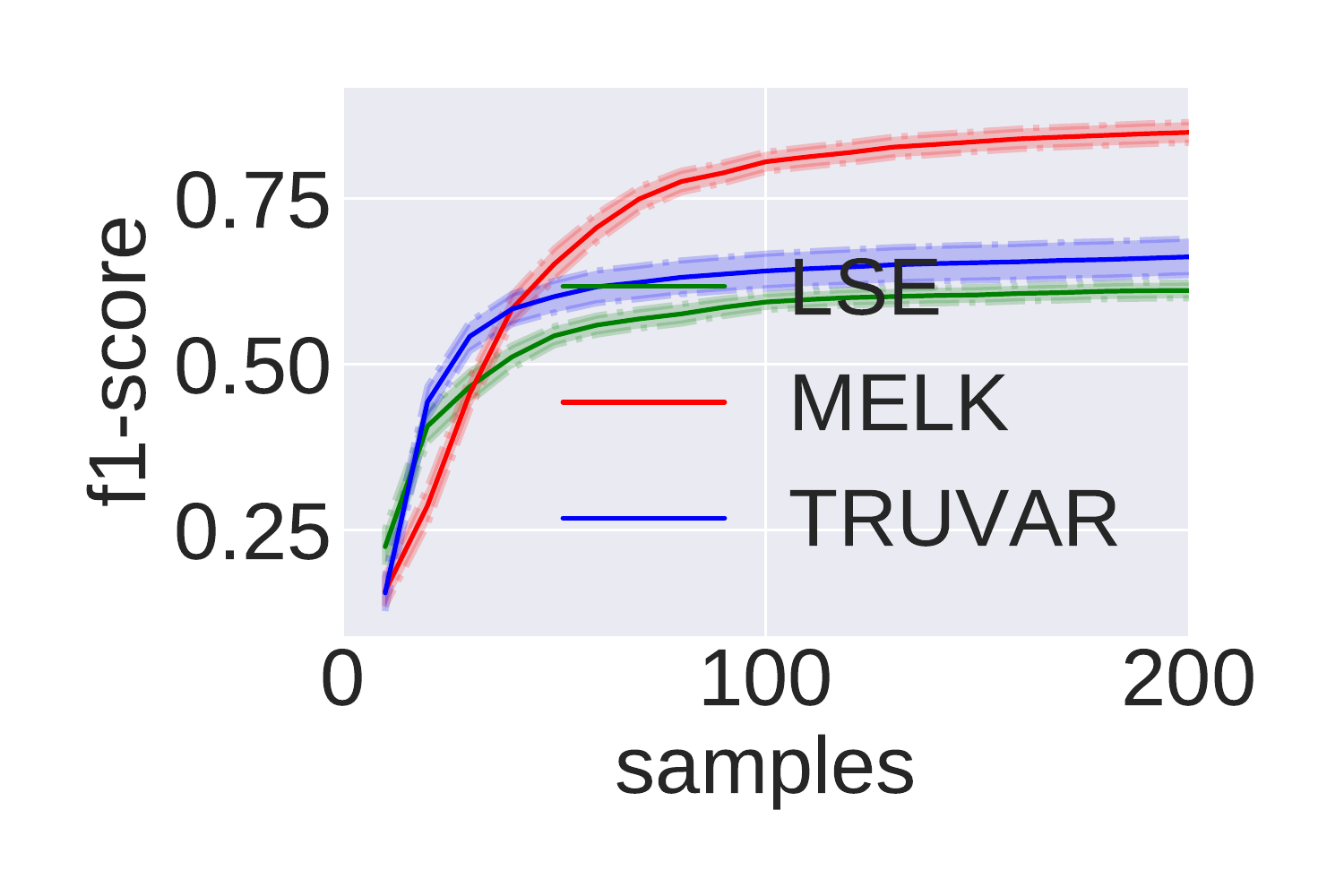}
    \caption{Cosine}
    \label{fig:rkhs_cosine}
\end{subfigure}
\caption{Performance of \texttt{MELK} and \texttt{MILK} versus Gaussian process baseline.}
\vspace{1em}
\label{fig:main_experiment}
\end{figure}

In this section, we compare our algorithms to existing baselines in the literature. Additional details of these methods and our experiments are in the Appendix.

\textbf{Warm-Up: Optimal Sampling.} 
In Figure~\ref{fig:allocations} we illustrate the sampling behavior of \texttt{MELK}.
We let $\mc{X} = \{(\frac{i}{30},\frac{j}{30})\}_{i,j=1}^{30}$ and considered the squared exponential kernel $k(\bx,\bx') = \exp(-\|\bx-\bx'\|^2/2\ell^2)$ with parameter $\ell = 0.1$. 
We also chose $\theta_*\sim \mc{N}(0,I_{900})$ and show a contour plot of $f(\bx) = \theta_*^{\top}\phi(\bx)$.
The black curve represents the boundary of the $\alpha = 0$ level set. 
We plot the sample allocations as the algorithm progresses (taking $\gamma = 0$). 
The initial distribution is mostly uniform with several sampling modes. 
In later rounds, the points nearest to the boundary of the level set, given by the black curve are sampled, and eventually, only the points with the smallest gaps (the most difficult regions) receive samples. 
As the number of samples in round $t$ is proportional to $2^{2t}$, we compute the sum of the designs weighted by the $2^{2t}$ to show the overall sampling design. 
Additionally, we plot the asymptotic allocation suggested by Theorem~\ref{thm:explicit_lower_bound}, namely $\lambda_\ast = \arg\min_{\lambda}\max_{x\in X} \|\phi(\bx)\|_{A^{(\gamma)}(\lambda)^{-1}}^2/(\theta_\ast^\top \phi(\bx) - \alpha)^2$. In particular, the weighted sum of the designs taken by \texttt{MELK} is nearly identical to $\lambda_\ast$.

\textbf{Gaussian Process Level Set Estimation.} For our main empirical evaluation, we focused on the Gaussian Process setting for the explicit level set problem. In the explicit level-set case we compare to LSE~\cite{gotovos2013active} and TruVar~\cite{bogunovic2016truncated}. 
We drew a function $f:[0,1]\rightarrow \mathbb{R}$ from the Gaussian process $\mc{N}(0,k(\bx,\bx'))$ where the kernel is a squared exponential kernel with parameter $\ell = .05$ and $[0,1]$ was uniformly discretized into 200 points. 
We assumed that the noise variance was $\sigma^2 = 1$ (high noise) and the threshold was chosen so that 10\% of the function values were above it.
In this setting, we implement a batched version of \texttt{MELK} that draws a fixed batch size of samples each round (namely 10) and then recomputes the design.
This reflects the practical constraint that experimenters may wish to collect a fixed number of samples at a time rather than a potentially growing amount.
To provide a fair comparison to the GP-based methods, we computed a posterior distribution on $f$ in each round. 
For each point we replaced our theoretically justified confidence intervals in the RKHS setting with confidence intervals arising from the posterior, namely $\hat{\mu}_t(\bx) \pm \beta^{1/2} \hat{\sigma}_t(\bx)$ where $\hat{\mu}_t, \hat{\sigma}_t$ are the posterior mean and standard deviations respectively. 
As in past works, we take $\beta^{1/2} = 3$ as theoretically justified choices of $\beta$ (eg. Theorem 1 of \cite{srinivas2009gaussian}) tend to be overly conservative. 
We also took $\gamma$ dropping like $1/i$ on the $i$-th round we computed the design. 
We ran 25 repetitions drawing a new choice of $f$ each run. Figure~\ref{fig:rkhs_well_spec} shows the average F1 score of the set of points each algorithm declares to be in $G_\alpha$ respectively with bars denoting $1$ standard error. 
Our algorithm performs very similarly to \texttt{TruVar} - an algorithm whose acquisition function samples in a way to reduce the average variance, unlike our method which tries to reduce the maximum variance.

Our second comparison is in Figure~\ref{fig:rkhs_cosine}: we took $f(x) = \cos(8\pi \bx)$, $\ell = .1$, $\sigma = .2$ (low noise regime) and chose the threshold so that $30\%$ of points were above it. We then considered 700 points uniformly in $[0,1]$.  In the appendix, we vary the underlying parameters of $\ell, \sigma^2$ to demonstrate the performance of these algorithms in different regimes.

%To evaluate the empirical performance of
%\texttt{MELK} and \texttt{MILK}, we compare them to state of 
%art Gaussian process baselines: \texttt{LSE} for explicit level set estimation and \texttt{LSE-imp} for implicit level set estimation. Both algorithms first appeared in \citet{gotovos2013active} and have been extended by \citet{bogunovic2016truncated, zanette2018robust, iwazaki2020bayesian}.
% \rob{Might be a good idea to remind the reader that LSE is a state-of-the-art Gaussian Process/Bayesian Optimization approach to level set estimation, and maybe also remind reader about other follow up papers that built on ideas in the gotovos thesis.} 
\textbf{Linear Implicit Case.} We additionally compare against \texttt{LSE-imp} in the linear setting where $\phi(\bx) = \bx$ on a benchmark example from the linear bandits literature designed to test the effectiveness of adaptive sampling algorithms \cite{soare2014best}. For $\bx_1, \cdots, \bx_{n} \in \R^{d}$, we take $\bx_1 = \bx_\ast = \theta_\ast = e_1$ and $\bx_2 = e_2$. The remaining $\bx_3, \cdots, \bx_n$  are set so that their first two coordinates are $\cos(\pi/4(1 + \xi))e_1$ and $\sin(\pi/4(1 + \xi))e_2$ for $\xi \sim \text{Unif}(-.2, .2)$. We set the threshold $\alpha = 0.5$, $n=100$, and $d=25$. 
Though it is far below $\alpha$, sampling arm $\bx_2$ provides the most information about which arms exceed the threshold. 
In this setting, we ran both algorithms with the exact confidence intervals as specified by their respective theoretical guarantees leading to large sample complexities, and we include further details in the appendix. Indeed, we see in \ref{fig:lin_implicit} that \texttt{MILK} outperforms \texttt{LSE-imp}. 
% We plot the average F1 score of the set of points each algorithm declares to be in $G_\alpha$ and $G_{\epsilon}$ respectively computed over $10$ independent trials with bars denoting $1$ standard error.
% We focus on \texttt{LSE} for the explicit setting as our primary baseline as \citet{zanette2018robust} notes that \texttt{TruVar} behave similarly to \texttt{LSE} in the setting we consider that the noise is homoscedastic. Additionally, \texttt{LSE} and \texttt{RMILE} have been shown to exhibit similar empirical performance with respect to F1 score after the initial samples \cite{zanette2018robust} since we consider longer sample horizons. 
% We also compare against the performance of \texttt{LSE} in the more general well specified and misspecified RKHS settings. In this setting, we implement a batched version of \texttt{MELK} that draws a fixed batch size of samples each round and then recomputes the design. The reflects the practical constraint that experimenters may wish to collect a fixed number of samples at a time. 
% In both cases we took $100$ points in $[0, 1]$. In the well-specified case shown in Figure~\ref{fig:rkhs_well_spec} and a random function $f$ from an RBF kernel. In the misspecified case in Figure~\ref{fig:rkhs_cosine}, we chose the function $f(\bx) = \cos(2\pi \bx)$. Again we plot the average F1 score with error bars denoting the standard error over 36 independent trials. 

\section{Conclusion}

In this work, we provide the first instance optimal algorithms for explicit and implicit level set estimation and provide theoretical and empirical justification for our algorithms. 
\bibliographystyle{apalike}
\bibliography{refs}

\clearpage
\onecolumn 
\appendix
\tableofcontents
\addcontentsline{toc}{section}{Appendices} % Add
% \textbf{Typos in the original main draft}

% An assumption for Theorem~\ref{thm:MELK_complex} was omitted that $\max(\Delta_{\min}(\alpha), \widetilde\beta) \geq \bar\beta(\alpha)$ for the sample complexity guarantee on line 228 but is not necessary for the containment on line 227. Additionally, $\bar\beta(\alpha)$ requires a leading constant of $4$ as opposed to 2. A similar assumption was also omitted from the statement of Theorem~\ref{thm:MILK_complex}. The sample complexity guarantee requires the additional assumption that $\max(\Delta_{\min}(\epsilon), \widetilde\beta) \geq \bar\beta(\epsilon)$. Additionally, the definition of $\bar\beta(\epsilon)$ requires a leading constant of $4$. These changes have been edited in the submitted supplement. 

\section{Impacts and Limitations}

Active learning uses a design objective to drive a sampling policy. In the simplest cases of active learning, such as regret minimization in standard multiarmed bandits, the relatively simple and unstructured setting leads to simple and easy to interpret sampling rules. For instance, the famed UCB algorithm simply forms confidence widths and pulls the arm with the largest upper bound. The transparency of this sampling rule makes UCB and algorithms like it inherently easy to diagnose and monitor in real time. 
% Monitoring active algorithms is crucial to guard against bias and other issues that occur when using algorithms in the wild. 
For past algorithms in level set estimation, the acquisition functions merit easy oversight. By contrast, our work introduces optimal design to the area of level set estimation. As we show in our work, this can lead to improved sample complexity both theoretically and empirically. However, as the sampling distributions are based on a more complicated objective, how the algorithm chooses which data to collect is less immediately obvious or intuitive. This may make detecting issues such as biased sampling harder to detect and guard against, and for any large scale use of these algorithms in the wild, special care should be given to understand which points are being sampled the most and why. Furthermore, a common issue for many active learning approaches, this work included, is the possibility of model mismatch for any assumptions made in the theoretical analysis. While this work removes the need for an assumed prior over the true function $f$, other assumptions are still needed for the analysis, such as the function $f$ not varying in time. 
If these assumptions are violated, the claims herein need not be true. 

Any assumption made in this paper may reasonably be considered a limitation on the work depending on the application domain, though we hope that analytical assumptions may be easily modified to alter the algorithms to the practitioner's needs. This is true, for instance in the case of all confidence widths we use. Another limitation of this work is computational complexity. The RIPS procedure necessary to compute estimates of individual function values relies on a robust estimator for each $\bx \in\X$. In this work, we leverage the Catoni estimator. While this is efficient for individual $\bx$'s, as we observed in our experiments, if the set $\X$ is large, this can become cumbersome. Additionally, how to best optimize the experimental design objectives is an active area of research and must be done carefully. Finally, our algorithms both suffer potentially bad logarithmic terms in the per-round sample complexity, and this can affect the real-world performance of \texttt{MELK} and \texttt{MILK}. The technique of \cite{katz2020empirical} may be able to avoid this.

\section{Summary of Gaussian Processes Approaches for Level Set Estimation}
In Table~\ref{table:lse-imp}, we briefly summarize past algorithmic approaches to level set estimation. In general,  past methods center around the design of an \textit{acquisition function} which at each time $t$ tells the algorithm which point to go sample. By contrast, the algorithms in this paper both use experimental design to to select batches of samples to go gather at one time. 

\begin{table}[thb]
\begin{tabular}{lll}
\hline
\multicolumn{1}{|l|}{\textbf{Algorithm}} & \multicolumn{1}{l|}{\textbf{Acquisition Function}}                                                                                   & \multicolumn{1}{l|}{\textbf{Theoretical guarantee}}                                       \\ \hline
Straddle                                 & $\arg\max_i u_i(t) - \tau \wedge \tau - \ell_i(t) $                                                                                  & None, $u_i(t)$ and $\ell_i(t)$ are set as $1.96 \cdot \sigma_{t-1}$.                      \\
LSE                                      & $\arg\max_i u_i(t) - \tau \wedge \tau - \ell_i(t) $                                                                                  & $\eta$-approximate solution in $T \lesssim \frac{\gamma_t\log(n/\delta)}{\eta^2}$ \\
TruVar                                   & $\arg\min_{x_i} \sum_{x_j}\sigma^2_{t-1 | x_i} (x_j) $                                                                               & $\eta$-approximate solution in $T \lesssim \frac{\Gamma_t\log(n/\delta)}{\eta^2}$ \\
RMILE                                    & $\begin{aligned}[t]
\arg\max_{x_i}\{ \mathbb{E} \sum (\mathbb{P}_{GP|x_i}(f(x_j) > \tau) \\ - \mathbb{P}_{GP}(f(x_j) > \tau)), \sigma^2(x_i)\}
\end{aligned}$&  can be shown to be similar to A-optimality, no complexity guarantee                                          \\
MELK                                     & G-optimal design                                                                                                                     & Matching upper and lower bounds in the linear case.                                      
\end{tabular}
\caption{Algorithms and theoretical guarantees for explicit LSE \label{table:lse}}
\end{table}

 \begin{table}[]
\begin{tabular}{lll}
\hline
\multicolumn{1}{|l|}{\textbf{Algorithm}} & \multicolumn{1}{l|}{\textbf{Acquisition function}}              & \multicolumn{1}{l|}{\textbf{Theoretical guarantee}}                               \\ \hline
LSE-imp                                  & $\arg\max_i \sigma^2(x_i)$                                      & $\eta$-approximate solution in $T \lesssim \frac{\Gamma_t\log(n/\delta)}{\eta^2}$ \\
MILK                                  & XY optimal design over vectors $\phi(x) - (1-\epsilon)\phi(x')$ & Upper bounds and matching lower for certain cases.                                   
\end{tabular}
\caption{Acquisition functions and theoretical guarantees for implicit level set estimation \label{table:lse-imp}}
\end{table}

\section{Robust estimators for function means}\label{sec:rips}

In order for the algorithm to declare whether points $\bx$ belong in $G_\alpha$ (or $G_\epsilon$ in the sequel) or not, we require an estimator of the function values $f(\bx)$. As we have introduced structure by assuming that $f$ is well approximated by a function $\theta_\ast$ in the RKHS $\H$, we seek an estimator that leverages this structure to provide accurate estimates of many arms given samples of only a few. As a warmup, in the linear case where $\phi(\cdot)$ is the identity map, one could form the least squares or regularized least squares estimate of $\theta_\ast$ denoted $\htheta$ and estimate the mean of any point $\bx$ as $\htheta^T\bx$. To sample to estimate $\theta_\ast$, 
optimal design procedures first compute a design $\lambda \in \triangle_{\X}$. Then for a specified number of samples $N$, it is common to use an efficient rounding procedure such as \citep{allen2017near} to compute an allocation of the $N$ samples to the arms $\X$ such that $\bx_i$ gets roughly $\lambda_i\cdot N$ samples \citep{fiez2019sequential, jun2020improved}. Efficient rounding procedures require that $N = \Omega(d)$, and while this is a minor assumption in the case of a linear RKHS where $\phi(\bx) = \bx$, in general $\phi(\bx)$ may be infinite dimensional, and naive rounding is not possible. Instead of performing rounding given design $\lambda$, one may instead sample from $\lambda$ directly and use inverse propensity scoring (IPS) which avoids bad dimensional factors but can have high variance. 

In this work, we leverage the RIPS estimator from \citep{camilleri2021highdimensional} which combines IPS with robust mean estimation and regularization to control variance and is presented in Algorithm~\ref{alg:rips}. RIPS requires a robust mean estimator for its performance and theoretical guarantees. 
% The recent work \citep{camilleri2021highdimensional} shows that the objective of the design optimization problem in Algorithm~\ref{alg:rips} and its gradient can be evaluated solely from kernel evaluations. 
% Thus, first order optimization methods can be used to compute the optimal design in Algorithm~\ref{alg:rips}. Hence, while we write algorithmic operations in terms of $\phi(\bx)$ for clarity, at no point does the algorithm need to compute the vectors $\phi(\bx)$. 
In Theorem~\ref{thm:robust_estimator}, we state the guarantee of this estimator. 
% Further, note that as in \citep{camilleri2021highdimensional} the dual representation of the potentially infinite dimensional estimator $\widehat{\theta}$ would be used in the algorithms of the paper: with $\widehat{\theta} = \sum_{i=1}^{|\mc{X}|}\alpha_i \phi(x_i)$, the inner products $\langle \widehat{\theta}, v \rangle = \sum_{i=1}^{|\mc{X}|}\alpha_i \langle  \phi(x_i), v \rangle$ are computed using the kernel matrix of $\mc{X}$.

\begin{algorithm}[] % was [tb]
\caption{\texttt{RIPS}: 
\textbf{R}obust \textbf{IPS} estimator \label{alg:rips}}
\label{alg:large}
\begin{algorithmic}[1]
\Require{Finite sets $\X \subset \R^d$ and $\V \subset \H$,
feature map $\phi: \R^d \rightarrow \H$, number of samples $\tau$, regularization $\gamma>0$, robust mean estimator $\widehat{\mu}: \R^* \rightarrow \R$}
% \begin{align}\label{eqn:RKHS_exp_design_in_RIPS_alg}
$$
    \lambda := \arg\min_{\lambda \in \triangle_{\X}} \max_{v \in \V}\|v\|_{\left(A^{(\gamma)}(\lambda)\right)^{-1}}
$$
\State{Randomly draw $\widetilde{x}_1, \ldots, \widetilde{x}_\tau$ from $\X$ according to $\lambda^*$}
\State{Set $W^{(v)} = \widehat{\mu}( \{ v^\top A^{(\gamma)}(\lambda^*)^{-1} \phi(\widetilde{x}_t) \widetilde{y}_t \}_{t=1}^\tau )$}
\newline
\Return{$\widehat{\theta} := \arg\min_{\theta} \max_{v \in \V} \frac{ | \langle \theta, v \rangle - W^{(v)}| }{\|v\|_{\left(A^{(\gamma)}(\lambda)\right)^{-1}}}$}
% \newline
% \Return{$\widehat{\theta}$}
\end{algorithmic}
\end{algorithm}

We next state the complete theoretical guarantee of the RIPS estimator. 

\begin{theorem}[Theorem 1, \citep{camilleri2021highdimensional}]\label{thm:supp_robust_estimator}
Consider the model $y = \langle \phi(\bx), \theta^\ast\rangle_{\H} + \zeta_{\bx} + \eta$ for misspecification $|\zeta_{\bx}| \leq h$ where it is assumed that $|y| \leq B$, $\E[\eta]=0$, and $\E[\eta^2]\leq\sigma^2$. Fix any finite sets $\X \subset \R^d$ and $\V \subset \H$, feature map $\phi: \R^d \rightarrow \H$, number of samples $\tau$ and regularization $\gamma>0$. If the RIPS procedure of Algorithm~\ref{alg:rips} is run with $\tfrac{\delta}{|\V|}$-robust mean estimator  $\widehat{\mu}(\cdot)$ and if $\tau \geq c_1 \log(|\V|/\delta)$ then with probability at least $1-\delta$, we have 
\begin{align*}
    \max_{\bv \in \V} \frac{ | W^{(v)} - \langle \theta_\ast, \bv \rangle| }{\|v\|_{\left(A^{(\gamma)}(\lambda)\right)^{-1}} } \leq  &\sqrt{\gamma} \|\theta_\ast \| + h \\
    & + c \sqrt{\tfrac{(B^2 + \sigma^2)}{\tau}\log(2|\V|/\delta) }
\end{align*}
Moreover, $W^{(v)} = \widehat{\mu}( \{ \bv^\top A^{(\gamma)}(\lambda)^{-1} \phi(\bx_t) y_t \}_{t=1}^\tau )$ can be replaced by $\langle \widehat{\theta},\bv\rangle$ by multiplying the RHS by a factor of $2$.
\end{theorem}
For RIPS, we leverage Catoni's estimator 
\citep{lugosi2019mean} for which $c_1=2$ and $c=4$ suffice. 

\section{Proofs for Explicit Level Set Estimation}

\subsection{Lower Bound}
\begin{proof}[Proof of Theorem~\ref{thm:explicit_lower_bound}]
Recall that we have assumed that $h = 0$ and $\phi(\bx) = \bx$. 
We begin with a result of \cite{fiez2019sequential} that will be useful here. 
\begin{lemma}[\cite{fiez2019sequential}, Remark 2]\label{lem:proj_onto_diff}
The projection onto the closure of the set $\{\theta \in \R^d: \theta^T\bx < \alpha\}$ under the $\|\cdot\|_{A(\lambda)}$ norm 
is given by 
$$\theta_{x} := \theta - \frac{(\theta^T\bx - \alpha)A(\lambda)^{-1}\bx}{\|\bx\|_{A(\lambda)^{-1}}^2}.$$
\end{lemma}
By \cite{kaufmann2016complexity}, we have that the any $\delta$-PAC algorithm for all-$\alpha$ requires 
$$ \min_\lambda \frac{KL(1-\delta, \delta)}{\min_{\theta' \in \text{Alt}(\theta_\ast)}\|\theta' - \theta_\ast\|_{A(\lambda)}}$$
where $\text{Alt}(\theta_\ast)$ is the set of alternates such that $G_\alpha(\theta_\ast) \neq G_\alpha(\theta')$ for any $\theta' \in \text{Alt}(\theta_\ast)$. The set of alternates may be decomposed as 
\begin{align*}
    \Alt(\theta_\ast) = \left(\bigcup_{\bx\in G_{\alpha}(\theta_\ast)} \{\theta': \bx\not \in G_{\alpha}(\theta')\} \right) \cup \left(\bigcup_{\bx\in G_{\alpha}(\theta_\ast)^c} \{\theta': \bx\in G_{\alpha}(\theta')\} \right)
\end{align*}
Note that $\bx \in G_\alpha(\theta_\ast)\  \iff\  \theta_\ast^T\bx > \alpha$. Hence, the set of alternates for any $\bx \in G_\alpha(\theta_\ast)$ such that $\bx \in G_\alpha^c(\theta')$ for any $\theta' \in \Alt(\theta_\ast)$
is given by 
$$A_{\bx} : =\{\theta \in \R^d: \theta^T\bx < \alpha\}. $$
Next note that $\bx \in G_\alpha^c(\theta_\ast) \iff \theta_\ast^T\bx < \alpha.$ Hence, for any $\bx \in G_\alpha^c(\theta_\ast)$ the set of alternates such that $\bx \in G_\alpha(\theta')$ for any $\theta' \in \Alt(\theta_\ast)$ is given by 
$$A_{\bx} : = \{\theta \in \R^d: \theta^T\bx > \alpha\}. $$
Next, we discuss how to project onto $A_{\bx}$. As this set is open, to be precise, we should take a point in the interior and consider the limit for a sequence approaching the boundary. For brevity, we simply project onto the closure and consider the closures of the $A_{\bx}$ sets. 
Using the decomposition of $\Alt(\theta_\ast)$ we have that
\begin{align*}
    \min_{\theta' \in \text{Alt}(\theta_\ast)}\|\theta' - \theta_\ast\|_{A(\lambda)} = \min_{\bx} \min_{\theta' \in A_{\bx}}\|\theta' - \theta_\ast\|_{A(\lambda)} = \min_{\S \in \{G_\alpha,G_\alpha^c \}}\min_{\bx \in \S} \min_{\theta \in A_{\bx}}\|\theta' - \theta_\ast\|_{A(\lambda)}.
\end{align*}
For $\bx \in G_\alpha(\theta_\ast)$, using Lemma~\ref{lem:proj_onto_diff} and recalling the definition of the set $\theta_{\bx}$ therein,
$$\min_{\theta' \in A_{\bx}}\|\theta' - \theta_\ast\|_{A(\lambda)} = \min_{\theta' \in \{\theta \in \R^d: \theta^T\bx \leq \alpha\}}\|\theta' - \theta_\ast\|_{A(\lambda)} =  \|\theta_{\bx} - \theta_\ast\|_{A(\lambda)}.$$
The statement for points in $G_\alpha^c$ follows identically. 
Hence, 
\begin{align*}
    \min_{\theta' \in \text{Alt}(\theta_\ast)}\|\theta' - \theta_\ast\|_{A(\lambda)} = \min_{\bx} \|\theta_{\bx} - \theta_\ast\|_{A(\lambda)}
\end{align*}
Note that 
$$\|\theta_{\bx} - \theta_\ast\|_{A(\lambda)} = \frac{(\theta_\ast^T(\bx' - \bx) - \alpha)^2}{2\|\bx\|_{A(\lambda)^{-1}}^2} $$
by Theorem 2 of \cite{fiez2019sequential}. 
Hence, any $\delta$-PAC algorithm requires at least
\begin{align*}
    2\min_{\lambda}\max_{\bx}\frac{\|\bx\|_{A(\lambda)^{-1}}^2}{(\theta_\ast^T\bx - \alpha)^2} KL(1-\delta, \delta)
\end{align*}
samples in expectation. Noting that the binary entropy $KL(1-\delta, \delta) \geq \log(1/2.4\delta)$ completes the proof. 
\end{proof}

\subsection{Upper Bound}

Next, we restate Theorem~\ref{thm:MELK_complex} that bounds the complexity of \texttt{MELK}. 

\begin{theorem}
Fix $\delta >0$, threshold $\alpha > 0$, tolerance $\widetilde\beta$, and regularization $\gamma\geq 0$. Define $\Delta_{\min}(\alpha) := \min|\phi(\bx)^T\theta_\ast - \alpha|$. 
Define also 
\begin{align*}
    \bar\beta(\alpha) = \min \{ \beta>0 : 4(\sqrt{\gamma} \|\theta_* \| + h)(2+ \sqrt{f(\X,\left\{\phi(\bx)| \bx \in\X, |\phi(\bx)^T\theta_\ast - \alpha| \leq  \beta\right\};\gamma)}) \leq  \beta \}.
\end{align*}
With probability at least $1-\delta$,  
\texttt{MELK} returns a set $\widehat{\mc{R}} =  (\X\setminus \hB_{t})$ at time $T_{\delta}$
such that
\begin{align*}
    \{\bx\in \mc{X}: f(\bx) \geq \alpha + \bar{\beta}(\alpha)\} \subseteq \widehat{\mc{R}} \subseteq \{\bx\in \mc{X}: f(\bx) \geq \alpha - \widetilde{\beta} -  \bar{\beta}(\alpha)\}
\end{align*}
and for any $\alpha$, $\widetilde\beta$ such that $\max(\Delta_{\min}(\alpha), \widetilde\beta) \geq \bar\beta(\alpha)$
\begin{align*}
    T_\delta \leq 256(B^2 + \sigma^2)\min_{\lambda \in \tX}  \max_{\bx\in \X}  \frac{\|  \phi(\bx)\|_{(A(\lambda)+\gamma I)^{-1}}^2}{\max\{(\phi(\bx)^T\theta_\ast - \alpha)^2, \Tilde{\beta}^2\}}&\log\left(\frac{4|\X|^2\lceil  \log_2(4(\Delta_{\min}(\alpha)\vee\widetilde{\beta})^{-1}\rceil^2}{\delta}\right)\\
    & + 2\log(|\X|/\delta)\lceil \log_2(4(\Delta_{\min}(\alpha)\vee\widetilde{\beta})^{-1}\rceil
\end{align*}
\end{theorem}

Recall the definition of the set $G_\alpha := \{\bx \in \X: f(\bx) > \alpha\}$. 

\begin{lemma}\label{lem:high_prob_event_melk}
For any $\V \subset \X$ define $f(\X,\V;\gamma) = \min_{\lambda\in\triangle_{\X}} \max_{\bv \in \V} \| \bv \|_{( \sum_{x \in \X} \lambda_x \phi(x)\phi(x)^\top + \gamma I)^{-1}}^2$.\\
In each round $t$, define the event
$$\cE_t = \{|\bx^T(\htheta_t - \theta_\ast)| \leq 2^{-t} + \left(\sqrt{\gamma}\|\theta_\ast\| + h \right) \sqrt{f(\X, \A_t; \gamma)} \ \forall \ \bx \in \A_t\}$$\\
Holds $\P(\bigcup_{t=1}^\infty \cE_t^c) \leq \delta$.
\end{lemma}

\begin{proof}
Using Theorem~\ref{thm:robust_estimator}, for any $\bx \in \A_t$
we have that with probability at least $1-\delta_t/|\X|^2$
\begin{align*}
    |\bx^T(\htheta_t - \theta_\ast)| & \leq \|\bx\|_{\left(\sum_{\bx\in\X}\lambda_{\bx} \bx\bx^T+\gamma I\right)^{-1}}\left(\sqrt{\gamma}\|\theta_\ast\| + h + c \sqrt{\tfrac{(B^2 + \sigma^2)}{N_t}\log(2t^2|\X|^2/\delta)}\right) \\
    & \leq \sqrt{f(\X, \A_t; \gamma)}\left(\sqrt{\gamma}\|\theta_\ast\| + h + 2^{-t}/\sqrt{f(\X, \A_t; \gamma)}\right)  \\
    & \leq 2^{-t} + \left(\sqrt{\gamma}\|\theta_\ast\| + h \right) \sqrt{f(\X, \A_t; \gamma)}
\end{align*}
Since $|\A_t| \leq |\X|^2$, $\cE_t$ holds for all $\bx \in \A_t$ with probability $1-\delta_t$ via a union bound. Taking a second union bound over rounds, we have that 
$$\P\left(\bigcup_{t=1}^\infty \cE_t^c\right) \leq \sum_{t=1}^\infty \P(\cE_t^c) \leq \sum_{t=1}^\infty \delta_t  = \sum_{t=1}^\infty \frac{\delta}{2t^2} \leq \delta  $$
\end{proof}

Define 
\begin{align*}
    \bar t &= \max \{ t : (\sqrt{\gamma} \|\theta_* \|_2 + h)(2+ \sqrt{f(\X,\left\{\bx \in \X: |\bx^T\theta_\ast - \alpha| \leq 2^{-t + 2}\right\};\gamma)}) \leq 2^{-t} \} .
\end{align*}
As we will see in Lemmas~\ref{lem:melk_active_good} and \ref{lem:melk_active_bad},
$$\A_t \subset \left\{\bx \in \X: |\bx^T\theta_\ast - \alpha| \leq 2^{-t + 1}\right\}.$$
Thus for $t\leq \bar t$, holds on $\bigcap_t \cE_t$ that 
\begin{align*}
    \forall \bx \in \A_t \;, \; |\bx^T(\htheta_t - \theta_\ast)| \leq 2 \cdot 2^{-t}.
\end{align*}

\begin{lemma}\label{lem:melk_good_arms}
On $\bigcap_t \cE_t$, when $t\leq \bar t$ holds $\hG_t \subset G^{\phi}_\alpha :
= \{ \bx \ : \    \phi(\bx)^T \theta_\ast > \alpha\}$. 
\end{lemma}

\textbf{Remark:} If $h = 0$, $G^\phi_\alpha = G_\alpha$. 

\begin{proof}
\begin{align*}
    \bx \in \hG_t & \iff \exists t' \leq t  \ : \ \phi(\bx)^T\htheta_{t'} \geq \alpha + 2 \cdot 2^{-t'}\\
    & \iff \exists t' \leq t  \ : \ \phi(\bx)^T(\htheta_{t'} - \theta_\ast) + \phi(\bx)^T\theta_\ast  \geq \alpha + 2 \cdot 2^{-t'}\\
    & \stackrel{\bigcap_t\cE_{t}}{\implies}   \phi(\bx)^T \theta_\ast > \alpha\\
    & \iff \bx \in G^\phi_\alpha.
\end{align*}
\end{proof}

\begin{lemma}\label{lem:melk_bad_arms}
On $\bigcap_t \cE_t$, when $t\leq \bar t$ holds, $\hB_t \subset (G^\phi_\alpha)^c$.
\end{lemma}

\begin{proof}
\begin{align*}
    \bx \in \hB_t \iff & \exists t' \leq t \ : \ \phi(\bx)^T\htheta_t \leq \alpha - 2 \cdot 2^{-t'} \\
    &\exists t'\leq t \ : \  \phi(\bx)^T(\htheta_t - \theta_\ast) +  \phi(\bx)^T\theta_\ast \leq \alpha - 2 \cdot 2^{-t'} \\
    & \stackrel{\bigcap_t\cE_{t}}{\implies}  \phi(\bx)^T\theta_\ast < \alpha \\
    & \iff \bx \in (G^\phi_\alpha)^c.
\end{align*}
\end{proof}

\begin{lemma}\label{lem:melk_active_good}
On the event $\bigcap_t \cE_t$, when $t\leq \bar t$ holds,
\begin{align*}
    \A_t\cap  G^\phi_\alpha \subset &\left\{\bx \in G^\phi_\alpha \bigg| | \phi(\bx)^T\theta_\ast - \alpha| \leq 2^{-t + 2} \right\} =: \S_t^\text{Above}
\end{align*}
\end{lemma}

\begin{proof}
For any $\bx \in G_\alpha^\phi$ such that $\phi(\bx)^T\theta_\ast > \alpha + 2^{-t+1}$, if $t \geq \log(4 (\alpha -  \phi(\bx)^T\theta_\ast)^{-1})$ and $t \leq \bar t$, then 
\begin{align*}
    \phi(\bx)^T\htheta_t = \phi(\bx)^T(\htheta_t - \theta_\ast) + \phi(\bx)^T\theta_\ast > -2^{-t + 1} + \alpha + 2^{-t+1} = \alpha \geq \alpha
\end{align*}
which implies that $\bx \in \hG_t$. Noting that $\A_t \cap \hG_{t-1} = \emptyset$ completes the proof. 
\end{proof}

\begin{lemma}\label{lem:melk_active_bad}
On the event $\bigcap_t \cE_t$, when $t\leq \bar t$ holds,
\begin{align*}
    \A_t\cap  (G^\phi_\alpha)^c \subset &\left\{\bx \in (G^\phi_\alpha)^C \bigg| | \phi(\bx)^T\theta_\ast - \alpha| \leq 2^{-t + 2} \right\} =: \S_t^\text{Below}
\end{align*}
\end{lemma}

\begin{proof}
The proof follows identically as that of Lemma~\ref{lem:melk_active_good}
\end{proof}

\textbf{Remark:} Lemmas~\ref{lem:melk_active_good} and \ref{lem:melk_active_bad} jointly imply that $\A_t \subset \left\{\bx | \phi(\bx)^T\theta_\ast - \alpha| \leq 2^{-t + 2} \right\} =: \S_t$ for $t \leq \bar{t}$. Furthermore, $f(\X, \A_t, \gamma) \leq f(\X, \S_t, \gamma)$. 

\textbf{Remark:}\\
The algorithm stops on either of two conditions. On one hand if $t \geq \lceil \log_2(4/\widetilde{\beta})\rceil =: t_\beta$, then it has achieved precision $\widetilde{\beta}$ as desired and it terminates. 
Otherwise, it terminates if $\hG_t \cup \hB_t = \X$. This occurs when $\widetilde{\beta}$ is very small. Define $\Delta_{\min}(\alpha) := \min|\phi(\bx)^T\theta_\ast - \alpha|$. 
% Next we establish the conditions on the choice of $\alpha$ such that $G^\phi_\alpha = G_\alpha$. 
% To conclude on the correctness of the algorithm, we need to define a $\bar{\alpha}$ such that such that for all $\bar \alpha \leq \alpha$ holds $\bar t \geq  \log_2(4/\Delta_{\min}^{\text{Below}}(\alpha))$ and $\bar t \geq  \log_2(4/\Delta_{\min}^{\text{Above}}(\alpha))$.
Recall 
\begin{align*}
    \bar t &= \max \{ t : (\sqrt{\gamma} \|\theta_* \|_2 + h)(2+ \sqrt{f(\X,\left\{\bx \in \X: |\phi(\bx)^T\theta_\ast - \alpha| \leq 4\cdot 2^{-t}\right\};\gamma)}) \leq 2^{-t} \} \\
    &= \max \{ t : 4(\sqrt{\gamma} \|\theta_* \|_2 + h)(2+ \sqrt{f(\X,\left\{\bx \in \X: |\phi(\bx)^T\theta_\ast - \alpha| \leq 4\cdot 2^{-t}\right\};\gamma)}) \leq 4\cdot 2^{-t} \} \\
    &= -2 + \max \{ t : 4(\sqrt{\gamma} \|\theta_* \|_2 + h)(2+ \sqrt{f(\X,\left\{\bx \in \X: |\phi(\bx)^T\theta_\ast - \alpha| \leq  2^{-t}\right\};\gamma)}) \leq  2^{-t} \} \\
    &= -3 - \log_2 (\min \{ \beta>0 : 4(\sqrt{\gamma} \|\theta_* \|_2 + h)(2+ \sqrt{f(\X,\left\{\bx \in \X: |\phi(\bx)^T\theta_\ast - \alpha| \leq  \beta\right\};\gamma)}) \leq  \beta \}) .
\end{align*}
This defines 
\begin{align*}
    \bar \beta = \min \{ \beta>0 : 4(\sqrt{\gamma} \|\theta_* \|_2 + h)(2+ \sqrt{f(\X,\left\{\bx \in \X: |\phi(\bx)^T\theta_\ast - \alpha| \leq  \beta\right\};\gamma)}) \leq  \beta \}. 
\end{align*}
% and
% \begin{align*}
%     \bar{\alpha} = \min\{ \alpha > 0 : \min\{\Delta_{\min}^{\text{Below}}(\alpha), \Delta_{\min}^{\text{Above}}(\alpha)\} \geq 2^{-\bar t + 2 } \} = \min\{ \alpha > 0 : \min\{\Delta_{\min}^{\text{Below}}(\alpha), \Delta_{\min}^{\text{Above}}(\alpha)\} \geq \bar \alpha \}\\
% \end{align*}
% where we recall the quantities $\Delta_{\min}^{\text{Above}}(\alpha) = \min_{\bx'}\min_{\bx \in G^\phi_\alpha} \alpha - \theta_\ast^T(\phi(\bx') - \phi(\bx)) $ and $\Delta_{\min}^{\text{Below}}(\alpha) = \min_{\bx \in (G^\phi_\alpha)^c}\max_{\bx': (\phi(\bx') - \phi(\bx))^T\theta_\ast > \alpha}  \theta_\ast^T(\phi(\bx') - \phi(\bx)) - \alpha $.
Let $t_{\max}$ denote the random variable of the last round before the algorithm terminates. 
The following Lemmas give a guarantee on the set $\X \setminus \hB_t$ at termination. 

\begin{lemma}\label{lem:melk_contains}
On the event $\bigcap_{t=1}^\infty \cE_t$, \texttt{MELK} returns a set $(\X \setminus \hB_{t_{\max}})$ such that
$\{\bx: f(\bx) > \alpha + \bar\beta(\alpha)\} \subset (\X \setminus \hB_{t_{\max}})$. 
\end{lemma}

\begin{proof}
Take any $\bx$ such that $f(\bx) > \alpha + \bar\beta(\alpha)$ and recall that by assumption $|f(\bx) -\phi(\bx)^T\theta_\ast| \leq h$ for all $\bx \in \X$. 
% By definition of $\bar\beta(\alpha)$, we have that $\bar\beta(\alpha) > 4h$. 
We consider two cases. 
In the first case, assume that $t_{\max} \leq \bar{t}$. 
We claim that in this case $\not\exists t$ such that $\bx \in \hB_t$. We prove this by contradiction. Assume not. Then $\exists t$ such that 
\begin{align*}
    \htheta_t^T\phi(\bx) < \alpha - 2^{-t + 1} 
    & \iff \phi(\bx)^T(\htheta_t - \theta_\ast) + \phi(\bx)^T\theta_\ast < \alpha - 2^{-t + 1} \\
    & \stackrel{\cE_t}{\implies} -2^{-t} - \left(\sqrt{\gamma}\|\theta_\ast\| + h \right) \sqrt{f(\X, \A_t; \gamma)} +  \phi(\bx)^T\theta_\ast < \alpha - 2^{-t + 1} \\
    & \implies -2^{-t} - \left(\sqrt{\gamma}\|\theta_\ast\| + h \right) \sqrt{f(\X, \S_t; \gamma)} +  \phi(\bx)^T\theta_\ast < \alpha - 2^{-t + 1} \\
    & \implies -\left(\sqrt{\gamma}\|\theta_\ast\| + h \right) \sqrt{f(\X, \S_t; \gamma)} +  f(\bx) - h < \alpha - 2^{-t} \\
    & \implies f(\bx) < \alpha - 2^{-t} + h + \left(\sqrt{\gamma}\|\theta_\ast\| + h \right) \sqrt{f(\X, \S_t; \gamma)}.
\end{align*}
Recalling that we have assumed that $f(\bx) > \alpha + \bar\beta(\alpha)$. Hence, this implies that
\begin{align*}
    \bar\beta(\alpha) < - 2^{-t} + h + \left(\sqrt{\gamma}\|\theta_\ast\| + h \right) \sqrt{f(\X, \S_t; \gamma)}.
\end{align*}
Note that $\bar\beta(\alpha) > 0$. 
% Assume that $\left(\sqrt{\gamma}\|\theta_\ast\| + h \right) \sqrt{f(\X, \A_t; \gamma)} > 2^{-t}$. Otherwise the contradiction is immediate. 
As we have assumed that, $t \leq t_{\max} \leq \bar t$, we have that 
$2^{-t} \geq \left(\sqrt{\gamma}\|\theta_\ast\| + h \right) \sqrt{f(\X, \S_t; \gamma)}$ using the definition of $\bar t$. Hence, we have that 
\begin{align*}
    h > \bar\beta(\alpha) > 4h
\end{align*}
which is a contradiction where the final inequality follows from the definition of $\bar\beta(\alpha)$ for $\gamma > 0$. Hence, in this case we have shown that $\{\bx: f(\bx) > \alpha + \bar\beta(\alpha)\} \subset (\X \setminus \hB_{t_{\max}})$. 

In the second case, assume that $t_{\max} > \bar{t}$ and take $\bx$ such that $f(\bx) > \alpha + \bar\beta(\alpha)$. 
We claim that $\bx \in \hG_{\bar{t}}$ and hence $\bx \not\in \A_{t}$ for any $t > \bar{t}$ and thus is never added to $\hB_t$.
This occurs if 
\begin{align*}
    \phi(\bx)^T\htheta_{\bar{t}} > \alpha + 2^{-\bar{t} + 1}
    & \iff \phi(\bx)^T(\htheta_{\bar{t}} - \theta_\ast) + \phi(\bx)^T\theta_\ast > \alpha + 2^{-\bar{t} + 1} \\
    & \stackrel{\cE_{\bar{t}}}{\impliedby} -2^{-\bar{t}} - \left(\sqrt{\gamma}\|\theta_\ast\| + h \right) \sqrt{f(\X, \A_{\bar{t}}; \gamma)} + \phi(\bx)^T\theta_\ast \geq \alpha + 2^{-\bar{t} + 1} \\
    & \impliedby -2^{-\bar{t}} - \left(\sqrt{\gamma}\|\theta_\ast\| + h \right) \sqrt{f(\X, \S_{\bar{t}}; \gamma)} + \phi(\bx)^T\theta_\ast \geq \alpha + 2^{-\bar{t} + 1} \\
    & \iff \phi(\bx)^T\theta_\ast \geq \alpha + 2^{-\bar{t} + 1} + 2^{-\bar{t}} + \left(\sqrt{\gamma}\|\theta_\ast\| + h \right) \sqrt{f(\X, \S_{\bar{t}}; \gamma)}
\end{align*}
Recall that $f(\bx) > \alpha + \bar\beta(\alpha)$.
% This implies that $\theta_\ast^T\phi(\bx) > \alpha + h$.
Furthermore, we have by the definition of $\bar{t}$ that
\begin{align*}
    2^{-\bar{t}} \geq \left(\sqrt{\gamma}\|\theta_\ast\| + h \right) \sqrt{f(\X, \S_{\bar{t}}; \gamma)}.
\end{align*}
Hence, the above is implied by $\bar\beta(\alpha) -h \geq 4\cdot 2^{-t} = 0.5 \bar\beta(\alpha)$ where the final equality holds by definition of $\bar\beta(\alpha)$. Noting that $\bar\beta(\alpha) > 4h$ proves this claim. 
In summary, we have shown that for any $\bx$ such that $f(\bx) > \alpha + \bar\beta(\alpha)$, if $t_{\max}\leq \bar{t}$, then $\bx$ is never added to $\hB_t$ and hence is contained in the set $\X\setminus \hB_t$ at termination, and if otherwise that $t_{\max}> \bar{t}$, then $\bx$ is added to the set $\hG_t$ before round $\bar{t} + 1$ and hence is removed from the active set and never added to $\hB_t$. Applying this argument to any $\bx$ such that $f(\bx) > \alpha + \bar\beta(\alpha)$ completes the proof.
\end{proof}

\begin{lemma}\label{lem:melk_is_contained}
On the event $\bigcap_{t=1}^\infty \cE_t$, \texttt{MELK} returns a set $(\X \setminus \hB_{t_{\max}})$ such that
$(\X \setminus \hB_{t_{\max}}) \subset \{\bx: f(\bx) > \alpha - \bar{\beta}(\alpha) - \widetilde{\beta}\}$.
\end{lemma}

\begin{proof}
Take any $\bx$ such that $f(\bx) < \alpha - \bar\beta(\alpha) - \widetilde{\beta}$. We claim that there exists a $t \leq t_{\max}$ such that $\bx$ is added to $\hB_t$ which implies that $\bx \not\in (\X \setminus \hB_{t_{\max}})$. Suppose for contradiction that this is not the case. Then for all $t \leq t_{\max}$, 
\begin{align*}
    \htheta_t^T\phi(\bx) > \alpha - 2^{-t + 1} 
    & \iff \phi(\bx)^T(\htheta_t - \theta_\ast) + \phi(\bx)^T\theta_\ast > \alpha - 2^{-t + 1} \\
    & \stackrel{\cE_t}{\implies} 2^{-t} + \left(\sqrt{\gamma}\|\theta_\ast\| + h \right) \sqrt{f(\X, \A_t; \gamma)} +  \phi(\bx)^T\theta_\ast > \alpha - 2^{-t + 1} \\
    & \implies 2^{-t} + \left(\sqrt{\gamma}\|\theta_\ast\| + h \right) \sqrt{f(\X, \S_t; \gamma)} +  \phi(\bx)^T\theta_\ast > \alpha - 2^{-t + 1}\\
    & \implies \left(\sqrt{\gamma}\|\theta_\ast\| + h \right) \sqrt{f(\X, \S_t; \gamma)} +  f(\bx) + h > \alpha - 2^{-t+1} - 2^{-t} \\
    & \implies f(\bx) > \alpha - 2^{-t+1} - 2^{-t} - h - \left(\sqrt{\gamma}\|\theta_\ast\| + h \right) \sqrt{f(\X, \S_t; \gamma)}.
\end{align*}
Plugging in $f(\bx) < \alpha - \bar\beta(\alpha) - \widetilde{\beta}$, the above implies
\begin{align}\label{eq:melk_lem_event}
     \bar\beta(\alpha) + \widetilde{\beta} <  2^{-t+1} + 2^{-t} + h + \left(\sqrt{\gamma}\|\theta_\ast\| + h \right) \sqrt{f(\X, \S_t; \gamma)}
\end{align}
Next, recall that \texttt{MELK} terminates either on the condition that $t = \lceil\log_2(4/\widetilde{\beta})\rceil$ or that $\hG_t \cup \hB_t = \X$. Using this, we brake our analysis into cases.

Case 1: $t_{\max} = \lceil\log_2(4/\widetilde{\beta})\rceil \leq \bar{t}$.

In this case, \texttt{MELK} stops due to the $\widetilde{\beta}$ tolerance in a round before $\bar{t}$. For $t \leq \bar{t}$, we have that $2^{-t} \geq  + \left(\sqrt{\gamma}\|\theta_\ast\| + h \right) \sqrt{f(\X, \S_t; \gamma)}$. Hence, the above implies that 
\begin{align*}
    \bar\beta(\alpha) + \widetilde{\beta} <  2^{-t+2} + h.
\end{align*}
As we have assumed this condition for all $t \leq t_{\max}$, we may plug in $t_{\max}$ which implies 
\begin{align*}
    \bar\beta(\alpha) + \widetilde{\beta} < \widetilde{\beta} + h.
\end{align*}
As $ \bar\beta(\alpha) > h$, this is a contradiction. Hence there must exist a $t$ such that $\bx \in \hB_t$. 

Case 2: $ t_{\max} \leq \bar{t} <  \lceil\log_2(4/\widetilde{\beta})\rceil$. 

In this case, \texttt{MELK} terminates before round $t = \lceil\log_2(4/\widetilde{\beta})\rceil$. Hence, it does so on the condition that $\hG_t \cup \hB_t = \X$. Note that for $f(\bx) < \alpha - \bar\beta(\alpha) - \widetilde{\beta}$, we have that $\bx \in (G_\alpha^\phi)^c$ since $\bar{\beta}(\alpha) > h$ and $\widetilde{\beta} \geq 0$. If we terminate before round $\bar{t}$, we have by Lemma~\ref{lem:melk_bad_arms} that $(G_\alpha^\phi)^c \subset \hB_t$ which implies that $\bx \in \hB_{t_{\max}}$. This contradicts the assumption that $\not\exists t: \bx \in \hB_t$. 

Case 3: $\bar{t} < t_{\max}$.

In this case, \texttt{MELK} terminates at a round after $\bar{t}$. In this setting, we argue that $\bx \in \hB_{\bar{t}}$. Recall that for any $t \leq \bar{t}$, \eqref{eq:melk_lem_event} simplifies to
\begin{align*}
    \bar{\beta}(\alpha) + \widetilde{\beta} < 2^{-t+2} + h
\end{align*}
Plugging in $\bar{t}$, and noting that $2^{-\bar{t} + 2} = \frac{1}{2}\bar{\beta}(\alpha)$, the above implies
\begin{align*}
    \bar{\beta}(\alpha) + \widetilde{\beta} < \frac{1}{2}\bar{\beta}(\alpha) + h. 
\end{align*}
Noting that $\bar{\beta}(\alpha) > 4h$, shows that the above is a contradiction. Hence, there exists a $t \leq \bar{t}$ such that $\bx \in \hB_t$. 

Therefore, in all cases we have shown that for any $\bx$ such that $f(\bx) < \alpha - \bar{\beta}(\alpha) - \widetilde{\beta}$, $\bx \in \hB_t$. Therefore, for the returned set $\X \setminus \hB_{t_{\max}}$, we have that \begin{align*}
    (\X \setminus \hB_{t_{\max}}) \subset \{\bx: f(\bx) > \alpha - \bar{\beta}(\alpha) - \widetilde{\beta}\}. 
\end{align*}
\end{proof}

\begin{proof}[Proof of Theorem~\ref{thm:MELK_complex}]
Throughout, assume the high probability event $\bigcap_T \cE_t$. By Lemmas~\ref{lem:melk_contains} and \ref{lem:melk_is_contained} in conjunction with the high probability event $\bigcap \cE_t$ we have correctness. It remains to control the sample complexity of \texttt{MELK}.
Recall that we have assumed that $\max(\Delta_{\min}(\alpha), \widetilde{\beta}) \geq \bar\beta(\alpha)$. This implies that $\min\{\lceil\log_2(4/\Delta_{\min}(\alpha)\rceil, \lceil\log_2(4/\widetilde{\beta})\rceil\} \leq \bar{t}$. Applying Lemmas~\ref{lem:melk_active_good} and \ref{lem:melk_active_bad}, we have that $t_{\max} \leq \min\{\lceil\log_2(4/\Delta_{\min}(\alpha)\rceil, \lceil\log_2(4/\widetilde{\beta})\rceil\} \leq \bar{t}$ and that $\A_t \subseteq \S_t$ for all rounds $t$. 
% Combining Lemmas~\ref{lem:melk_active_good} and \ref{lem:melk_active_bad}, we have that if $\Delta_{\min}(\alpha) \geq \bar\beta$ and $\widetilde{\beta} \leq \bar{\beta}$, then $\hG_t \cup \hB_t = \X$ for $t \geq \lceil \log_2(4/\Delta_{\min}(\alpha))\rceil$ implying termination. 
% \texttt{MELK} terminates in a round $t_{\max} \leq \min \{\lceil \log_2(4/\Delta_{\min}(\alpha))\rceil, \lceil \log_2(4/\widetilde{\beta})\rceil\} =  \lceil \log_2(4(\Delta_{\min}(\alpha)\vee\widetilde{\beta})^{-1}\rceil$.
Now we proceed by bounding the total number of samples drawn. 
% On the additional assumption made that $\min(\Delta_{\min}(\alpha), \widetilde{\beta}) \geq \bar\beta(\alpha)$ we have that $\min\{\lceil\log_2(4/\Delta_{\min}(\alpha)\rceil, \lceil\log_2(4/\widetilde{\beta})\rceil\} \leq \bar{t}$ hence, we may

\begin{align*}
    \tau & \leq \sum_{t=1}^{t_{\max}} N_t \\
    %%%%%%%%%%%%%%%%%%%%%
    &\leq \sum_{t=1}^{\min \{t \geq \lceil \log_2(4/\Delta_{\min}(\alpha))\rceil, t \geq \lceil \log_2(4/\widetilde{\beta})\rceil\}} N_t\\
    %%%%%%%%%%%%%%%%%%%%%%
    & = \sum_{t=1}^{\lceil \log_2(4(\Delta_{\min}(\alpha)\vee\widetilde{\beta})^{-1}\rceil} \max\left\{c_1\log(|\X|/\delta), c^2 2^{2t} f(\A_t;\gamma) (B^2 + \sigma^2) \log(2t^2|\X|^2/\delta)\right\}\\
    %%%%%%%%%%%%%%%%%%%%%%
    & \leq c_1\log(|\X|/\delta)\lceil \log_2(4(\Delta_{\min}(\alpha)\vee\widetilde{\beta})^{-1}\rceil + c^2(B^2 + \sigma^2)\sum_{t=1}^{\lceil \log_2(4(\Delta_{\min}(\alpha)\vee\widetilde{\beta})^{-1}\rceil} 2^{2t} f(\A_t;\gamma)\cdot\log(2t^2|\X|^2/\delta) \\
    %%%%%%%%%%%%%%%%%%%%%%
    & = c_1\log(|\X|/\delta)\lceil \log_2(4(\Delta_{\min}(\alpha)\vee\widetilde{\beta})^{-1}\rceil + \\
    & \hspace{1cm}c^2(B^2 + \sigma^2)\sum_{t=1}^{\lceil \log_2(4(\Delta_{\min}(\alpha)\vee\widetilde{\beta})^{-1}\rceil} 2^{2t} \min_{\lambda \in \tX}\max_{\bx \in \A_t} \|\bx\|_{\left(\sum_{\bx \in \X}\lambda_t(\bx)\phi(\bx)\phi(\bx)^T+\gamma I\right)^{-1}}^2\cdot\log(2t^2|\X|^2/\delta) \\
    %%%%%%%%%%%%%%%%%%%%%%
    & \leq c_1\log(|\X|/\delta)\lceil \log_2(4(\Delta_{\min}(\alpha)\vee\widetilde{\beta})^{-1}\rceil + \\
    & \hspace{1.5cm}c^2(B^2 + \sigma^2)\log\left(\frac{4|\X|^2\lceil \log_2(4(\Delta_{\min}(\alpha)\vee\widetilde{\beta})^{-1}\rceil^2}{\delta}\right)\\
    & \hspace{1cm}\cdot \sum_{t=1}^{\lceil \log_2(4(\Delta_{\min}(\alpha)\vee\widetilde{\beta})^{-1}\rceil} 2^{2t} \min_{\lambda \in \tX}\max_{\bx \in \A_t} \|\bx\|_{\left(\sum_{\bx \in \X}\lambda_t(\bx)\phi(\bx)\phi(\bx)^T+\gamma I\right)^{-1}}^2\\
    %%%%%%%%%%%%%%%%%%%%%%
    & \stackrel{\A_t\subset\S_t}{\leq} c_1\log(|\X|/\delta)\lceil \log_2(4(\Delta_{\min}(\alpha)\vee\widetilde{\beta})^{-1}\rceil + \\
    & \hspace{1.5cm}c^2(B^2 + \sigma^2)\log\left(\frac{4|\X|^2\lceil \log_2(4(\Delta_{\min}(\alpha)\vee\widetilde{\beta})^{-1}\rceil^2}{\delta}\right)\\
    & \hspace{1cm}\cdot \sum_{t=1}^{\lceil \log_2(4(\Delta_{\min}(\alpha)\vee\widetilde{\beta})^{-1}\rceil} 2^{2t} \min_{\lambda \in \tX}\max_{\bx \in \S_t} \|\bx\|_{\left(\sum_{\bx \in \X}\lambda_t(\bx)\phi(\bx)\phi(\bx)^T+\gamma I\right)^{-1}}^2.
\end{align*}
It remains to control the final summation. To do so, note that 
\begin{align*}
    &\frac{1}{\lceil \log_2(4(\Delta_{\min}(\alpha)\vee\widetilde{\beta})^{-1}\rceil} \sum_{t=1}^{\lceil \log_2(4(\Delta_{\min}(\alpha)\vee\widetilde{\beta})^{-1}\rceil} 2^{2t} \min_{\lambda \in \tX}\max_{\bx \in \S_t} \|\bx\|_{\left(\sum_{\bx \in \X}\lambda_t(\bx)\phi(\bx)\phi(\bx)^T+\gamma I\right)^{-1}}^2 \\
    % %%%%%%%%%%%%%%%%%%%%%%%%
    & \leq \max_{t \leq \lceil \log_2(4(\Delta_{\min}(\alpha)\vee\widetilde{\beta})^{-1}\rceil}\min_{\lambda \in \tX} 2^{2t} \min_{\lambda \in \tX}\max_{\bx \in \S_t} \|\bx\|_{\left(\sum_{\bx \in \X}\lambda_t(\bx)\phi(\bx)\phi(\bx)^T+\gamma I\right)^{-1}}^2\\
    %%%%%%%%%%%%%%%%%%%%%%%%%
    & \leq \min_{\lambda \in \tX}\max_{t \leq \lceil \log_2(4(\Delta_{\min}(\alpha)\vee\widetilde{\beta})^{-1}\rceil} 2^{2t} \min_{\lambda \in \tX}\max_{\bx \in \S_t} \|\bx\|_{\left(\sum_{\bx \in \X}\lambda_t(\bx)\phi(\bx)\phi(\bx)^T+\gamma I\right)^{-1}}^2\\
    % %%%%%%%%%%%%%%%%%%%%%%%%%
    % & =   \min_{\lambda \in \tX} \max_{t \leq \lceil\log_2(2\Delta_{\min}^{-1}(\alpha))\rceil} \max\left\{\max_{\phi(\bx') - \phi(\bx) \in \S_t^\text{Above}} 2^{2t}\|\phi(\bx') - \phi(\bx)\|_{(H(\lambda)+\gamma I)^{-1}}^2, \max_{\phi(\bx') - \phi(\bx) \in \S_t^\text{Below}} 2^{2t}\|\phi(\bx') - \phi(\bx)\|_{(H(\lambda)+\gamma I)^{-1}}^2\right\}\\
    % %%%%%%%%%%%%%%%%%%%%%%%%%
    % & \stackrel{\text{Lemmas~\ref{lem:melk_active_good}, \ref{lem:melk_active_bad}}}{\leq}     16\min_{\lambda \in \tX} \max_{t \leq \lceil\log_2(2\Delta_{\min}^{-1}(\alpha))\rceil} \max\left\{\max_{\bx: \bx \in G^\phi_\alpha \land \exists \bx': \phi(\bx') - \phi(\bx) \in \S_t^\text{Above}}\max_{\bx': \phi(\bx') - \phi(\bx) \in \S_t^\text{Above}} \frac{\|\phi(\bx') - \phi(\bx)\|_{(H(\lambda)+\gamma I)^{-1}}^2}{(\alpha - (\phi(\bx') - \phi(\bx))^T\theta_\ast)^2},\right. \\
    % &\hspace{5cm}\left. \max_{\bx: \bx \in (G^\phi_\alpha)^c \land \exists \bx': \phi(\bx') - \phi(\bx) \in \S_t^\text{Below}}\max_{\bx': \phi(\bx') - \phi(\bx) \in \S_t^\text{Below}} \frac{\|\phi(\bx') - \phi(\bx)\|_{(H(\lambda)+\gamma I)^{-1}}^2}{((\phi(\bx') - \phi(\bx))^T\theta_\ast - \alpha)^2}\right\}\\
    % %%%%%%%%%%%%%%%%%%%%%%%%%
    & \leq 16\min_{\lambda \in \tX} \max_{\bx}\frac{\|\phi(\bx)\|_{(\sum_{\bx \in \X}\lambda_t(\bx)\phi(\bx)\phi(\bx)^T+\gamma I)^{-1}}^2}{\max\{(\phi(\bx)^T\theta_\ast - \alpha)^2, \widetilde{\beta}^2\}}
\end{align*}
Plugging this along with $c=4$ and $c_1=2$ for Theorem \ref{thm:supp_robust_estimator}
from RIPS with the Catoni estimator 
in completes the proof.
\end{proof}

\section{Proofs for Implicit Level Set Estimation}

\subsection{Lower Bounds}
\begin{proof}[Proof of Theorem~\ref{thm:mult_lower_bound}]
Recall that in this setting, $h = 0$ and $\phi(\bx) = \bx$. 
By \cite{kaufmann2016complexity}, we have that the any $\delta$-PAC algorithm for all-$\epsilon$ requires 
$$ \min_\lambda \frac{KL(1-\delta, \delta)}{\min_{\theta' \in \text{Alt}(\theta_\ast)}\|\theta' - \theta_\ast\|_{A(\lambda)}}$$
where $\text{Alt}(\theta_\ast)$ is the set of alternates such that $G_\epsilon(\theta_\ast) \neq G_\epsilon(\theta')$ for any $\theta' \in \text{Alt}(\theta_\ast)$. The set of alternates may be decomposed as 
\begin{align*}
    \Alt(\theta_\ast) = \left(\bigcup_{\bx\in G_{\epsilon}(\theta_\ast)} \{\theta': \bx\not \in G_{\epsilon}(\theta')\} \right) \cup \left(\bigcup_{\bx\in G_{\epsilon}(\theta_\ast)^c} \{\theta': \bx\in G_{\epsilon}(\theta')\} \right)
\end{align*}
By Lemma~\ref{lem:mult_alle_sets}, $\bx \in G_\epsilon \iff \forall \bx': \theta_\ast^T(\bx - (1-\epsilon)\bx') > 0$. Hence, the set of alternates for any $\bx \in G_\epsilon(\theta_\ast)$ such that $\bx \in G_\epsilon^c(\theta')$ for any $\theta' \in \Alt(\theta_\ast)$
is given by 
$$A_{\bx} : =\bigcup_{\bx' \in \X} \{\theta \in \R^d: \theta^T(\bx - (1-\epsilon)\bx') < 0\}. $$
Furthermore, by Lemma~\ref{lem:mult_alle_sets} $\bx \in G_\epsilon^c \iff \exists \bx': \theta_\ast^T(\bx - (1-\epsilon)\bx') < 0$. Hence, for any $\bx \in G_\epsilon^c(\theta_\ast)$ the set of alternates such that $\bx \in G_\epsilon(\theta')$ for any $\theta' \in \Alt(\theta_\ast)$ is given by 
$$A_{\bx} := \bigcap_{\bx' \in \X} \{\theta \in \R^d: \theta^T(\bx - (1-\epsilon)\bx') > 0\}. $$
Next, we discuss how to project onto $A_{\bx}$. As this set is open, to be precise, we should take a point in the interior and consider the limit for a sequence approaching the boundary. For brevity, we simply project onto the closure and consider the closures of the $A_{\bx}$ sets. 
Using the decomposition of $\Alt(\theta_\ast)$ we have that
\begin{align*}
    \min_{\theta' \in \text{Alt}(\theta_\ast)}\|\theta' - \theta_\ast\|_{A(\lambda)} = \min_{\bx} \min_{\theta' \in A_{\bx}}\|\theta' - \theta_\ast\|_{A(\lambda)} = \min_{\S \in \{G_\epsilon,G_\epsilon^c \}}\min_{\bx \in \S} \min_{\theta \in A_{\bx}}\|\theta' - \theta_\ast\|_{A(\lambda)}.
\end{align*}
Reminiscent of Lemma~\ref{lem:proj_onto_diff}, we define
\begin{align*}
    \theta_{\bx,\bx'}^\epsilon(\lambda) := \theta_\ast - [\theta_\ast^T(\bx - (1-\epsilon)\bx')]\frac{\A(\lambda)^{-1}(\bx - (1-\epsilon)\bx'))}{\|\bx - (1-\epsilon)\bx'\|_{\A(\lambda)^{-1}}^2}.
\end{align*}
For $\bx \in G_\epsilon(\theta_\ast)$, using Lemma~\ref{lem:proj_onto_diff},
$$\min_{\theta' \in A_{\bx}}\|\theta' - \theta_\ast\|_{A(\lambda)} = \min_{\theta' \in \bigcup_{\bx' \in \X} \{\theta \in \R^d: \theta^T(\bx - (1-\epsilon)\bx') < 0\}}\|\theta' - \theta_\ast\|_{A(\lambda)} =  \min_{\bx'}\|\theta_{\bx,\bx'}^\epsilon(\lambda) - \theta_\ast\|_{A(\lambda)} $$
where the latter equality follows since projecting onto a union of hyperplanes is achieved by the projection onto the closest constituent. 

For $\bx \in G_\epsilon^c(\theta_\ast)$ note that $A_{\bx}$ is an intersection of half spaces $\{\theta \in \R^d: \theta^T(\bx - (1-\epsilon)\bx) > 0\}$ for $\bx' \in \X$. As it is not in general possible to give a closed form expression for projection onto an intersection of convex sets. However, we may at a (possibly very loose) minimum note that the projection onto the union of the hyperplanes is at least as far as the projection onto the furthest hyperplane. 
Therefore, for any $\bx \in G_\epsilon(\theta_\ast)^c$, 
$$\min_{\theta' \in A_{\bx}}\|\theta' - \theta_\ast\|_{A(\lambda)} = \min_{\theta' \in \bigcap_{\bx' \in \X} \{\theta \in \R^d: \theta^T(\bx - (1-\epsilon)\bx') > 0\}}\|\theta' - \theta_\ast\|_{A(\lambda)} \leq  \max_{\bx'}\|\theta_{\bx,\bx'}^\epsilon(\lambda) - \theta_\ast\|_{A(\lambda)} $$
Hence
we have that 
\begin{align*}
    \min_{\theta' \in \text{Alt}(\theta_\ast)}\|\theta' - \theta_\ast\|_{A(\lambda)} \leq \min\left\{\min_{\bx \in G_\epsilon}\min_{\bx'}\|\theta_{\bx,\bx'}^\epsilon(\lambda)- \theta_\ast\|_{A(\lambda)}, \min_{\bx \in G_\epsilon^c}\max_{\bx'}\|\theta_{\bx,\bx'}^\epsilon(\lambda) - \theta_\ast\|_{A(\lambda)}\right\}.
\end{align*}
Note that 
$$\|\theta_{\bx,\bx'}^\epsilon(\lambda) - \theta_\ast\|_{A(\lambda)} = 2\frac{(\theta_\ast^T(\bx - (1-\epsilon)\bx'))^2}{\|\bx -(1-\epsilon)\bx'\|_{A(\lambda)^{-1}}^2} $$
by Theorem 2 of \cite{fiez2019sequential}. 
Hence, any $\delta$-PAC algorithm requires
\begin{align*}
    2\min_{\lambda}\max\left\{\max_{\bx \in G_\epsilon} \max_{\bx'}\frac{\|\bx -(1-\epsilon)\bx'\|_{A(\lambda)^{-1}}^2}{(\theta_\ast^T(\bx - (1-\epsilon)\bx'))^2}, \max_{\bx \in G_\epsilon^c} \min_{\bx'}\frac{\|\bx -(1-\epsilon)\bx'\|_{A(\lambda)^{-1}}^2}{(\theta_\ast^T(\bx - (1-\epsilon)\bx'))^2}\right\} KL(1-\delta, \delta)
\end{align*}
samples in expectation. 
Noting that $KL(1-\delta, \delta) \geq \log(1/2.4\delta)$ completes the proof. 
\end{proof}

\subsection{Comparison to the lower bound of 
\cite{mason2020finding}}
Here, we compare the sample complexity given in Theorem~\ref{thm:MILK_complex} to the result of Mason et al., \cite{mason2020finding} studying the problem of finding all $\epsilon$-good arms in multi-armed bandits. Our setting captures this problem in the special case that $\phi(\bx) = \bx$, $\bx_i = e_i \in \R^{|\X|}$, $h = 0$, and $\widetilde{\beta} = 0$. Additionally, take $\gamma \rightarrow 0$. 
For consistency with the notation of \cite{mason2020finding}, let $\mu_i = f(\bx_i)$ and $|\X| = n$. 
In this setting, the problem of implicit level set estimation reduces to identifying the set $\{i: \mu_i > (1-\epsilon)\mu_1\}$ where we assume without loss of generality that the means are sorted in descending order such that $\mu_1\geq \mu_2 \geq \cdots \geq \mu_{n}$. 

\begin{lemma}\label{lem:east_compare}
The term $H^\texttt{MILK}(\theta_\ast) = cH_{(ST)^2}$ for a constant $c$ where $H_{(ST)^2}$ is the complexity parameter of the $(ST)^2$ algorithm from \cite{mason2020finding}. 
\end{lemma}

In particular, \cite{mason2020finding} show in Theorem 4.1 that a complexity of $H_{(ST)^2}$ is optimal up to logarithmic factors for any fixed $\delta$ via a moderate confidence bound. This exceeds the lower bound given in Theorem~\ref{thm:mult_lower_bound} specialized to this case. 
In particular, this highlights that the lower bound given in Theorem~\ref{thm:mult_lower_bound} is not achievable except possibly as $\delta \rightarrow 0$. Instead, we show that \texttt{MILK} achieves the optimal non-asymptotic sample complexity for finding all $\epsilon$-good arms. 

\begin{proof}[Proof of Lemma~\ref{lem:east_compare}]
First, we recall some notation from \cite{mason2020finding} necessary for this lemma. Let $\Tilde{\alpha}_\epsilon = \min_{i \in G_\epsilon}\mu_i - (1-\epsilon)\mu_1$ and let $\Tilde{\beta}_\epsilon = \min_{i \in G_\epsilon^c}(1-\epsilon)\mu_1 - \mu_i$. For brevity, we let $k = \arg\min_{i\in G_\epsilon}\mu_i$ and $k+1 = \arg\max_{i\in G_\epsilon^c}\mu_i$ where we take $n > k$. If this condition does not hold the same argument as below suffices ignoring all terms in $G_\epsilon^c$. Hence we have that $\frac{\mu_k}{1-\epsilon} = \mu_1 + \frac{\Tilde{\alpha}_\epsilon}{1-\epsilon}$ and $\frac{\mu_{k+1}}{1-\epsilon} = \mu_1 - \frac{\Tilde{\beta}_\epsilon}{1-\epsilon}$. Furthermore, \cite{mason2020finding} restrict to the case of $\epsilon \in [1/2, 1)$. 

% In the case that $\bx_i = \bx_i = e_i$ this problem reduces to standard multi-armed bandits.
We begin by lower bounding the complexity parameter $H^\texttt{MILK}(\theta_\ast)$. We analyze the two terms given in Theorem~\ref{thm:MILK_complex}, $H^\texttt{MILK1}$ and $H^\texttt{MILK2}$ separately. 
$H^\texttt{MILK1}$ reduces to
\begin{align*}
    \max_{e_i \in G_\epsilon}\max_{e_j} \frac{\|e_j - e_i\|_{A(\lambda)^{-1}}^2}{(\mu_i - (1-\epsilon)\mu_j)^2} &= \max_{e_i \in G_\epsilon}\max_{e_j} \frac{1/\lambda_i + 1/\lambda_j}{(\mu_i - (1-\epsilon)\mu_j)^2} \\
    %%%%%%%%%%%%%%%%%%%
    & \geq  \max\left\{\max_{e_i \in G_\epsilon} \frac{1/\lambda_i }{(\mu_i - (1-\epsilon)\mu_1)^2}, \max_{e_j} \frac{1/\lambda_j}{(\frac{\mu_k}{1-\epsilon} - \mu_j)^2} \right\}\\
    %%%%%%%%%%%%%%%%%%%%%
    & = \max\left\{\max_{e_i \in G_\epsilon} \frac{1/\lambda_i }{(\mu_1 - \mu_i -\epsilon)^2}, \max_{e_j} \frac{1/\lambda_j}{(\mu_1 +\frac{\Tilde{\alpha}_\epsilon}{1-\epsilon} - \mu_j)^2} \right\}
\end{align*}
where the final step follows by the definition of $\Tilde{\alpha}_\epsilon$. The penultimate step follows by first maximizing over $i \in G_\epsilon$ which introduces a factor of $\mu_k$. Then we may multiply the denominator by $(1-\epsilon)^2 / (1-\epsilon)^2$ and upper bound $(1-\epsilon)^2 \leq 0.25 < 1$ since $\epsilon \geq 1/2$ to achieve the result.

$H^\texttt{MILK2}$ reduces to 
\begin{align*}
    \max_{e_i \in G_\epsilon^c}\max_{e_j} \frac{\|e_j - e_i\|_{A(\lambda)^{-1}}^2}{((1-\epsilon)\mu_1 - \mu_i)^2} &= \max_{e_i \in G_\epsilon^c}\max_{e_j} \frac{1/\lambda_i + 1/\lambda_j}{((1-\epsilon)\mu_1 - \mu_i)^2} \\
    %%%%%%%%%%%%%%%%%%%
    & \geq \max\left\{\max_{e_i \in G_\epsilon^c} \frac{1/\lambda_i }{((1-\epsilon)\mu_1 - \mu_i)^2}, \max_{e_i \in G_\epsilon^c}\max_{e_j} \frac{1/\lambda_j}{((1-\epsilon)\mu_1 - \mu_i)^2} \right\}\\
    %%%%%%%%%%%%%%%%%%%
    & \geq \max\left\{\max_{e_i \in G_\epsilon^c} \frac{1/\lambda_i }{((1-\epsilon)\mu_1 - \mu_i)^2}, \max_{e_j} \frac{1/\lambda_j}{((1-\epsilon)\mu_1 - \mu_{k+1})^2} \right\}\\
    %%%%%%%%%%%%%%%%%%%
    & \geq \max\left\{\max_{e_i \in G_\epsilon^c} \frac{1/\lambda_i }{((1-\epsilon)\mu_1 - \mu_i)^2}, \max_{e_j} \frac{1/\lambda_j}{(\mu_1 - \frac{\mu_{k+1}}{1-\epsilon})^2} \right\}\\
    %%%%%%%%%%%%%%%%%%%
    & = \max\left\{\max_{e_i \in G_\epsilon^c} \frac{1/\lambda_i }{((1-\epsilon)\mu_1 - \mu_i)^2}, \max_{e_j} \frac{1/\lambda_j}{\left(\left(\mu_1 - \frac{\Tilde{\beta}_\epsilon}{1-\epsilon}\right) - \mu_1\right)^2} \right\}\\
    %%%%%%%%%%%%%%%%%%%
    & = \max\left\{\max_{e_i \in G_\epsilon^c} \frac{1/\lambda_i }{((1-\epsilon)\mu_1 - \mu_i)^2}, \max_{e_j} \frac{1/\lambda_j}{\left(\left(\mu_1 + \frac{\Tilde{\beta}_\epsilon}{1-\epsilon}\right) - \mu_1\right)^2} \right\}\\
    %%%%%%%%%%%%%%%%%%%
    & \geq \max\left\{\max_{e_i \in G_\epsilon^c} \frac{1/\lambda_i }{((1-\epsilon)\mu_1 - \mu_i)^2}, \max_{e_j} \frac{1/\lambda_j}{\left(\left(\mu_1 + \frac{\Tilde{\beta}_\epsilon}{1-\epsilon}\right) - \mu_j\right)^2} \right\}
\end{align*}
where the final step follows since $\mu_1 + \frac{\Tilde{\beta}_\epsilon}{1-\epsilon} > \mu_i \forall i$ and $\mu_j \leq \mu_1$. The third inequality follows by the same approach as taken for $H^\texttt{MILK1}$ of multiplying the denominator by $(1-\epsilon)^2 / (1-\epsilon)^2$.

Hence, we have that 
\begin{align*}
    H(\theta_\ast) \geq \min_\lambda\max_i \max\left\{\frac{\frac{1}{\lambda_i}}{((1-\epsilon)\mu_1 - \mu_i)^2}, \frac{\frac{1}{\lambda_i}}{(\mu_1 + \frac{\Tilde{\alpha}_\epsilon}{1-\epsilon} -\mu_i)^2}, \frac{\frac{1}{\lambda_i}}{(\mu_1 + \frac{\Tilde{\beta}_\epsilon}{1-\epsilon} -\mu_i)^2} \right\}.
\end{align*}
Solving for $\lambda$ gives
\begin{align*}
    H(\theta_\ast) \geq \sum_{i=1}^n\max\left\{\frac{1}{((1-\epsilon)\mu_1 - \mu_i)^2}, \frac{1}{(\mu_1 + \frac{\Tilde{\alpha}_\epsilon}{1-\epsilon} -\mu_i)^2}, \frac{1}{(\mu_1 + \frac{\Tilde{\beta}_\epsilon}{1-\epsilon} -\mu_i)^2} \right\} = c_1\cdot H_{(ST)^2}
\end{align*}
for a constant $c_1$.
To upper bound $H^\texttt{MILK}(\theta_\ast)$, we may choose a specific $\lambda$. Choosing 
\begin{align*}
    \lambda_i := \frac{\max\{((1-\epsilon)\mu_1 - \mu_i)^{-2}, (\mu_1 + \frac{\Tilde{\alpha}_\epsilon}{1-\epsilon} -\mu_i)^{-2}, (\mu_1 + \frac{\Tilde{\beta}_\epsilon}{1-\epsilon} -\mu_i)^{-2}\}}{\sum_j\max\{((1-\epsilon)\mu_1 - \mu_j)^{-2}, (\mu_1 + \frac{\Tilde{\alpha}_\epsilon}{1-\epsilon} -\mu_j)^{-2}, (\mu_1 + \frac{\Tilde{\beta}_\epsilon}{1-\epsilon} -\mu_j)^{-2}\}},
\end{align*}
a similar computation shows that $H^\texttt{MILK}(\theta_\ast) \leq c_2H_{(ST)^2}$ for a constant $c_2$. 
\end{proof}

\subsection{Upper Bound}

First we restate Theorem~\ref{thm:MILK_complex} bounding the sample complexity of \texttt{MILK}. 
\begin{theorem}
Fix $\delta >0$, threshold $\alpha > 0$, tolerance $\widetilde\beta$, and regularization $\gamma>0$. 
Define the quantities $\Delta_{\min}^{\text{Above}}(\epsilon) = \min_{\bx \in G_\epsilon}\min_{\bx'}  \theta_\ast^{\top}(\phi(\bx) - (1-\epsilon)\phi(\bx')) $ and $\Delta_{\min}^{\text{Below}}(\epsilon) = \min_{\bx \in G_\epsilon^c}\max_{\bx': (\phi(\bx) - (1-\epsilon)\phi(\bx'))^{\top}\theta_\ast < 0}  (\phi(\bx) - (1-\epsilon)\phi(\bx'))^{\top}\theta_\ast$, and $\Delta_{\min} = \min\left\{\Delta_{\min}^{\text{Above}}(\epsilon), \Delta_{\min}^{\text{Below}}(\epsilon) \right\}$.
Define also 
\begin{align*}
    \bar \beta(\epsilon) = \min \{ \beta>0 : 4(\sqrt{\gamma} \|\theta_* \| + h)(2+ \sqrt{f(\X,\left\{\by \in \Y^\epsilon(\X \times \X): |\by^{\top}\theta_\ast| \leq  \beta\right\};\gamma)}) \leq  \beta \}.
\end{align*}
With probability $1-\delta$, \texttt{MILK} returns a set $\widehat{\mc{R}} = (\X \setminus \hB_t)$ at a time $T_\delta$ such that 
\begin{align*}
    \{\bx \in \X: f(\bx) \geq (1-\epsilon)f(\bx_\ast) + \bar\beta(\epsilon)\} \subseteq \widehat{\mc{R}} \subseteq \{\bx \in \X: f(\bx) \geq (1-\epsilon)f(\bx_\ast) - \widetilde\beta - \bar\beta(\epsilon)\}
\end{align*}
and for any $\alpha$, $\widetilde\beta$ such that $\max(\Delta_{\min}(\epsilon), \widetilde\beta) \geq \bar\beta(\epsilon)$
\begin{align*}
    T_\delta \leq & 
     256(B^2 + \sigma^2) H^\texttt{MILK}(\theta_\ast) \log\left(\frac{4|\X|^2\lceil\log_2(4(\Delta_{\min}(\epsilon)\vee \widetilde{\beta})^{-1})\rceil^2}{\delta}\right) + 2\log\left(\frac{|\X|}{\delta}\right)\lceil\log_2(4(\Delta_{\min}(\epsilon)\vee \widetilde{\beta})^{-1})\rceil
\end{align*}
for a sufficiently large constant $c$
where $H^\texttt{MILK}(\theta_\ast) = \min_{\lambda \in \tX}\max\left\{H_\lambda^\texttt{MILK1}(\theta_\ast), H_\lambda^\texttt{MILK2}(\theta_\ast)\right\}$ and 
\begin{align*}
    H_\lambda^\texttt{MILK1}(\theta_\ast) := \max_{\bx \in G_\epsilon}\max_{\bx'\in \X} \frac{\|\phi(\bx) - (1-\epsilon)\phi(\bx')\|_{(A(\lambda)+\gamma I)^{-1}}^2}{\max\{((\phi(\bx) - (1-\epsilon)\phi(\bx'))^{\top}\theta_\ast)^2, \widetilde \beta^2\}}
\end{align*}    
\begin{align*}
    H_\lambda^\texttt{MILK2}(\theta_\ast) := \max_{\bx \in G_\epsilon^c}\max_{\bx'} \frac{\|\phi(\bx) - (1-\epsilon)\phi(\bx')\|_{(A(\lambda)+\gamma I)^{-1}}^2}{\max\{((\phi(\bx) - (1-\epsilon)\phi(\bx_\ast))^{\top}\theta_\ast)^2, \widetilde \beta^2\}}.
\end{align*}
\end{theorem}

Now we show a high probability concentration result that we will use for the remainder of this section. 

\begin{lemma}\label{lem:high_prob_event_RIPS_milk}
For any $\V \subset \Y^\epsilon(\X\times \X)$ define $f(\X,\V;\gamma) = \min_{\lambda\in\triangle_{\X}} \max_{\bv \in \V} \| \bv \|_{( \sum_{x \in \X} \lambda_x \phi(x)\phi(x)^\top + \gamma I)^{-1}}^2$.\\
In each round $t$, define the event
$$\cE_t = \{|\by^T(\htheta_t - \theta_\ast)| \leq 2^{-t} + \left(\sqrt{\gamma}\|\theta_\ast\| + h \right) \sqrt{f(\X, \Y^\epsilon(\Y^\epsilon(\A_t)); \gamma)} \ \forall \ \by \in \Y^\epsilon(\A_t)\}$$\\
Holds $\P(\bigcup_{t=1}^\infty \cE_t^c) \leq \delta$.
\end{lemma}

\begin{proof}
Using Theorem~\ref{thm:robust_estimator}, for any $\by \in \Y^\epsilon(\A_t)$
we have that with probability at least $1-\delta_t/|\X|^2$
\begin{align*}
    |\by^T(\htheta_t - \theta_\ast)| & \leq \|\by\|_{\left(\sum_{\bx\in\X}\lambda_{\bx} \phi(\bx)\phi(\bx)^T+\gamma I\right)^{-1}}\left(\sqrt{\gamma}\|\theta_\ast\| + h + c \sqrt{\tfrac{(B^2 + \sigma^2)}{N_t}\log(2t^2|\X|^2/\delta)}\right) \\
    & \leq \sqrt{f(\X, \Y^\epsilon(\A_t); \gamma)}\left(\sqrt{\gamma}\|\theta_\ast\| + h + 2^{-t}/\sqrt{f(\X, \Y^\epsilon(\A_t); \gamma)}\right)  \\
    % & = \sqrt{\frac{2\|\bx_i - \bx_j\|_{\left(\sum_{x\in \X}\lambda_t(\bx)\phi(\bx)\phi(\bx)^T\right)^{-1}}^2(1+\varphi)\log\left(\frac{2|\X|^2}{\delta_t}\right)}{2 \cdot 2^{2t} \max \|\bx_i - \bx_j\|_{\left(\sum_{x\in \X}\lambda_t(\bx)\phi(\bx)\phi(\bx)^T\right)^{-1}}^2 (1+\varphi) \log\left(\frac{2|\X|^2}{\delta_t}\right)}} \\
    & \leq 2^{-t} + \left(\sqrt{\gamma}\|\theta_\ast\| + h \right) \sqrt{f(\X, \Y^\epsilon(\A_t); \gamma)}
\end{align*}
Since $|\Y^\epsilon(\A_t)| \leq |\X|^2$, $\cE_t$ holds for all $\by \in \Y^\epsilon(\A_t)$ with probability $1-\delta_t$ via a union bound. Taking a second union bound over rounds, we have that 
$$\P\left(\bigcup_{t=1}^\infty \cE_t^c\right) \leq \sum_{t=1}^\infty \P(\cE_t^c) \leq \sum_{t=1}^\infty \delta_t  = \sum_{t=1}^\infty \frac{\delta}{2t^2} \leq \delta  $$
\end{proof}

Define 
\begin{align*}
    \bar t &= \max \{ t : (\sqrt{\gamma} \|\theta_* \|_2 + h)(2+ \sqrt{f(\X,\left\{\by \in \Y^\epsilon(\X\times\X): |\by^T\theta_\ast| \leq 2^{-t + 2}\right\};\gamma)}) \leq 2^{-t} \} .
\end{align*}
As we will see in Lemmas~\ref{lem:milk_active_good} and \ref{lem:milk_active_bad},
$$\Y^\epsilon(\A_t) \subset \left\{\by \in \Y^\epsilon(\X\times\X): |\by^T\theta_\ast| \leq 2^{-t + 1}\right\}.$$
Thus for $t\leq \bar t$, holds on $\bigcap_t \cE_t$ that 
\begin{align*}
    \forall \by \in \Y^\epsilon(\A_t) \;, \; |\by^T(\htheta_t - \theta_\ast)| \leq 2 \cdot 2^{-t}.
\end{align*}

\begin{lemma}\label{lem:milk_good_arms}
On $\bigcap_t \cE_t$, when $t\leq \bar t$ holds $\hG_t \subset G^{\phi}_\epsilon :
= \{ \bx  \ : \   (\phi(\bx) - (1-\epsilon)\phi(\bx'))^T \theta_\ast > 0 \ \forall \ \bx' \in \X\}$. 
\end{lemma}

\begin{proof}
\begin{align*}
    \bx \in \hG_t & \iff \forall \bx' \ \exists t_{x'} \leq \bar t  \ : \ (\phi(\bx) - (1-\epsilon)\phi(\bx'))^T\htheta_{t_{x'}} \geq  2 \cdot 2^{-t_{x'}}\\
    & \iff \forall \bx' \ \exists t_{x'} \leq \bar t \ : \  (\phi(\bx) - (1-\epsilon)\phi(\bx'))^T(\htheta_{t_{x'}} - \theta_\ast) + (\phi(\bx) - (1-\epsilon)\phi(\bx'))^T \theta_\ast \geq 2 \cdot 2^{-t_{x'}} \\
    & \stackrel{\bigcap_t\cE_{t}}{\implies}  \forall \bx'  \ : \   (\phi(\bx) - (1-\epsilon)\phi(\bx'))^T \theta_\ast > 0\\
    & \iff \bx \in G^\phi_\epsilon.
\end{align*}
\end{proof}

% \begin{lemma}\label{lem:milk_all_good_in}
% Define $\Delta_{\min}^{\text{Above}}(\epsilon) := \min_{\bx'}\min_{\bx \in G^\phi_\epsilon} \epsilon - \theta_\ast^T(\phi(\bx') - \phi(\bx)) $. 
% On $\bigcap_t \cE_t$, $G^\phi_\epsilon \subset \hG_t $ for $\bar t \geq t \geq \log_2(4/\Delta_{\min}^{\text{Above}}(\epsilon))$.
% \end{lemma}

% \begin{proof}
% For any $\bx \in G^\phi_\epsilon$ and any $\bx'$, define $t_{x'} := \log_2(4 (\epsilon - (\phi(\bx') - \phi(\bx))^T\theta_\ast)^{-1})$. The condition $\bar t \geq t \geq t_{x'}$ implies
% \begin{align*}
%     (\phi(\bx') - \phi(\bx))^T\htheta_t & = (\phi(\bx') - \phi(\bx))^T(\htheta_t - \theta_\ast) + (\phi(\bx') - \phi(\bx))^T \theta_\ast \\
%     & \stackrel{\bigcap_t\cE_{t}}{<} (\phi(\bx') - \phi(\bx))^T \theta_\ast + 2 \cdot 2^{-t} \\
%     & \leq \epsilon - 2 \cdot 2^{-t}
% \end{align*}
% and the left hand side of the equivalence implies that $\phi(\bx') - \phi(\bx)$ is removed from $\A$. Hence, $\bar t \geq t \geq \max_{\bx'}\log(4 (\epsilon - (\phi(\bx') - \phi(\bx))^T\theta_\ast)^{-1})$ implies that $\not \exists \bx' : (\phi(\bx') - \phi(\bx)) \in \A$ and hence $\bx \in \hG_t$. Further maximizing over all $\bx \in G^\phi_\epsilon$ completes the proof. 
% \end{proof}

\begin{lemma}\label{lem:milk_bad_arms}
On $\bigcap_t \cE_t$, when $t\leq \bar t$ holds $\hB_t \subset (G^\phi_\epsilon)^c$.
\end{lemma}

\begin{proof}
\begin{align*}
    \bx \in \hB_t & \iff \exists \bx', t_{x'} \leq \bar t \ : \ (\phi(\bx) - (1-\epsilon)\phi(\bx'))^T\htheta_t \leq - 2 \cdot 2^{-t_{x'}} \\
    & \iff \exists \bx', t_{x'}\leq \bar t \ : \ (\phi(\bx) - (1-\epsilon)\phi(\bx'))^T(\htheta_t - \theta_\ast) + (\phi(\bx) - (1-\epsilon)\phi(\bx'))^T\theta_\ast \leq - 2 \cdot 2^{-t_{x'}} \\
    & \stackrel{\bigcap_t\cE_{t}}{\implies} \exists \bx' \ : \ (\phi(\bx) - (1-\epsilon)\phi(\bx'))^T\theta_\ast > \epsilon \\
    & \iff \bx \in (G^\phi_\epsilon)^c.
\end{align*}
\end{proof}

\begin{lemma}\label{lem:milk_active_good}
On the event $\bigcap_t \cE_t$ for $t \leq \bar{t}$, 
\begin{align*}
    \{(\bx, \bx'): (\bx, \bx'), \bx \in G_\epsilon^\phi\} \subset &\left\{(\bx, \bx') \bigg| |(\phi(\bx) - (1-\epsilon)\phi(\bx'))^T\theta_\ast| \leq 2^{-t + 2}\right\} =: \S_t^\text{Above}
\end{align*}
\end{lemma}

\begin{proof}
On $\bigcap_t\cE_t$ for $t \leq \bar{t}$, for any $\by \in \A_t$
$$|\by^T\htheta_t| \geq |\by^T\theta_\ast| - |\by^T(\htheta_t - \theta_\ast)| \stackrel{\cE_t}{\geq } |\by^T\theta_\ast| - 2 \cdot 2^{-t}.$$
For $\by$ such that $|\by^T\theta_\ast| \geq 2 \cdot 2^{-t+1}$, the above implies that 
$$|\by^T\htheta_{t}| \geq 2 \cdot 2^{-t}.$$
By the elimination condition, this implies that $\by$ is removed from $\A_{t}$. 
Hence
$$\A_{t+1} \subset \left\{\by \in \Y(\X): |\by^T\theta_\ast - \epsilon| \leq 2 \cdot 2^{-t + 1}\right\}.$$
Specializing this argument to $\{(\bx, \bx'): (\bx, \bx'), \bx \in G_\epsilon^\phi\} \subset \A_t$ completes the proof. 
\end{proof}

% \begin{lemma}\label{lem:milk_active_bad}
% On the event $\bigcap_t \cE_t$, 
% \begin{align*}
%     \Y^\epsilon(\A_t)\cap \left(\bigcup_{\bx \in (G^\phi_\epsilon)^c}\Y_{z}\right) \subset & \left\{\phi(\bx') - \phi(\bx) \bigg| |(\phi(\bx') - \phi(\bx))^T\theta_\ast - \epsilon| \leq 2^{-t + 3}\right.\\
%     &\hspace{1cm}\left.\text{ and } \bx \in (G^\phi_\epsilon)^c
%     \text{ s.t. } \max_{\bx'': (\phi(\bx'') - \phi(\bx))^T\theta_\ast > \epsilon}\{(\phi(\bx'') - \phi(\bx))^T\theta_\ast - \epsilon\} \leq 2^{-t + 3} \right\} =: \S_t^\text{Below}
% \end{align*}
% \end{lemma}
\begin{lemma}\label{lem:milk_active_bad}
On the event $\bigcap_t \cE_t$ for $t \leq \bar{t}$, 
\begin{align*}
    \{(\bx, \bx'): (\bx, \bx'), \bx \in (G_\epsilon^\phi)^c\}  \subset & \left\{(\bx, \bx') \bigg| |(\phi(\bx) - (1-\epsilon)\phi(\bx'))^T\theta_\ast - \epsilon| \leq 2^{-t + 2}\right.\\
    &\left.\text{ and } \{(\phi(\bx) - (1-\epsilon)\phi(\bx_\ast))^T\theta_\ast\} \geq -2^{-t + 2} \right\} =: \S_t^\text{Below}
\end{align*}
\end{lemma}

\begin{proof}
The guarantee that $|(\phi(\bx) - (1-\epsilon)\phi(\bx'))^T\theta_\ast - \epsilon| \leq 2^{-t + 2}$ for any $(\bx, \bx') \in\A_t$ follows by the same argument as Lemma~\ref{lem:milk_active_good}. For the additional statement, that $(\phi(\bx) - (1-\epsilon)\phi(\bx_\ast))^T\theta_\ast \geq -2^{-t + 2}$, note that if 
\begin{align*}
    (\phi(\bx) - (1-\epsilon)\phi(\bx_\ast))^T\htheta_t \leq -2^{-t + 1}
\end{align*}
then the pair $(\bx, \bx_\ast)$ is eliminated from $\A_t$. If 
\begin{align*}
    (\phi(\bx) - (1-\epsilon)\phi(\bx_\ast))^T\htheta_t \leq -2^{-t + 2},
\end{align*}
then using this and the event $\bigcap_t \cE_t$
\begin{align*}
(\phi(\bx) - 
   (1-\epsilon)\phi(\bx_\ast))^T\htheta_t & =  (\phi(\bx) - (1-\epsilon)\phi(\bx_\ast))^T(\htheta_t  - \theta_\ast)+ (1-\epsilon)\phi(\bx_\ast))^T\theta_\ast \leq -2^{-t+1}.
\end{align*}
Hence, the only pairs $(\bx, \bx_\ast)$ that remain in $\A_t$ where $\bx_\ast \in (G_\epsilon^\phi)^c$ are such that $(\phi(\bx) - (1-\epsilon)\phi(\bx_\ast))^T\theta_\ast\} \geq -2^{-t + 2}$. We conclude by noting that the above argument for $\bx_\ast$ could be repeated for any $\bx'$ such that $(\phi(\bx) - (1-\epsilon)\phi(\bx'))^T\theta_\ast < 0 $. 
\end{proof}

% \begin{proof}
% We have that $\Y^\epsilon(\A_t)\cap \left(\bigcup_{\bx \in (G^\phi_\epsilon)^c}\Y_{z}\right) = \bigcup_{\bx \in (G^\phi_\epsilon)^c} (\Y^\epsilon(\A_t) \cap \Y_z)$. By Lemma~\ref{lem:milk_all_bad_in}, we have that for any $\bx \in (G^\phi_\epsilon)^c$
% $$t \geq \min_{\bx': (\phi(\bx') - \phi(\bx))^T\theta_\ast > \epsilon}\log_2\left(\frac{4}{(\phi(\bx') - \phi(\bx))^T\theta_\ast - \epsilon} \right) $$
% implies that there exists $\bx'$ such that $(\phi(\bx') - \phi(\bx))^T\htheta_t > \epsilon + 2 \cdot  2^{-t}$. Hence, the elimination condition is satisfied and the remainder of $\Y_z$ is removed from $\Y^\epsilon(\A_t)$. 
% Therefore, $\Y_z \cap \A_{t+1} = \emptyset$.
% Hence,
% \begin{align*}
%     \bigcup_{\bx \in (G^\phi_\epsilon)^c} (\Y^\epsilon(\A_t) \cap \Y_z) \subset \left\{\phi(\bx') - \phi(\bx) \bigg| \bx \in (G^\phi_\epsilon)^c \text{ s.t. } \max_{\bx'': (\phi(\bx'') - \phi(\bx))^T\theta_\ast > \epsilon}\{(\phi(\bx'') - \phi(\bx))^T\theta_\ast - \epsilon\} \leq 2^{-(t - 3)} \right\}
% \end{align*}
% As in Lemma~\ref{lem:milk_active_good}, we have that 
% $$\A_{t+1} \subset \left\{\by \in \Y(\X): |\by^T\theta_\ast - \epsilon| \leq 2 \cdot 2^{-t + 1}\right\}.$$
% Intersecting this with the above step completes the proof. 
% \end{proof}

\textbf{Remark:} Lemmas~\ref{lem:milk_active_good} and \ref{lem:milk_active_bad} jointly imply that $\A_t \subset \S_t^\text{Above} \cup \S_t^\text{Below}=: \S_t$ for $t \leq \bar{t}$. Furthermore, $f(\X, \Y^\epsilon(\A_t), \gamma) \leq f(\X, \Y^\epsilon(\S_t), \gamma)$.

\textbf{Remark:}\\
The algorithm stops on either of two conditions. On one hand if $t \geq \lceil \log_2(4/\widetilde{\beta})\rceil =: t_\beta$, then it has achieved precision $\widetilde{\beta}$ as desired and it terminates. 
Otherwise, it terminates if $\hG_t \cup \hB_t = \X$. This occurs when $\widetilde{\beta}$ is very small. 
Define the quantities $\Delta_{\min}^{\text{Above}}(\epsilon) = \min_{\bx \in G_\epsilon}\min_{\bx'}  \theta_\ast^{\top}(\phi(\bx) - (1-\epsilon)\phi(\bx')) $ and $\Delta_{\min}^{\text{Below}}(\epsilon) = \min_{\bx \in G_\epsilon^c}\max_{\bx': (\phi(\bx) - (1-\epsilon)\phi(\bx'))^{\top}\theta_\ast < 0}  (\phi(\bx) - (1-\epsilon)\phi(\bx'))^{\top}\theta_\ast$, and $\Delta_{\min}(\epsilon) = \min\left\{\Delta_{\min}^{\text{Above}}(\epsilon), \Delta_{\min}^{\text{Below}}(\epsilon) \right\}$.
Recall
% To conclude on the correctness of the algorithm, we need to define a $\bar{\epsilon}$ such that such that for all $\bar \epsilon \leq \epsilon$ holds $\bar t \geq  \log_2(4/\Delta_{\min}^{\text{Below}}(\epsilon))$ and $\bar t \geq  \log_2(4/\Delta_{\min}^{\text{Above}}(\epsilon))$. 
% Recall 
\begin{align*}
    \bar t &= \max \{ t : (\sqrt{\gamma} \|\theta_* \|_2 + h)(2+ \sqrt{f(\X,\left\{\by \in \Y^\epsilon(\X\times\X): |\by^T\theta_\ast| \leq 4\cdot 2^{-t}\right\};\gamma)}) \leq 2^{-t} \} \\
    &= \max \{ t : 4(\sqrt{\gamma} \|\theta_* \|_2 + h)(2+ \sqrt{f(\X,\left\{\by \in \Y^\epsilon(\X\times\X): |\by^T\theta_\ast| \leq 4\cdot 2^{-t}\right\};\gamma)}) \leq 4\cdot 2^{-t} \} \\
    &= -2 + \max \{ t : 4(\sqrt{\gamma} \|\theta_* \|_2 + h)(2+ \sqrt{f(\X,\left\{\by \in \Y^\epsilon(\X\times\X): |\by^T\theta_\ast| \leq  2^{-t}\right\};\gamma)}) \leq  2^{-t} \} \\
    &= -3 + \log_2 (\min \{ \beta>0 : 4(\sqrt{\gamma} \|\theta_* \|_2 + h)(2+ \sqrt{f(\X,\left\{\by \in \Y^\epsilon(\X\times\X): |\by^T\theta_\ast| \leq  \beta\right\};\gamma)}) \leq  \beta \}) .
\end{align*}
This defines 
\begin{align*}
    \bar\beta(\epsilon) = \min \{ \beta>0 : 4(\sqrt{\gamma} \|\theta_* \|_2 + h)(2+ \sqrt{f(\X,\left\{\by \in\Y^\epsilon(\X\times\X): |\by^T\theta_\ast| \leq  \beta\right\};\gamma)}) \leq  \beta \}.
\end{align*}
Let $t_{\max}$ denote the random variable of the last round before the algorithm terminates. 
The following Lemmas give a guarantee on the set $\X \setminus \hB_t$ at termination. 

\begin{lemma}\label{lem:milk_contains}
On the event $\bigcap_{t=1}^\infty \cE_t$, \texttt{MILK} returns a set $(\X \setminus \hB_{t_{\max}})$ such that
$\{\bx: f(\bx) > (1-\epsilon)f(\bx_\ast) + \bar\beta(\epsilon)\} \subset (\X \setminus \hB_{t_{\max}})$. 
\end{lemma}

\begin{proof}
Take any $\bx$ such that $f(\bx) > (1-\epsilon)f(\bx_\ast) + \bar\beta(\alpha)$ and recall that by assumption $|f(\bx) -\phi(\bx)^T\theta_\ast| \leq h$ for all $\bx \in \X$. 
We consider two cases. 
In the first case, assume that $t_{\max} \leq \bar{t}$. 
We claim that in this case $\not\exists t$ such that $\bx \in \hB_t$. We prove this by contradiction. Assume not. Then $\exists t$ and a $\bx'$ such that 
\begin{align*}
    &\htheta_t^T(\phi(\bx) - (1-\epsilon)\phi(\bx')) < - 2^{-t + 1} \\
    & \iff (\phi(\bx) - (1-\epsilon)\phi(\bx'))^T(\htheta_t - \theta_\ast) + (\phi(\bx) - (1-\epsilon)\phi(\bx'))^T\theta_\ast < - 2^{-t + 1} \\
    & \stackrel{\cE_t, t_{\max} \leq \bar{t}}{\implies} -2^{-t+1}  +  (\phi(\bx) - (1-\epsilon)\phi(\bx'))\theta_\ast < - 2^{-t + 1} \\
    & \iff (\phi(\bx) - (1-\epsilon)\phi(\bx'))\theta_\ast < 0 \\
    & \implies f(\bx) - (1-\epsilon)f(\bx') < h + (1-\epsilon)h
\end{align*}
Recall that we have assumed that $f(\bx) > (1-\epsilon)f(\bx_\ast) + \bar\beta(\alpha)$ and $\bar\beta(\epsilon) > 4h$ by definition. Hence, this implies that
\begin{align*}
   (1-\epsilon)f(\bx_\ast) - (1-\epsilon)f(\bx') < h + (1-\epsilon)h - \bar\beta(\alpha) < 0
\end{align*}
which is a contradiction since $f(\bx_\ast) \geq f(\bx')$ by definition. Hence, we have shown in the case that $t_{\max} \leq \bar{t}$, $\{\bx: f(\bx) > (1-\epsilon)f(\bx_\ast) + \bar\beta(\epsilon)\} \subset (\X \setminus \hB_{t_{\max}})$.

In the second case, assume that $t_{\max} > \bar{t}$ and take $\bx$ such that $f(\bx) > (1-\epsilon)f(\bx_\ast) + \bar\beta(\alpha)$. 
We claim that $\bx \in \hG_{\bar{t}}$ and hence $(\bx, \bx') \not\in \A_{t}$ for any $t > \bar{t}$ and thus $\bx$ is never added to $\hB_t$.
This occurs if for every $\phi(\bx')$
\begin{align*}
    &(\phi(\bx)-(1-\epsilon)\phi(\bx'))^T\htheta_{\bar{t}} > 2^{-\bar{t} + 1} \\
    & \iff (\phi(\bx)-(1-\epsilon)\phi(\bx'))^T(\htheta_{\bar{t}} - \theta_\ast) + (\phi(\bx)-(1-\epsilon)\phi(\bx'))^T\theta_\ast > 2^{-\bar{t} + 1} \\
    & \stackrel{\cE_{\bar{t}}}{\impliedby} -2^{-\bar{t}+1} + (\phi(\bx)-(1-\epsilon)\phi(\bx'))^T\theta_\ast \geq  2^{-\bar{t} + 1} \\
    & \iff (\phi(\bx)-(1-\epsilon)\phi(\bx'))^T\theta_\ast \geq  2^{-\bar{t} + 2} \\
    & \iff (\phi(\bx)-(1-\epsilon)\phi(\bx'))^T\theta_\ast \geq  0.5 \bar\beta(\epsilon)\\
    & \impliedby
    f(\bx) - (1-\epsilon)f(\bx')
     \geq  0.5 \bar\beta(\epsilon) + h + (1-\epsilon)h
\end{align*}
where the penultimate step follows by definition of $\bar\beta(\epsilon)$. 
Recall that $f(\bx) > (1-\epsilon)f(\bx_\ast) + \bar\beta(\alpha)$.
Hence, the above is implied by 
\begin{align*}
    &(1-\epsilon)f(\bx_\ast) + \bar\beta(\alpha) - (1-\epsilon)f(\bx') \geq  0.5 \bar\beta(\epsilon) + h + (1-\epsilon)h \\
    & \impliedby \bar\beta(\epsilon) \geq 0.5 \bar\beta(\epsilon) + h + (1-\epsilon)h
\end{align*}
where the final step follows by noting that $f(\bx_\ast) \geq f(\bx')$ for any $\bx'$. The final statement is true since $\bar\beta(\epsilon)$ and thus implies the claim. Therefore, we have shown that $\bx \in \hG_{\bar{t}}$ and is therefore not added to $\hB_t$ in a later round. These two cases together complete the proof. 
\end{proof}

\begin{lemma}\label{lem:milk_is_contained}
On the event $\bigcap_{t=1}^\infty \cE_t$, \texttt{MILK} returns a set $(\X \setminus \hB_{t_{\max}})$ such that
$(\X \setminus \hB_{t_{\max}}) \subset \{\bx: f(\bx) > (1-\epsilon)f(\bx_\ast) - \bar{\beta}(\epsilon) - \widetilde{\beta}\}$.
\end{lemma}

\begin{proof}
Take any $\bx$ such that $f(\bx) < (1-\epsilon)f(\bx_\ast) - \bar\beta(\epsilon) - \widetilde{\beta}$. We claim that there exists a $t \leq t_{\max}$ such that $\bx$ is added to $\hB_t$ which implies that $\bx \not\in (\X \setminus \hB_{t_{\max}})$. Suppose for contradiction that this is not the case. Then for all $t \leq t_{\max}$, 
\begin{align*}
    &\htheta_t^T(\phi(\bx) - (1-\epsilon)\phi(\bx_\ast)) > - 2^{-t + 1}  \\
    & \iff (\phi(\bx) - (1-\epsilon)\phi(\bx_\ast))^T(\htheta_t - \theta_\ast) + (\phi(\bx) - (1-\epsilon)\phi(\bx_\ast))^T\theta_\ast > - 2^{-t + 1} \\
    & \stackrel{\cE_t}{\implies} 2^{-t} + \left(\sqrt{\gamma}\|\theta_\ast\| + h \right) \sqrt{f(\X, \A_t; \gamma)} +  (\phi(\bx) - (1-\epsilon)\phi(\bx_\ast))^T\theta_\ast >  - 2^{-t + 1} \\
    & \implies \left(\sqrt{\gamma}\|\theta_\ast\| + h \right) \sqrt{f(\X, \A_t; \gamma)} +  f(\bx) - (1-\epsilon)f(\bx_\ast) + h + (1 -\epsilon)h >  - 2^{-t+1} - 2^{-t} \\
    & \implies f(\bx) - (1-\epsilon)f(\bx_\ast)>  - 2^{-t+1} - 2^{-t} - h -(1 -\epsilon)h - \left(\sqrt{\gamma}\|\theta_\ast\| + h \right) \sqrt{f(\X, \S_t; \gamma)}.
\end{align*}
Plugging in $f(\bx) < (1-\epsilon)f(\bx_\ast) - \bar\beta(\epsilon) - \widetilde{\beta}$, the above implies
\begin{align}\label{eq:milk_lem_event}
     \bar\beta(\epsilon) + \widetilde{\beta} <  2^{-t+1} + 2^{-t} + h + (1-\epsilon)h + \left(\sqrt{\gamma}\|\theta_\ast\| + h \right) \sqrt{f(\X, \S_t; \gamma)}
\end{align}
Next, recall that \texttt{MILK} terminates either on the condition that $t = \lceil\log_2(4/\widetilde{\beta})\rceil$ or that $\hG_t \cup \hB_t = \X$. Using this, we brake our analysis into cases.

Case 1: $t_{\max} = \lceil\log_2(4/\widetilde{\beta})\rceil \leq \bar{t}$.

In this case, \texttt{MILK} stops due to the $\widetilde{\beta}$ tolerance in a round before $\bar{t}$. For $t \leq \bar{t}$, we have that $2^{-t} \geq  + \left(\sqrt{\gamma}\|\theta_\ast\| + h \right) \sqrt{f(\X, \S_t; \gamma)}$. Hence, the above implies that 
\begin{align*}
    \bar\beta(\alpha) + \widetilde{\beta} <  2^{-t+2} + h+ (1-\epsilon)h.
\end{align*}
As we have assumed this condition for all $t \leq t_{\max}$, we may plug in $t_{\max}$ which implies 
\begin{align*}
    \bar\beta(\alpha) + \widetilde{\beta} < \widetilde{\beta} + h + (1-\epsilon)h.
\end{align*}
As $ \bar\beta(\alpha) > 4h$, this is a contradiction. Hence there must exist a $t$ such that $\bx \in \hB_t$. 

Case 2: $ t_{\max} \leq \bar{t} <  \lceil\log_2(4/\widetilde{\beta})\rceil$. 

In this case, \texttt{MILK} terminates before round $t = \lceil\log_2(4/\widetilde{\beta})\rceil$. Hence, it does so on the condition that $\hG_t \cup \hB_t = \X$. Note that for $f(\bx) < \alpha - \bar\beta(\alpha) - \widetilde{\beta}$, we have that $\bx \in (G_\alpha^\phi)^c$ since $\bar{\beta}(\alpha) > h$ and $\widetilde{\beta} \geq 0$. If we terminate before round $\bar{t}$, we have by Lemma~\ref{lem:milk_bad_arms} that $(G_\alpha^\phi)^c \subset \hB_t$ which implies that $\bx \in \hB_{t_{\max}}$. This contradicts the assumption that $\not\exists t: \bx \in \hB_t$. 

Case 3: $\bar{t} < t_{\max}$.

In this case, \texttt{MILK} terminates at a round after $\bar{t}$. In this setting, we argue that $\bx \in \hB_{\bar{t}}$. Recall that for any $t \leq \bar{t}$, \eqref{eq:milk_lem_event} simplifies to
\begin{align*}
    \bar{\beta}(\alpha) + \widetilde{\beta} < 2^{-t+2} + h + (1-\epsilon)h. 
\end{align*}
Plugging in $\bar{t}$, and noting that $2^{-\bar{t} + 2} = \frac{1}{2}\bar{\beta}(\alpha)$, the above implies
\begin{align*}
    \bar{\beta}(\alpha) + \widetilde{\beta} < \frac{1}{2}\bar{\beta}(\alpha) + h + (1-\epsilon)h. 
\end{align*}
Noting that $\bar{\beta}(\alpha) > 4h$, shows that the above is a contradiction. Hence, there exists a $t \leq \bar{t}$ such that $\bx \in \hB_t$. 

Therefore, in all cases we have shown that for any $\bx$ such that $f(\bx) < \alpha - \bar{\beta}(\alpha) - \widetilde{\beta}$, $\bx \in \hB_t$. Therefore, for the returned set $\X \setminus \hB_{t_{\max}}$, we have that \begin{align*}
    (\X \setminus \hB_{t_{\max}}) \subset \{\bx: f(\bx) > \alpha - \bar{\beta}(\alpha) - \widetilde{\beta}\}. 
\end{align*}
\end{proof}

\begin{proof}[Proof of Theorem~\ref{thm:MILK_complex}]
Throughout, assume the high probability event $\bigcap_T \cE_t$. 
By Lemmas~\ref{lem:milk_contains} and \ref{lem:milk_is_contained} in conjunction with the high probability event $\bigcap \cE_t$ we have correctness. It remains to control the sample complexity of \texttt{MILK}.
Recall that we have assumed that $\max(\Delta_{\min}(\epsilon), \widetilde{\beta}) \geq \bar\beta(\epsilon)$. This implies that $\min\{\lceil\log_2(4/\Delta_{\min}(\epsilon))\rceil, \lceil\log_2(4/\widetilde{\beta})\rceil\} \leq \bar{t}$. Applying Lemmas~\ref{lem:milk_active_good} and \ref{lem:milk_active_bad}, we have that $t_{\max} \leq \min\{\lceil\log_2(4/\Delta_{\min}(\epsilon))\rceil, \lceil\log_2(4/\widetilde{\beta})\rceil\} \leq \bar{t}$ and that $\A_t \subseteq \S_t$ for all rounds $t$. 
Now we proceed by bounding the total number of samples drawn. 
\begin{align*}
    \tau & \leq \sum_{t=1}^{t_{\max}} N_t \\
    %%%%%%%%%%%%%%%%%%%%%
    &\leq \sum_{t=1}^{\min\{\lceil\log_2(4/\Delta_{\min}(\epsilon))\rceil, \lceil\log_2(4/\widetilde{\beta})\rceil\}} N_t\\
    %%%%%%%%%%%%%%%%%%%%%
    &= \sum_{t=1}^{\lceil\log_2(4(\Delta_{\min}(\epsilon)\vee \widetilde{\beta})^{-1})\rceil} N_t\\
    %%%%%%%%%%%%%%%%%%%%%%
    & = \sum_{t=1}^{\lceil\log_2(4(\Delta_{\min}(\epsilon)\vee \widetilde{\beta})^{-1})\rceil} \max\left\{c_1\log(|\X|/\delta), c^2 2^{2t} f(\Y^\epsilon(\A_t);\gamma) (B^2 + \sigma^2) \log(2t^2|\X|^2/\delta)\right\}\\
    %%%%%%%%%%%%%%%%%%%%%%
    & \leq c_1\log(|\X|/\delta)\lceil\log_2(4(\Delta_{\min}(\epsilon)\vee \widetilde{\beta})^{-1})\rceil + c^2(B^2 + \sigma^2)\sum_{t=1}^{\lceil\log_2(4(\Delta_{\min}(\epsilon)\vee \widetilde{\beta})^{-1})\rceil} 2^{2t} f(\Y^\epsilon(\A_t);\gamma)\cdot\log(2t^2|\X|^2/\delta) \\
    %%%%%%%%%%%%%%%%%%%%%%
    & = c_1\log(|\X|/\delta)\lceil\log_2(4(\Delta_{\min}(\epsilon)\vee \widetilde{\beta})^{-1})\rceil + \\
    & \hspace{1.5cm}c^2(B^2 + \sigma^2)\sum_{t=1}^{\lceil\log_2(4(\Delta_{\min}(\epsilon)\vee \widetilde{\beta})^{-1})\rceil} 2^{2t} \min_{\lambda \in \tX}\max_{\by \in \Y^\epsilon(\A_t)} \|\by\|_{\left(A(\lambda)+\gamma I\right)^{-1}}^2\cdot\log(2t^2|\X|^2/\delta) \\
    %%%%%%%%%%%%%%%%%%%%%%
    & \leq c_1\log(|\X|/\delta)\lceil\log_2(4(\Delta_{\min}(\epsilon)\vee \widetilde{\beta})^{-1})\rceil + \\
    & \hspace{1cm}c^2(B^2 + \sigma^2)\log\left(\frac{4|\X|^2\lceil\log_2(4(\Delta_{\min}(\epsilon)\vee \widetilde{\beta})^{-1})\rceil^2}{\delta}\right)\\
    &\hspace{1cm} \cdot \sum_{t=1}^{\lceil\log_2(4(\Delta_{\min}(\epsilon)\vee \widetilde{\beta})^{-1})\rceil} 2^{2t} \min_{\lambda \in \tX}\max_{\by \in \Y^\epsilon(\A_t)} \|\by\|_{\left(A(\lambda)+\gamma I\right)^{-1}}^2\\
    %%%%%%%%%%%%%%%%%%%%%%
    & \leq c_1\log(|\X|/\delta)\lceil\log_2(4(\Delta_{\min}(\epsilon)\vee \widetilde{\beta})^{-1})\rceil + \\
    & \hspace{1cm}c^2(B^2 + \sigma^2)\log\left(\frac{4|\X|^2\lceil\log_2(4(\Delta_{\min}(\epsilon)\vee \widetilde{\beta})^{-1})\rceil^2}{\delta}\right)\\
    &\hspace{1cm} \cdot \sum_{t=1}^{\lceil\log_2(4(\Delta_{\min}(\epsilon)\vee \widetilde{\beta})^{-1})\rceil} 2^{2t} \min_{\lambda \in \tX}\max_{\by \in \Y^\epsilon(\S_t)} \|\by\|_{\left(A(\lambda)+\gamma I\right)^{-1}}^2\\
    %%%%%%%%%%%%%%%%%%%%%%
    & = c_1\log(|\X|/\delta)\lceil\log_2(4(\Delta_{\min}(\epsilon)\vee \widetilde{\beta})^{-1})\rceil + \\
    & \hspace{1cm}c^2(B^2 + \sigma^2)\log\left(\frac{4|\X|^2\lceil\log_2(4(\Delta_{\min}(\epsilon)\vee \widetilde{\beta})^{-1})\rceil^2}{\delta}\right)\\
    &\hspace{1cm} \cdot \sum_{t=1}^{\lceil\log_2(4(\Delta_{\min}(\epsilon)\vee \widetilde{\beta})^{-1})\rceil}  \min_{\lambda \in \tX}\max\left\{2^{2t}\max_{\by \in \Y^\epsilon(\S_t^\text{Above})} \|\by\|_{\left(A(\lambda)+\gamma I\right)^{-1}}^2, 2^{2t}\max_{\by \in \Y^\epsilon(\S_t^\text{Below})} \|\by\|_{\left(A(\lambda)+\gamma I\right)^{-1}}^2\right\}.
\end{align*}
where the final equality follows by partitioning $\S_t = \S_t^\text{Above} \cup \S_t^\text{Below}$. 
Focusing on this final summation, note that 
\begin{align*}
    &\frac{1}{\lceil\log_2(4(\Delta_{\min}(\epsilon)\vee \widetilde{\beta})^{-1})\rceil} \sum_{t=1}^{\lceil\log_2(4(\Delta_{\min}(\epsilon)\vee \widetilde{\beta})^{-1})\rceil} 2^{2t} \min_{\lambda \in \tX}\max\left\{\max_{\by \in \S_t^\text{Above}} \|\by\|_{(A(\lambda)+\gamma I)^{-1}}^2, \max_{\by \in \S_t^\text{Below}} \|\by\|_{(A(\lambda)+\gamma I)^{-1}}^2\right\} \\
    % %%%%%%%%%%%%%%%%%%%%%%%%
    & \leq \max_{t \leq \lceil\log_2(4(\Delta_{\min}(\epsilon)\vee \widetilde{\beta})^{-1})\rceil}\min_{\lambda \in \tX} 2^{2t} \max\left\{\max_{\by \in \S_t^\text{Above}} \|\by\|_{(A(\lambda)+\gamma I)^{-1}}^2, \max_{\by \in \S_t^\text{Below}} \|\by\|_{(A(\lambda)+\gamma I)^{-1}}^2\right\}\\
    %%%%%%%%%%%%%%%%%%%%%%%%%
    & \leq  \min_{\lambda \in \tX} \max_{t \leq \lceil\log_2(4(\Delta_{\min}(\epsilon)\vee \widetilde{\beta})^{-1})\rceil} \max\left\{\max_{\by \in \S_t^\text{Above}} 2^{2t}\|\by\|_{(A(\lambda)+\gamma I)^{-1}}^2, \max_{\by \in \S_t^\text{Below}} 2^{2t}\|\by\|_{(A(\lambda)+\gamma I)^{-1}}^2\right\}\\
    %%%%%%%%%%%%%%%%%%%%%%%%%
    & =   \min_{\lambda \in \tX} \max_{t \leq \lceil\log_2(4(\Delta_{\min}(\epsilon)\vee \widetilde{\beta})^{-1})\rceil} \max\left\{\max_{(\bx, \bx') \in \S_t^\text{Above}} 2^{2t}\|\phi(\bx) - (1-\epsilon)\phi(\bx')\|_{(A(\lambda)+\gamma I)^{-1}}^2\right., \\
    &\hspace{5cm} \left.\max_{(\bx, \bx') \in \S_t^\text{Below}} 2^{2t}\|\phi(\bx) - (1-\epsilon)\phi(\bx')\|_{(A(\lambda)+\gamma I)^{-1}}^2\right\}\\
    %%%%%%%%%%%%%%%%%%%%%%%%%
    & \stackrel{\text{Lemmas~\ref{lem:milk_active_good}, \ref{lem:milk_active_bad}}, \widetilde{\beta}}{\leq}     16\min_{\lambda \in \tX} \max_{t \leq \lceil\log_2(4(\Delta_{\min}(\epsilon)\vee \widetilde{\beta})^{-1})\rceil} \max\left\{\max_{(\bx,\bx')\in \S_t^\text{Above}} \frac{\|\phi(\bx) - (1-\epsilon)\phi(\bx')\|_{(A(\lambda)+\gamma I)^{-1}}^2}{\max\{((\phi(\bx) - (1-\epsilon)\phi(\bx'))^T\theta_\ast)^2, \widetilde{\beta}^2\}},\right. \\
    &\hspace{5cm}\left. \max_{(\bx, \bx') \in \S_t^\text{Below}} \frac{\|\phi(\bx) - (1-\epsilon)\phi(\bx')\|_{(A(\lambda)+\gamma I)^{-1}}^2}{\max\{((\phi(\bx) - (1-\epsilon)\phi(\bx_\ast))^T\theta_\ast - \epsilon)^2, \widetilde{\beta}^2\}}\right\}\\
    % %%%%%%%%%%%%%%%%%%%%%%%%%
    & \leq 16\min_{\lambda \in \tX}  \max\left\{\max_{\bx\in  G_\epsilon}\max_{\bx'} \frac{\|\phi(\bx) - (1-\epsilon)\phi(\bx')\|_{(A(\lambda)+\gamma I)^{-1}}^2}{\max\{((\phi(\bx) - (1-\epsilon)\phi(\bx'))^T\theta_\ast)^2, \widetilde{\beta}^2\}},\right. \\
    &\hspace{5cm}\left. \max_{\bx\in  G_\epsilon^c}\max_{\bx'} \frac{\|\phi(\bx) - (1-\epsilon)\phi(\bx')\|_{(A(\lambda)+\gamma I)^{-1}}^2}{\max\{((\phi(\bx) - (1-\epsilon)\phi(\bx_\ast))^T\theta_\ast - \epsilon)^2, \widetilde{\beta}^2\}}\right\}
\end{align*}
Plugging this in with $c=4$ and $c_1=2$ from Theorem~\ref{thm:supp_robust_estimator} for RIPS with the Catoni estimator completes the proof.
\end{proof}
\section{Additional Experiment Details}

In this section we discuss additional experimental details not covered in the main paper. We first give an overview of the algorithms implemented in the following section. All code was written in python and run on a 64 core cluster machine. We have included implementations of all methods and a demo file showing how to call and run the various algorithms. 

\subsection{Algorithms Implemented}
In this section we briefly discuss the algorithms implemented and the hyper-parameters used in the algorithms. The algorithms implemented are s follows:

\textbf{Gaussian Process Experiments}
For all the algorithms in this section we assumed a GP Prior $N(0, k(x,x'))$ where $k(x,x')$ was the RBF kernel given by $k(x,x') = \exp(-\|x-x'\|^2/2\ell^2)$. 

At every time step we builds the confidence interval 
    \begin{align*}
        Q_{t}(\bx) \coloneqq \left[\mu_{t-1}(\bx) \pm \beta_{t}^{1 / 2} \sigma_{t-1}(\bx)\right]
    \end{align*}
    where $\mu_{t-1}$, and $\sigma_{t-1}$ is the posterior mean and variance function over the observed points. For an observation $\by_t$ at time $t$ we define $\mu_{t-1}$, and $\sigma_{t-1}$ as follows:
    \begin{align*}
    \mu_{t}(\bx) & \coloneqq\boldsymbol{k}_{t}(\boldsymbol{x})^{T}\left(\boldsymbol{K}_{t}+\sigma^{2} \boldsymbol{I}\right)^{-1} \by_{t} \\
    %%%%%%%%%%%%%%%%%%%%%
    k_{t}\left(\bx, \bx^{\prime}\right) & \coloneqq k\left(\bx, \bx^{\prime}\right)-\boldsymbol{k}_{t}(\bx)^{T}\left(\boldsymbol{K}_{t}+\sigma^{2} \boldsymbol{I}\right)^{-1} \boldsymbol{k}_{t}(\boldsymbol{x}) \\
    %%%%%%%%%%%%%%%%%%%%
    \sigma_{t}^{2}(\boldsymbol{x}) & \coloneqq k_{t}(\bx, \bx)
    \end{align*}
where, $\boldsymbol{k}_{t}(\bx)=\left[k\left(\bx_{1}, \bx\right), \ldots, k\left(\bx_{t}, \bx\right)\right]^{T}$ and $\boldsymbol{K}_{t}$ is the kernel matrix over the observed points.

\begin{enumerate}
    \item \textbf{LSE:} We implemented the LSE algorithm by \cite{gotovos2013active}. This algorithm maintains an active set of unclassified points defined as $U_t$ and the super-level set $H_t$ and sub-level set $L_t$.  %Further the quantity $\beta_{t}$ signifies the confidence width. 
    
    % and we set it as 
    % $$
    % \beta_{t} \coloneqq 2 \log \left(|\mathcal{X}| \pi^{2} t^{2} /(6 \delta)\right)
    % $$ 
    % as defined in Theorem 1 of \cite{gotovos2013active}. We set $\delta = 0.1$. The sub-level set $L_t$ and super-level set $H_t$ is updated at every timestep based on the confidence region $C_{t}(\bx)$, which is constructed by  intersecting successive confidence intervals, as follows:
    At every round LSE selects the most ambiguous point, where the ambiguity is defined as 
    $$
    a_{t}(\boldsymbol{x})=\min \left\{\max \left(Q_{t}(\boldsymbol{x})\right)-\alpha, \alpha-\min \left(Q_{t}(\boldsymbol{x})\right)\right\}
    $$
    that is, the points LSE is most unsure to classify into $H_t$ or $L_t$. Note that in contrast to this approach MELK follows the optimal allocation over the active set to select the next sample.
    
    \item\textbf{TruVar:} We also implemented a modified version of TruVar\cite{bogunovic2016truncated} with zero cost and homoscedastic noise. TruVar samples in such a fashion to ensure the maximum decrease of the posterior variance. As above, we maintain a Gaussian Process Posterior and we sample the arm 
    \[\argmax_{x\in \mc{X}} \sum_{\bar{x}\in \mc{A}_t} \sigma^2_{t}(\bar{x}) - \sum_{\bar{x}\in \mc{A}_t} \sigma^2_{t-1|x}(\bar{x})\]
    where $\sigma^2_{t-1|x}(\bar{x})$ is the posterior variance of $\bar{x}$ if we sample $x$.
     \item \textbf{MELK:} As described in the text, we compute the means and variances of the arms using a Gaussian posterior (identical to above) and eliminate arms when their lower/upper bound is below/above the specified threshold $\tau$. We implemented a batched sampling algorithm where we compute the design 
     \[\min_{\lambda\in \X}\max_{z\in A_t}\|z\|_{(A(\lambda) + \gamma I)^{-2}}^2\]
     ever 10 samples and then sample from it. At the $i$-th calculation, $\gamma = 1/(10*i)$.   
     We also use the Frank-Wolfe method to compute the optimal allocation over the active set before every round as described in Section~\ref{sec:kernel_compute}. 
     We set the step-size of Frank-Wolfe method as $1$ and cap the maximum number of iteration to converge for Frank-Wolfe to $500$.
\end{enumerate}

\textbf{Linear Bandits Examples}

Additionally, we also consider comparing algorithms exactly as written using theoretically justified confidence widths in all cases. This presents a challenge as \texttt{MELK} and \texttt{MILK} are designed for the frequentist regime and \texttt{LSE} and \texttt{TruVar} are Bayesian in nature. To level the playing field, we consider all algorithms in the frequentist regime. 
For this experiment, we focused primarily on comparing \texttt{MELK} to \texttt{LSE} and \texttt{MILK} to \texttt{LSE-imp}
\texttt{LSE} can naturally be adapted to the frequentist setting with the tight RKHS confidence bounds from \cite{chowdhury2017kernelized}. These bounds scale with the maximum information gain $\Gamma_T$. To make the comparison fair, we consider all algorithms in the linear regime where $\Gamma_T = O(d\log(T))$. By contrast, for the squared exponential kernel, $\Gamma_T = O\left(\log(T)^d\right)$, and this leads to overly pessimistic confidence widths preventing a meaningful comparison of the algorithms. Indeed, even for moderate $d$ such as $d=4$, \texttt{LSE} had confidence widths that were more that an order of magnitude wider for the squared exponential kernel. Hence, we focus on the case of the linear kernel for our experimental comparison where the differences are not so stark. Below, we describe all algorithms in this regime. 
 
\texttt{LSE} follows the same acquisition function described in the previous section. We provide additional details about \texttt{MELK}, \texttt{MILK}, and \texttt{LSE-imp} in this setting. 

\begin{enumerate}

    \item \textbf{MELK:} We implement the MELK algorithm as defined in \Cref{alg:MELK}. Recall that $|f(x)| \leq B$, and for the experiments we set $B=1$. We set the confidence parameter $\delta = 0.1$, the regularization parameter $\gamma = 1e-7$. Note that we use the original confidence width of $(B^2 + \sigma^2) \log(2t^2|\X|^2/\delta)$ as stated in our algorithm, where $\sigma^2$ is the noise parameter specific to the environment. We also use the Frank-Wolfe method to compute the optimal allocation over the active set before every  round. We set the step-size of Frank-Wolfe method as $0.5$ and cap the maximum number of iteration to converge for Frank-Wolfe to $2000$.
    
    \item \textbf{LSE-imp:} We implement the LSE-Implicit algorithm as stated in \cite{gotovos2013active}. LSE-Implicit proceeds quite similarly to LSE by constructing the confidence region $C_{t}(\bx)$ (as defined above) and classifying points to the sub-level set $L_t$ or super-level set $H_t$. We set the confidence width as in \texttt{LSE} for calculating the confidence region. Note that LSE-Implicit works in the implicit level set estimation setting and so constructs an estimate of the function maximum to classify points into $H_t$ or $L_t$. It builds an  optimistic and pessimistic estimate of the function maximum as 
    \begin{align*}
        f_{t}^{o p t} \coloneqq \max _{x \in U_{t}} \max \left(C_{t}(x)\right), \quad f_{t}^{\text {pes }}=\max _{x \in U_{t}} \min \left(C_{t}(x)\right)
    \end{align*}
    respectively. A point $\bx$ is classified into $H_t$ if $\min \left(C_{t}(\bx)\right) \geq (1-\epsilon) f_{t}^{opt}$ or classified into $L_t$ if $\max \left(C_{t}(\bx)\right) \leq (1-\epsilon) f_{t}^{pes}$. Finally, LSE-Implicit selects the next point with the largest confidence region width, defined as follows:
    $$
    w_{t}(\boldsymbol{x})=\max \left(C_{t}(\boldsymbol{x})\right)-\min \left(C_{t}(\boldsymbol{x})\right)
    $$
    such that this leads to more exploration. Again, note that in contrast MILK in \Cref{alg:MILK} uses the optimal allocation proportion over the active set to sample the next point.
    
    \item \textbf{MILK:} We implement the MILK algorithm as stated in \Cref{alg:MILK}. Note that MILK proceeds as similarly to MELK but with the allocation calculated over the difference of vectors $\Y^\epsilon(\A)$ over the active set and a different elimination condition depending on $\epsilon$. For MILK we set a similar hyper-parameters like MELK. We set the confidence parameter $\delta = 0.1$, the regularization parameter $\gamma = 1e-7$, and the confidence width of $(B^2 + \sigma^2) \log(2t^2|\X|^2/\delta)$. We use the Frank-Wolfe method to compute the optimal allocation over the active set of points and set the step-size of Frank-Wolfe method as $0.5$ and cap the maximum number of iteration to converge for Frank-Wolfe to $2000$. Note that we set $\epsilon$ depending on specific environment setting.
\end{enumerate}

\subsection{Additional Experiments}
All experiments were done with 25 repetitions. We consider the $f1$-scores on three environments considered below.

\begin{figure}[H]
\centering
\includegraphics[width=0.8\linewidth]{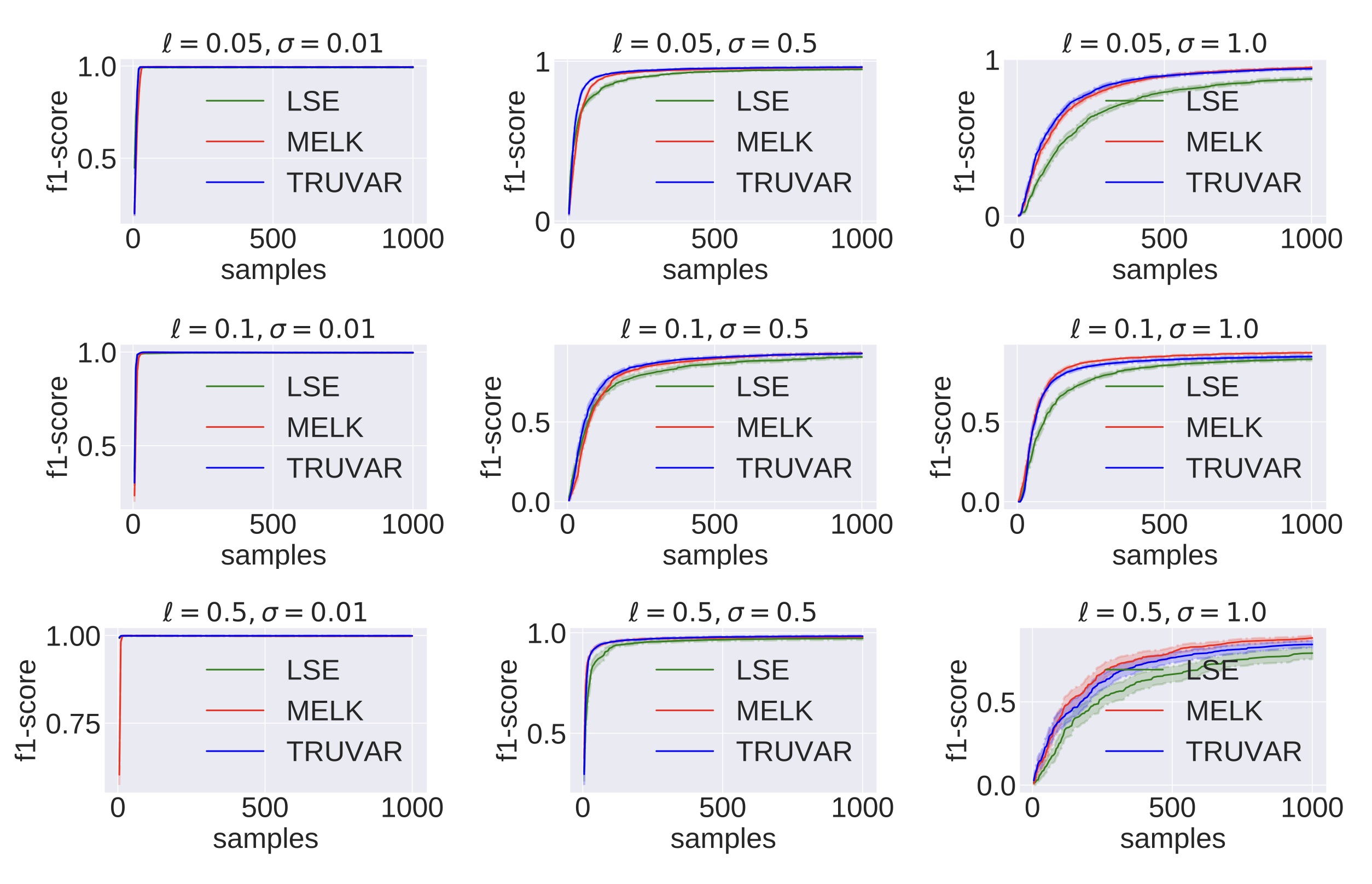}
\caption{ $f$ drawn randomly from a squared exponential kernel $N(0,k(\bx,\bx'))$. $\sigma$ denotes the standard deviation of the noise and $\ell$ denotes the bandwidth of the kernel (i.e., $k(\bx, \by) = \exp(-\|\bx - \by\|/2\ell^2))$. }
\label{fig:supp_well_spec}
\end{figure}

\begin{figure}[H]
\centering
\includegraphics[width=0.8\linewidth]{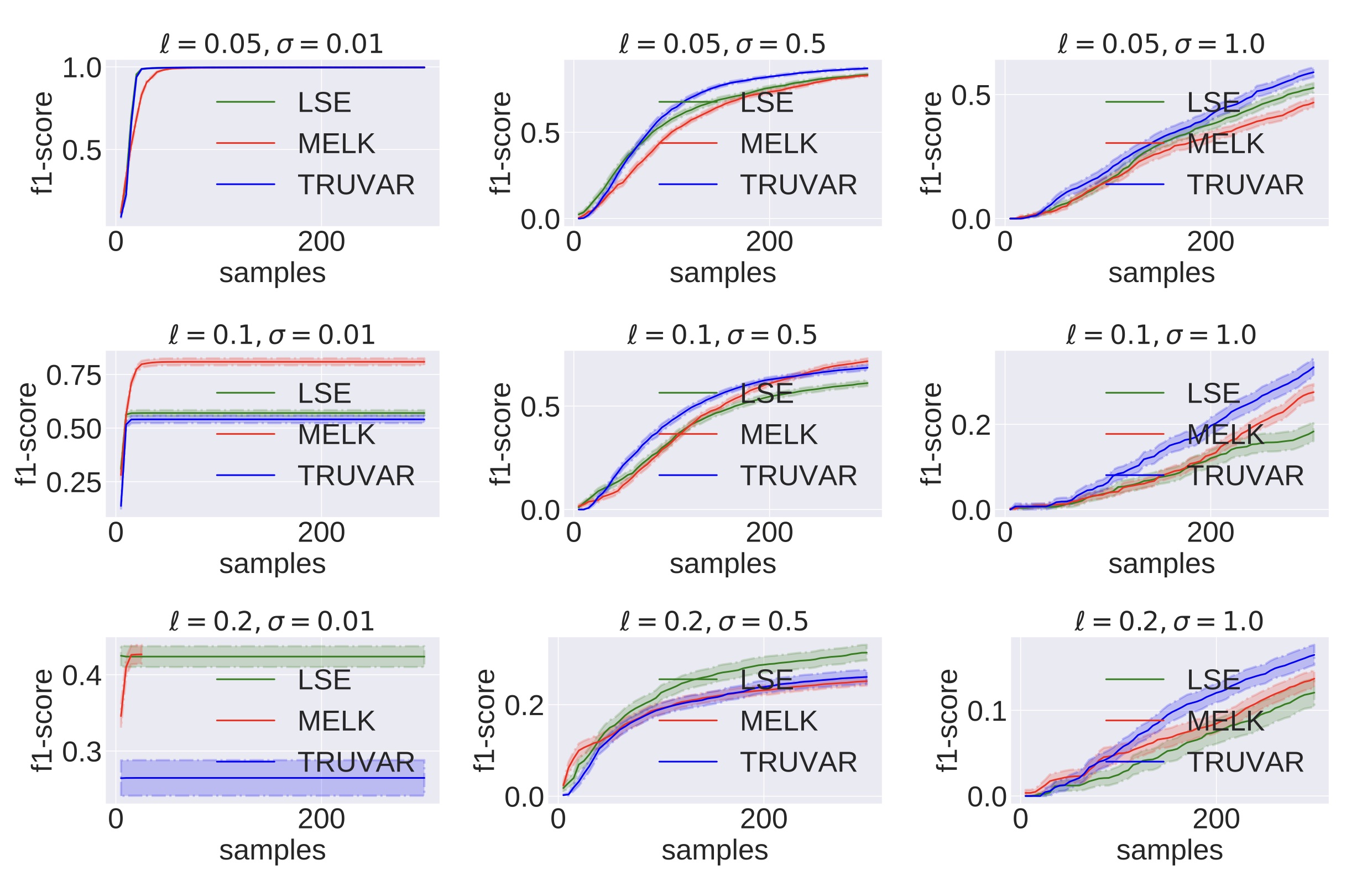}
\caption{$f(x) = \cos(8\pi x)$. $\sigma$ denotes the standard deviation of the noise and $\ell$ denotes the bandwidth of the kernel (i.e., $k(\bx, \by) = \exp(-\|\bx - \by\|/2\ell^2))$. }
\label{fig:supp_miss_spec}
\end{figure}

\begin{figure}[H]
\centering
\includegraphics[width=0.8\linewidth]{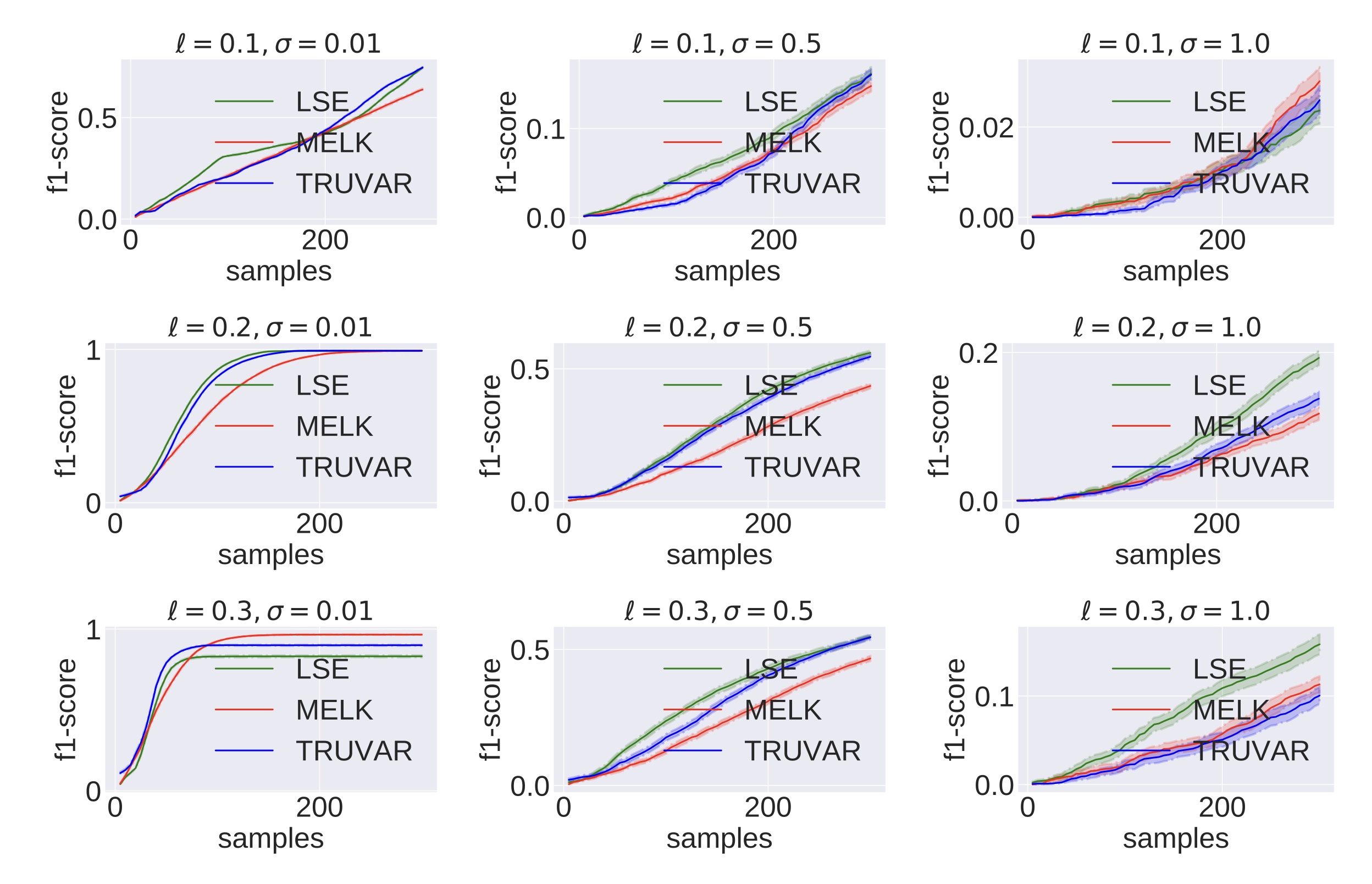}
\caption{$f(x,y) = \cos(2\pi x)\sin(2\pi y)$. $\sigma$ denotes the standard deviation of the noise and $\ell$ denotes the bandwidth of the kernel (i.e., $k(\bx, \by) = \exp(-\|\bx - \by\|/2\ell^2))$. }
\label{fig:supp_2dmiss_spec}
\end{figure}

\textbf{Linear Examples with true confidence widths}

Finally we compare the performance of the methods using exact confidence widths. 

\begin{figure}[H]
\centering

\begin{subfigure}{0.48\textwidth}
  \centering
  \includegraphics[width=\linewidth]{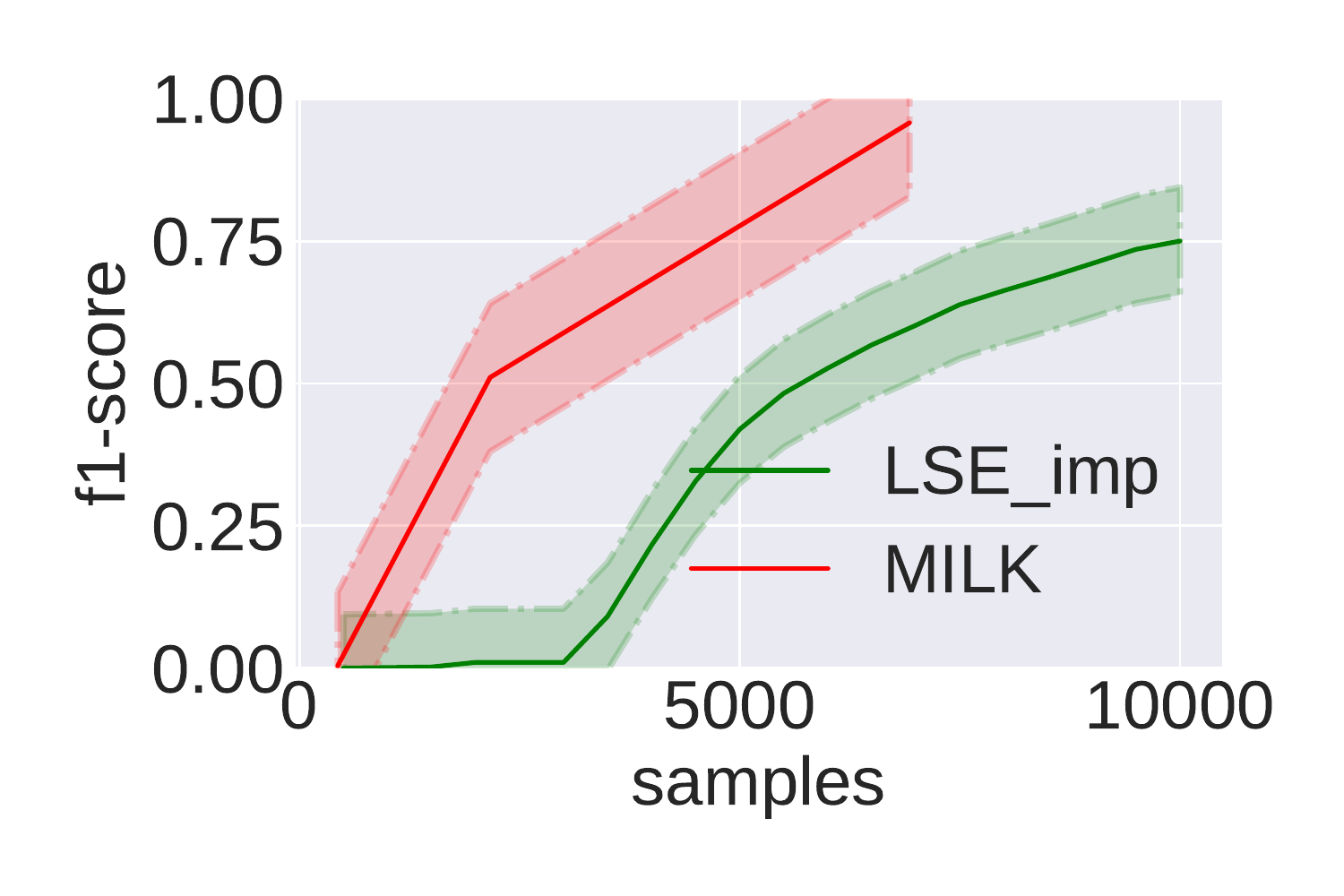}
  \caption{Linear, explicit}
  \label{fig:lin_explicit_supp}
    \end{subfigure}
\begin{subfigure}{.48\textwidth}
  \centering
  \includegraphics[width=\linewidth]{img/soare_implicit_v2.pdf}
\caption{Linear, Implicit}
\label{fig:lin_implicit_supp}
\end{subfigure}
\label{fig:soare}
\caption{Comparison of algorithms using theoretically justified confidence widths on a linear bandit setting.}
\end{figure}

For the Linear kernel experiments in Figures~\ref{fig:lin_explicit_supp} and \ref{fig:lin_implicit_supp}, we run all algorithms with exact confidence intervals as specified by theoretical guarantees and use the theoretical upper bound on information gain $\gamma_T$ shown in \citep{srinivas2009gaussian} for the confidence widths from \citep{valko2013finite} needed for LSE. 
We compare the methods on a benchmark example from the linear bandits literature. For $\bx_1, \cdots, \bx_{n} \in \R^{d}$, we take $\bx_1 = \bx_\ast = \theta_\ast = e_1$ and $\bx_2 = e_2$. The remaining $\bx_3, \cdots, \bx_n$  are set so that their first two coordinates are $\cos(\pi/4(1 + \xi))e_1$ and $\sin(\pi/4(1 + \xi))e_2$ for $\xi \sim \text{Unif}(-.2, .2)$. We set the threshold $\alpha = 0.5$, $n=100$, and $d=25$. Figure~\ref{fig:lin_explicit_supp} shows that \texttt{MELK} outperforms \texttt{LSE} when both algorithms are run with their exact confidence widths. 

In the implicit setting, this example is especially informative and highlights the importance of designing to choose which arms to sample. 
Though it is far below $\alpha$, sampling arm $\bx_2$ provides the most information about which arms exceed the implicit threshold. Indeed, we see in  \ref{fig:lin_implicit_supp} that both \texttt{MILK} greatly outperforms \texttt{LSE-imp} respectively.

% \begin{enumerate}
%     \item \textbf{Sinusoid well-specified Experiment}: In the well-specified Sinusoid experiment we consider the function $\sin \left(10 x_{1}\right)+\cos \left(4 x_{2}\right)-\cos \left(3 x_{1} x_{2}\right)$ whose contours are shown in \Cref{fig:allocations} (top-left). The contours are plotted on a box of $[0,1]\times[0,1]$ uniformly separated into a grid of $30\times30$. So there are $900$ points (arms) in this setting. The noise parameter is set $\sigma^2 = 1$. The feature maps $\phi(x)$ is constructed using the radial basis function (rbf) kernel with the rbf bandwidth set as $\sigma_{\text{rbf}} = 0.1$. To construct this as a well-specified case we choose the means $f(x) = \theta_{\ast}^T\phi(\bx)$. Finally the threshold is set at $\alpha = -0.05$. A similar synthetic experiment has also been considered by \citep{zanette2018robust}. For MILK we set the $\epsilon = 0.6$.
    
%     \item \textbf{Sinusoid mis-specified experiment}: In the mis-specified setting 
% \end{enumerate}
\section{Reducing Experimental Design in an RKHS to a finite dimensional optimization}\label{sec:kernel_compute}

In this section we describe the use of the kernel trick and Frank-Wolfe to compute the design
\[f(\lambda) = \min_{\lambda\in \triangle_{\X}}\max_{x\in C} \|\phi(\bx)\|_{A^{\gamma}(\lambda)^{-1}}\]
where $C\subset \X$.

Since this is a convex optimization problem on the finite dimensional simplex $\triangle_{\X}$ we employ the Frank-Wolfe algorithm.
\setlength{\textfloatsep}{5pt}
\begin{algorithm}[tbh]
\caption{Frank-Wolfe to minimize $f$}
\begin{algorithmic}[1]
\Require{Arms $\X$, iterations $T$}
\State{$\lambda_0 = \mathbf{e}_1$ (first standard basis vector)}

\For{$\bx \in \A_t$}
   \State{$x_t \gets \arg\max_{x\in \mc{X}} \|\phi(\bx)\|^2_{A^{\gamma}(\lambda)}$ }
   \State{$g_t = \nabla_{\lambda_{t-1}}\|\phi(\bx_{t})\|^2_{A^{\gamma}(\lambda)}$ }
   \State{$j_t = \arg\max_{1\leq j\leq |X|} e_i^{\top} g_t $}
   \State{$\eta_{t} = \frac{1}{t+2}$}
   \State{$\lambda_t = (1-\eta_{t})\lambda_{t-1} + \eta_t$}
\EndFor
%\Return{$\hG_{t}$}
\Return{$\lambda_{T}$}
\end{algorithmic} 
\end{algorithm} 
Note that $\lambda_t$ is at most $t$-sparse. The primary challenge is in the computation of the gradient of $f$. To do so we leverage a small modification of Lemma 1 of ~\cite{camilleri2021highdimensional}. 
\begin{lemma} Assume that $\lambda$ is $s$-sparse and (without loss of generality) with it's support corresponding to $x_1, \cdots, x_s\in \mc{X}$. Then,
\[\phi(\bx)^{\top}A^{\gamma}(\lambda)\phi(\by) = \frac{k(\bx,\by)}{\gamma} - \frac{1}{\gamma}k_{\lambda}(\bx)^{\top}(K_{\lambda} + \gamma I_{s})k_{\lambda}(\by)\]
where $k_{\lambda}(\cdot)\in \mathbb{R}^s$ with $[k_{\lambda(\bx)}]_i = \sqrt{\lambda_i}k(\bx_i, \bx)$ for $i\leq s$ and $K_{\lambda}\in \mathbb{R}^{s\times s}$ with $[K_{\lambda}]_{i,j} = \sqrt{\lambda_i\lambda_j} k(\bx_i, \bx_j)$.
\end{lemma}

Now, identifying $\X$ with an indexing of it's entries, i.e. $\X = \{\bx^1, \cdots, \bx^{|X|}\}$ a computation shows that 
\[\mathbf{e}_i^{\top}[g_t] = -(\phi(\bx_t)A^{\gamma}(\lambda_t) \phi(\bx^i))^2\]
which can be computed by the above lemma.
Note that computationally, the most difficult step is the inversion of a $t\times t$ matrix at iteration $t$. For a small number of iterations (<2000), this is not prohibitive.

\end{document}